\documentclass[11pt]{article}

\usepackage{xspace}

\usepackage{amsmath}
\usepackage{graphicx}
\usepackage{framed}
\usepackage{bbm}
\usepackage{algorithm}
\usepackage{amsthm}
\usepackage[algo2e, vlined, ruled, linesnumbered]{algorithm2e}

\usepackage{natbib}

\usepackage{mathrsfs}
\usepackage{amssymb}
\usepackage{bm}
\usepackage[colorlinks=true, linkcolor=blue, citecolor=blue,urlcolor=black,linktocpage=true]{hyperref}
\usepackage{enumitem}
\setlist[itemize]{leftmargin=1cm}
\usepackage{makecell}
\usepackage{tabulary}
\usepackage{xcolor}
\usepackage{prettyref}
\usepackage{mathtools}
\usepackage{amsmath}
\usepackage{multirow}
\usepackage{nicefrac}
\usepackage{amssymb}% http://ctan.org/pkg/amssymb
\usepackage{pifont}% http://ctan.org/pkg/pifont

\definecolor{Green}{rgb}{0.13, 0.65, 0.3}
\usepackage{colortbl}
\usepackage[]{color-edits}
\addauthor{HL}{red}
\DeclareMathOperator*{\argmin}{argmin} % no space, limits underneath in displays
\newcommand{\eps}{\varepsilon}
\newcommand{\alglog}{\text{logdet-FTRL}\xspace}

\newcommand{\e}{\mathrm{e}\xspace}

\newcommand{\hatx}{\hat{x}}
\newcommand{\ind}{\mathbb{I}}

\newcommand{\Reg}{\text{\rm Reg}}

\newcommand{\hatSigma}{\hat{\Sigma}}
\newcommand{\hatH}{\hat{H}}

\newcommand{\calT}{\mathcal{T}}

\newcommand{\calA}{\mathcal{A}}

\newcommand{\calE}{\mathcal{E}}
\newcommand{\calZ}{\mathcal{Z}}
\newcommand{\calH}{\mathcal{H}}

\newcommand{\calP}{\mathcal{P}}
\newcommand{\calD}{\mathcal{D}}

\newcommand{\tildep}{\tilde{p}}

\newcommand{\tildeH}{\tilde{H}}

\newcommand{\otil}{\widetilde{\order}}

\newcommand{\poly}{\text{poly}}
\newcommand{\E}{\mathbb{E}}
\newcommand{\order}{\mathcal{O}}

\newcommand{\hatell}{\hat{\ell}}

\newcommand{\norm}[1]{\left\|#1\right\|}

\newcommand{\calF}{\mathcal{F}}

\newcommand{\inner}[1]{\langle#1\rangle}

\newcommand{\stable}{\text{STABILISE}\xspace}

\newcommand{\alg}{{\small\textsf{\textup{ALG}}}\xspace}

\newtheorem{theorem}{Theorem}
\newtheorem{assumption}{Assumption}

\newtheorem{lemma}[theorem]{Lemma}
\newtheorem{corollary}[theorem]{Corollary}
\newtheorem{proposition}{Proposition}
\newtheorem{protocol}{Protocol}
\newtheorem{definition}[theorem]{Definition}

\usepackage{tikz}

\newcommand{\nonl}{\renewcommand{\nl}{\let\nl}}

\usepackage{prettyref}
\newcommand{\pref}[1]{\prettyref{#1}}

\newcommand{\savehyperref}[2]{\texorpdfstring{\hyperref[#1]{#2}}{#2}}
\newrefformat{eq}{\savehyperref{#1}{Eq.~\textup{(\ref*{#1})}}}
\newrefformat{eqn}{\savehyperref{#1}{Equation~\ref*{#1}}}
\newrefformat{lem}{\savehyperref{#1}{Lemma~\ref*{#1}}}
\newrefformat{lemma}{\savehyperref{#1}{Lemma~\ref*{#1}}}
\newrefformat{def}{\savehyperref{#1}{Definition~\ref*{#1}}}
\newrefformat{line}{\savehyperref{#1}{Line~\ref*{#1}}}
\newrefformat{thm}{\savehyperref{#1}{Theorem~\ref*{#1}}}
\newrefformat{corr}{\savehyperref{#1}{Corollary~\ref*{#1}}}
\newrefformat{cor}{\savehyperref{#1}{Corollary~\ref*{#1}}}
\newrefformat{sec}{\savehyperref{#1}{Section~\ref*{#1}}}
\newrefformat{subsec}{\savehyperref{#1}{Section~\ref*{#1}}}
\newrefformat{app}{\savehyperref{#1}{Appendix~\ref*{#1}}}
\newrefformat{assum}{\savehyperref{#1}{Assumption~\ref*{#1}}}
\newrefformat{ex}{\savehyperref{#1}{Example~\ref*{#1}}}
\newrefformat{fig}{\savehyperref{#1}{Figure~\ref*{#1}}}
\newrefformat{alg}{\savehyperref{#1}{Algorithm~\ref*{#1}}}
\newrefformat{rem}{\savehyperref{#1}{Remark~\ref*{#1}}}
\newrefformat{conj}{\savehyperref{#1}{Conjecture~\ref*{#1}}}
\newrefformat{prop}{\savehyperref{#1}{Lemma~\ref*{#1}}}
\newrefformat{proto}{\savehyperref{#1}{Protocol~\ref*{#1}}}
\newrefformat{prob}{\savehyperref{#1}{Problem~\ref*{#1}}}
\newrefformat{claim}{\savehyperref{#1}{Claim~\ref*{#1}}}
\newrefformat{que}{\savehyperref{#1}{Question~\ref*{#1}}}
\newrefformat{op}{\savehyperref{#1}{Open Problem~\ref*{#1}}}
\newrefformat{fn}{\savehyperref{#1}{Footnote~\ref*{#1}}}
\newrefformat{tab}{\savehyperref{#1}{Table~\ref*{#1}}}
\newrefformat{fig}{\savehyperref{#1}{Figure~\ref*{#1}}}
\newrefformat{proc}{\savehyperref{#1}{Procedure~\ref*{#1}}}

\usepackage{caption}

\usepackage{times} 
\usepackage{amsmath}

\newcommand{\calC}{\mathcal{C}}

\newcommand{\CW}[1]{{\color{red}{\small{[CW: #1]}}}}

\newcommand{\emp}{\frac{1}{t-1}\sum_{\tau=1}^{t-1}}

\DeclareMathOperator{\Cov}{Cov}
\DeclareMathOperator{\tr}{Tr}
\DeclareMathOperator{\Var}{Var}
\DeclareMathOperator{\supp}{supp}
\usepackage[utf8]{inputenc} % allow utf-8 input
\usepackage[T1]{fontenc}    % use 8-bit T1 fonts
\usepackage{hyperref}       % hyperlinks
\usepackage{url}            % simple URL typesetting
\usepackage{booktabs}       % professional-quality tables
\usepackage{amsfonts}       % blackboard math symbols
\usepackage{nicefrac}       % compact symbols for 1/2, etc.
\usepackage{microtype}      % microtypography
\usepackage{xcolor}         % colors
\usepackage{titletoc}
\usepackage[toc,page,header]{appendix}
\usepackage{algorithmic}

\title{Bypassing the Simulator: \\ Near-Optimal Adversarial Linear Contextual Bandits}

\date{}

\author{%
    Haolin Liu\thanks{The authors are listed in alphabetical order. } \\
    \scalebox{0.9}{University of Virginia}\\
    \scalebox{0.9}{\texttt{srs8rh@virginia.edu}} 
    \and 
    Chen-Yu Wei$^*$\thanks{This work was done when Chen-Yu Wei was at MIT Institute for Data, Systems, and Society. } \\
    \scalebox{0.9}{University of Virginia} \\ 
    \scalebox{0.9}{\texttt{chenyu.wei@virginia.edu}} 
    \and Julian Zimmert$^*$ \\
    \scalebox{0.9}{Google Research} \\ 
    \scalebox{0.9}{\texttt{zimmert@google.com}}
}

\begin{document}

\maketitle

\begin{abstract}
We consider the adversarial linear contextual bandit problem, 
where the loss vectors are selected fully adversarially and the per-round action set (i.e. the context) is drawn from a fixed distribution.
%where the contexts are drawn from a fixed distribution and the loss vectors are fully adversarial. 
Existing methods for this problem either require access to a simulator to generate free i.i.d. contexts, achieve a sub-optimal regret no better than $\otil(T^{\nicefrac{5}{6}})$, or are computationally inefficient. 
We greatly improve these results by achieving a regret of $\otil(\sqrt{T})$ without a simulator, while maintaining computational efficiency when the action set in each round is small. 
In the special case of sleeping bandits with adversarial loss and stochastic arm availability, our result answers affirmatively the open question by \cite{saha2020improved} on whether there exists a polynomial-time algorithm with  $\poly(d)\sqrt{T}$ regret. 
%The key elements in our approach are 1) tighter concentration bounds and a more refined estimator for the feature covariance matrix that enable efficient reuse of previously collected contexts, and 2) a combination of the log-determinant barrier in the follow-the-regularized-leader framework and the biasing technique that facilitates the construction of negatively biased loss estimators. 
Our approach naturally handles the case where the loss is linear up to an additive misspecification error, and our regret shows near-optimal dependence on the magnitude of the error.  
\end{abstract}

%\begin{abstract}
%  The abstract paragraph should be indented \nicefrac{1}{2}~inch (3~picas) on
%  both the left- and right-hand margins. Use 10~point type, with a vertical
%  spacing (leading) of 11~points.  The word \textbf{Abstract} must be centered,
%  bold, and in point size 12. Two line spaces precede the abstract. The abstract
%  must be limited to one paragraph.
%\end{abstract}

\section{Introduction}

Contextual bandit is a widely used model for sequential decision making. The interaction between the learner and the environment proceeds in rounds: in each round, the environment provides a context; based on it, the learner chooses an action and receive a reward. The goal is to maximize the total reward across multiple rounds. This model has found extensive applications in fields such as medical treatment \citep{tewari2017ads}, personalized recommendations \citep{beygelzimer2011contextual}, and online advertising \citep{chu2011contextual}.

Algorithms for contextual bandits with provable guarantees have been developed under various assumptions. In the linear regime, the most extensively studied model is the \emph{stochastic linear contextual bandit}, in which the context can be arbitrarily distributed in each round, while the reward is determined by a fixed linear function of the context-action pair. Near-optimal algorithms for this setting have been established in, e.g., \citep{chu2011contextual, abbasi2011improved, li2019nearly, foster2020adapting}. Another model, which is the focus of this paper, is the \emph{adversarial linear contextual bandit}, in which the context is drawn from a fixed distribution, while the reward is determined by a time-varying linear function of the context-action pair. \footnote{Apparently, the stochastic and adversarial linear contextual bandits defined here are incomparable, and their names do not fully capture their underlying assumptions. However,  these are the terms commonly used in the literature (e.g., \citep{abbasi2011improved, neu2020efficient}). } A computationally efficient algorithm for this setting is first proposed by \cite{neu2020efficient}. However, existing research for this setting still faces challenges in achieving near-optimal regret and sample complexity when the context distribution is unknown. 

The algorithm by \cite{neu2020efficient} requires the learner to have \emph{full knowledge} on the context distribution, and access to an \emph{exploratory policy} that induces a feature covariance matrix with a smallest eigenvalue at least $\lambda$. Under these assumptions, their algorithm provides a regret guarantee of $\otil(\sqrt{d\log(|\calA|)T/\lambda})$\footnote{The linear contextual bandit problem formulation in \cite{neu2020efficient} is different from ours. However, it can be reduced to our setting with dimension $d|\calA|$, where $|\calA|$ is the maximum number of actions in each round. The $\otil(\sqrt{d\log(|\calA|)T/\lambda})$ bound reported here is obtained by adopting their techniques to our setting. }, where $d$ is the feature dimension, $|\calA|$ is the maximum size of the action set, and $T$ is the number of rounds. These assumptions are relaxed in the work of \cite{luo2021policy}, who studied a more general linear MDP setting. When specialized to linear contextual bandits, \cite{luo2021policy} only requires access to a \emph{simulator} from which the learner can draw free i.i.d. contexts. Their algorithm achieves a $\otil((d\log(|\calA|)T^2)^{\nicefrac{1}{3}}))$ regret. The regret is further improved to the near-optimal  one $\otil(\sqrt{d\log(|\calA|)T})$ by \cite{dai2023refined} through refined loss estimator construction. 

All results that attain $\otil(T^{\nicefrac{2}{3}})$ or $\otil(\sqrt{T})$ regret bound discussed above rely on access to the simulator. In their algorithms, the number of calls to the simulator significantly exceeds the number of interactions between the environment and the learner, but this is concealed from the regret bound. Therefore, their regret bounds do not accurately reflect the sample complexity of their algorithms. Another set of results for linear MDPs \citep{luo2021policy, dai2023refined, sherman2023improved, kong2023improved} also consider the simulator-free scenario, essentially using interactions with the environment to fulfill the original purpose of the simulator. When applying their techniques to linear contextual bandits, their algorithms only achieve a regret bound of $\otil(T^{\nicefrac{5}{6}})$ at best (see detailed analysis and comparison in \pref{app: comparison}). 

Our result significantly improves the previous ones: without simulators, we develop an algorithm that ensures a regret bound of order $\otil(d^2\sqrt{T})$, and it is computationally efficient as long as the size of the action set is small in each round (similar to all previous work). 
Unlike previous algorithms which always collect new contexts (through simulators or interactions with the environment) to estimate the feature covariance matrix, we leverage the context samples the learner received in the past to do this. Although natural, establishing a near-tight regret requires highly efficient use of context samples, necessitating a novel way to construct the estimator of feature covariance matrix and a tighter concentration bound for it. Additionally, to address the potentially large magnitude and the bias of the loss estimator, we turn to the use of log-determinant (logdet) barrier in the follow-the-regularized-leader (FTRL) framework. Logdet accommodates larger loss estimators and induces a larger bonus term to cancel the bias of the loss estimator, both of which are crucial for our result.

Our setting subsumes sleeping bandits with stochastic arm availability \citep{kanade2009sleeping, saha2020improved} and combinatorial semi-bandits with stochastic action sets \citep{neu2014online}. 
Our result answers affirmatively the main open question left by \cite{saha2020improved} on whether there exists a polynomial-time algorithm with  $\poly(d)\sqrt{T}$ regret for sleeping bandits with adversarial loss and stochastic availability. 

As a side result, we give a computationally inefficient algorithm that achieves an improved $\otil(d\sqrt{T})$ regret without a simulator. While this is a direct extension from the EXP4 algorithm \citep{auer2002nonstochastic}, such a result has not been established to our knowledge, so we include it for completeness. %\footnote{\cite{olkhovskaya2023first, kong2023improved} also proposed EXP4-based algorithms for linear CBs/MDPs. The setting in \cite{olkhovskaya2023first} is slightly different from ours, so their bound inevitably has a $\poly(|\calA|)$ dependence. The way \cite{kong2023improved} constructs loss estimators is different from ours, which leads to sub-optimal regret in CB. }

\renewcommand{\arraystretch}{1.4}
\begin{table}[t]
\caption{\fontsize{0.80\baselineskip}{1.0\baselineskip}\selectfont Related works in the ``S-A'' category. CB stands for contextual bandits and SB stands for semi-bandits. The relations among settings are as follows: Sleeping Bandit $\subset$ Contextual SB $\subset$ Linear CB, Linear CB $\subset$ Linear MDP, and Linear CB $\subset$ General CB. The table compares our results with the Pareto frontier of the literature. For algorithms dealing more general settings, we have carefully translated their techniques to Linear CB and reported the resulting bounds. $\Sigma_{\pi}$ denotes the feature covariance matrix induced by policy $\pi$. $|\calA|$ and $|\Pi|$ are sizes of the action set and the policy set. }
\vspace*{5pt}
\centering
%\hspace*{-32pt}
\scalebox{0.90}{
    \begin{tabular}{|c|c|c|c|c|c|}
\hline
Target Setting & Algorithm & Regret & Simulator & Computation & Assumption\\
\hline
General CB & \cite{syrgkanis2016improved} & $(\log|\Pi|)^{\nicefrac{1}{3}}(|\calA|T)^{\nicefrac{2}{3}}$ & \checkmark & $\poly(|\calA|, \log|\Pi|, T)$ & ERM Oracle \\
\hline
\multirow{4}{*}{Linear MDP} & \cite{dai2023refined} & $\sqrt{dT\log|\calA|}$ & \checkmark & $\poly(|\calA|, d,T)$ & \\
\cline{2-6}
 & \makecell{ \\[-8pt]\cite{dai2023refined}\\
 \cite{sherman2023improved}} & $d(\log|\calA|)^{\nicefrac{1}{6}}T^{\nicefrac{5}{6}}$ & & $\poly(|\calA|,d,T)$ & \\
\cline{2-6}
 & \cite{kong2023improved} & $(d^7T^4)^{\nicefrac{1}{5}} + \poly\left(\frac{1}{\lambda}\right)$ & & $T^d$ & $\exists \pi, \Sigma_\pi \succeq \lambda I$ \\
\hline
\multirow{2}{*}{Linear CB} & \pref{alg: FTRL} &  $d^2 \sqrt{T}$ & & $\poly(|\calA|,d,T)$ & \\
\cline{2-6}
 & \pref{alg: linear exp4} & $d\sqrt{T}$ & & $T^d$ & \\
\hline
\makecell{Contextual SB} & \cite{neu2014online} & $(dT)^{\nicefrac{2}{3}}$ & & $\poly(d,T)$ & \\
\hline
Sleeping Bandit & \cite{saha2020improved} & $\sqrt{2^d T}$ & & $\poly(d,T)$ ($|\calA|\leq d$) & \\
\hline
\end{tabular}
    }
\end{table}
\subsection{Related work}
We review the literature of various contextual bandit problems, classifying them based on the nature of the context and the reward function, specifically whether they are stochastic/fixed or adversarial. 

\paragraph{Contextual bandits with i.i.d. contexts and fixed reward functions (S-S)} Significant progress has been made in contextual bandits with i.i.d. contexts and fixed reward functions, under general reward function classes or policy classes \citep{langford2007epoch, dudik2011efficient, agarwal2012contextual, agarwal2014taming, simchi2022bypassing, xu2020upper}. In the work by \cite{dudik2011efficient, agarwal2012contextual, agarwal2014taming}, the algorithms also use previously collected contexts to estimate the inverse probability of selecting actions under the current policy. However, these results only obtain regret bounds that polynomially depend on the number of actions. Furthermore, these results rely on having a fixed reward function, making their techniques not directly applicable to our case. %For the linear case, \cite{hanna2022contexts} provides a reduction from the original problem to one with a fixed action set and fixed reward function. Our work can be viewed as a generalization of their result to the adversarial reward setting. 

\paragraph{Contextual bandits with adversarial contexts and fixed reward functions (A-S)}
In this category, the most well-known results are in the linear setting \citep{chu2011contextual, abbasi2011improved, zhao2023variance}. Besides the linear case, previous work has investigated specific reward function classes \citep{russo2013eluder, li2022understanding, foster2018practical}.  Recently,  \cite{foster2020beyond} introduced a general approach to deal with general function classes with a finite number of actions, which has since been improved or extended by \cite{foster2021efficient, foster2021instance, zhang2022feel}. This category of problems is not directly comparable to the setting studied in this paper, but both capture a certain degree of non-stationarity of the environment.  

\paragraph{Contextual bandits with i.i.d. contexts and adversarial reward functions (S-A)}
This is the category which our work falls into. Several oracle efficient algorithms that require simulators have been proposed for general policy classes \citep{rakhlin2016bistro, syrgkanis2016improved}. The oracle they use (i.e., the empirical risk minimization, or ERM oracle), however, is not generally implementable in an efficient manner. For the linear case, the first computationally efficient algorithm is by \cite{neu2020efficient}, under the assumption that the context distribution is known. This is followed by \cite{olkhovskaya2023first} to obtain refined data-dependent bounds. A series of works \citep{neu2021online, luo2021policy, dai2023refined, sherman2023improved} apply similar techniques to linear MDPs, but when specialized to linear contextual bandits, they all assume known context distribution, or access to a simulator, or only achieves a regret no better than $\otil(T^{\nicefrac{5}{6}})$. The work of \cite{kong2023improved} also studies linear MDPs; when specialized to contextual bandits, they obtain a regret bound of $\otil(T^{\nicefrac{4}{5}}+\poly(\frac{1}{\lambda}))$ without a simulator but with a computationally inefficient algorithm and an undesired inverse dependence on the smallest eigenvalue of the covariance matrix. 
Related but simpler settings have also been studied. The sleeping bandit problem with stochastic arm availability and adversarial reward \citep{kleinberg2010regret, kanade2009sleeping, saha2020improved} is a special case of our problem where the context is always a subset of standard unit vectors. Another special case is the combinatorial semi-bandit
problem with stochastic action sets and adversarial reward \citep{neu2014online}. While these are special cases, the regret bounds in these works are all worse than $\otil(\poly(d)\sqrt{T})$. Therefore, our result also improves upon theirs. \footnote{For combinatorial semi-bandit problems, our algorithm is not as computationally efficient as \cite{neu2014online}, which can handle exponentially large action sets. }

\paragraph{Contextual bandits with adversarial contexts and adversarial reward functions (A-A)}
When both contexts and reward functions are adversarial, there are computational \citep{kanade2014learning} and oracle-call \citep{hazan2016computational} lower bounds  showing that no sublinear regret is achievable unless the computational cost scales polynomially with the size of the policy set. Even for the linear case, \cite{neu2020efficient} argued that the problem is at least as hard as online learning a one-dimensional threshold function, for which sublinear regret is impossible. For this challenging category, besides using the inefficient EXP4 algorithm, previous work makes stronger assumptions on the contexts \citep{syrgkanis2016efficient} or resorts to alternative benchmarks such as dynamic regret \citep{luo2018efficient, chen2019new} and approximate regret \citep{emamjomeh2021adversarial}. 

\paragraph{Lifting and exploration bonus for high-probability adversarial linear bandits}  
Our technique is related to those obtaining high-probability bounds for linear bandits. Early development in this line of research only achieves computational efficiency when the action set size is small \citep{bartlett2008high} or only applies to special action sets such as two-norm balls \citep{abernethy2009beating}. 
Recently, near-optimal high-probability bounds for general convex action sets have been obtained by lifting the problem to a higher dimensional one, which allows for a computationally efficient way to impose bonuses \citep{lee2020bias, zimmert2022return}. 
The lifting and the bonus ideas we use are inspired by them, though for different purposes. However, due to the extra difficulty arising in the contextual case, currently we only obtain a computationally efficient algorithm when the action set size is small. 

%2. Existing approach requires simulator: this is due to the construction of covariance matrix. There was also a barrier to improve over $T^{2/3}$ even under simulator -- when using standard exponential weight, it cannot tolerate losses that has large magnitude

%3. sqrt{K} is achieved by Dai et al. under simulator. Their first solution log-barrier has poly(A) regret dependence, the second solution use reduced-magnitude loss estimator which requires additional calls to simulators. Dai and Sherman also has simulator-free setting, which simply spend online sample to estimate covariance. The best bound obtained from this is $T^{5/6}$. 

%4. Other related work: Stochastic context + adversarial loss: oracle efficient + simulator ($T^{2/3}$), transductive setting. Gergo's sleeping semi-bandits with FTPL and its related work. 

%5. stochastic context + stochastic loss: monster, minimonster, RegElimination, Falcon

%6. adversarial context + stochastic loss: LinUCB, RegCB, SquareCB

%7. Computational hardness on adversarial contexts + adversarial loss? 

\subsection{Computational Complexity}
Our main algorithm is based on log-determinant barrier optimization similar to \cite{foster2020adapting,zimmert2022return}. 
Computing its action distribution is closely related to computing the D-optimal experimental design \citep{khachiyan1990complexity}. 
Per step, this is shown to require $\otil(|\calA_t|\poly(d))$ computational and $\otil(\log(|\calA_t|)\poly(d))$ memory complexity \cite[Proposition 1]{foster2020adapting}, where $|\calA_t|$ is the action set size at round $t$. The computational bottleneck comes from  
(approximately) maximizing a quadratic function over the action set. 
It is an open question whether linear optimization oracles or other type of oracles can lead to efficient implementation of our algorithm for continuous action sets. 

In the literature, there are few linear context bandit algorithms that provably avoid $|\calA|$ computation per round. The LinUCB algorithm \citep{chu2011contextual, abbasi2011improved} suffers from the same quadratic function maximization issue, and therefore is computationally comparable to our algorithm. The SquareCB.Lin algorithm by \cite{foster2020adapting} is based on the same log-determinant barrier optimization. Another recent algorithm by \cite{zhang2022feel} only admits an efficient implementation for continuous action sets in the Bayesian setting but not in the frequentist setting (though they provided an efficient heuristic implementation in their experiments). The Thompson sampling algorithm by 
\cite{agrawal2013thompson}, which has efficient implementation, also relies on well-specified Gaussian prior. The only work that we know can avoids $|\calA|$ computation in the frequentist setting is 
\cite{zhu2022contextual}, but their technique is only known to handle the A-S setting.

\section{Preliminaries}
We study the adversarial linear contextual bandit problem where the loss vectors are selected fully adversarially and the per-round action set (i.e. the context) is drawn from a fixed distribution. %each round $t$ is generated from a fixed distribution $D$. Let $supp(D)$ be the support set of $D$. 
The learner and the environment interact in the following way. Let $\mathbb{B}_2^d$ be the L2-norm unit ball in $\mathbb{R}^d$. 

For $t = 1, \cdots, T$, 
\begin{enumerate}
    \item The environment decides an adversarial loss vector $y_t \in \mathbb{B}_2^d$, and  generates a random action set (i.e., context) $\calA_t\subset \mathbb{B}_2^d$ from a fixed distribution $D$ independent from anything else. 
    \item The learner observes $\calA_t$, and (randomly) chooses an action $a_t \in \calA_t$. 
    \item  The learner receives the loss $\ell_t\in[-1,1]$ with $\E[\ell_t] = \langle a_t, y_t \rangle$.
\end{enumerate}
A policy $\pi$ is a mapping which, given any action set $\calA\subset\mathbb{R}^d$, maps it to an element in the convex hull of $\calA$. We use $\pi({\calA})$ to refer to the element that it maps $\calA$ to. The learner's \emph{regret with respect to policy $\pi$} is defined as the expected performance difference between the learner and policy $\pi$: 
%to choose appropriate actions such that the cumulative loss $\sum_{t=1}^T \langle a_t, y_t \rangle $ is as small as possible. The learner is given a policy set $\Pi$, in which every policy $\pi\in\Pi$ is a mapping from an arbitrary action set $\calA$ to an element in the action set $\calA$. We use $\pi(\calA)$ to denote this mapping, and thus $\pi(\calA)\in\calA$. The learner's regret with respect to policy $\pi\in\Pi$ is defined as 
\begin{equation*}
\Reg(\pi) =  \E\left[\sum_{t=1}^T  \langle a_t, y_t \rangle - \sum_{t=1}^T \langle \pi(\calA_t), y_t \rangle  \right]
\end{equation*}
where the expectation is taken over all randomness from the environment ($y_t$ and $\calA_t$) and from the learner ($a_t$).    
The \emph{pseudo-regret} (or just \emph{regret}) is defined as $\Reg=\max_{\pi}\Reg(\pi)$, where the maximization is taken over all possible policies. %With abuse of notation, %we also use $u$ to  denote the expected action played by policy $u$ under action set distribution $D$, that is, $u=\E_{\calA\sim D}[u^{\calA}]\in \mathbb{R}^d$. As will be clear, $\E_{\calA\sim D}[u^{\calA}]$ is sufficient to determine the expected loss of policy $u$.  % We assume $\calA_t\subset \mathbb{B}_2^d$ and $\|y_t\|_2\leq 1$ with probability $1$. 

%Inherited from adversarial bandits literature, the comparator of our regret is the best fixed policy $\pi^\star$ in a given policy set $\Pi$. Namely, the learner tries to minimize the pseudo regret defined as
%\begin{equation*}
%\Reg = \max\limits_{\pi \in \Pi} \sum_{t=1}^T\mathbb{E}\left[ \langle a_t, y_t \rangle - \langle \pi(\calA_t), y_t \rangle \right]
%\end{equation*}

% For any $\calA \in supp(D)$, let $\Delta(\calA)$ be the space of probability measures on $\calA$ with the Borel $\sigma$-algebra. Let $\mathcal{F}_t = \sigma(\calA_s, a_s, \forall s \le t)$ be the $\sigma$-algebra for round $t$. Define $\mathbb{E}_t[\cdot] = \mathbb{E}[\cdot | \mathcal{F}_{t-1}]$ and $\mathbb{P}_t[\cdot] = \mathbb{P}[\cdot|\mathcal{F}_{t-1}]$. 

\textbf{Notations } For any matrix $A$, we use $\lambda_{\max}(A)$ and $\lambda_{\min}(A)$ to denote the maximum and minimum eigenvalues of $A$, respectively. We use $\tr(A)$ to denote the trace of matrix $A$. %We use $A \succeq 0$ to denote that $A$ is a positive semi-definite matrix. For matrix $A,B$, $A \succeq B$ means $A - B \succeq 0$. 
For any action set $\calA$, let $\Delta(\calA)$ be the space of probability measures on $\calA$. Let $\mathcal{F}_t = \sigma(\calA_s, a_s, \forall s \le t)$ be the $\sigma$-algebra at round $t$. Define $\mathbb{E}_t[\cdot] = \mathbb{E}[\cdot | \mathcal{F}_{t-1}]$.  
Given a differentiable convex
function $F: \mathbb{R}^d \rightarrow \mathbb{R} \cup \{\infty\}$, the Bregman divergence with respect to $F$ is defined as $D_F(x, y) = F(x) - F(y) - \langle \nabla F(y), x-y\rangle$. Given a positive semi-definite (PSD) matrix $A$,  for any vector $x$, define the norm generated by $A$ as $\|x\|_A = \sqrt{x^\top A x}$. 
For any context $\calA \subset \mathbb{R}^d$ and $p \in \Delta(\calA)$, define $\mu(p) = \mathbb{E}_{a \sim p}[a]$ and $\Cov(p) = \mathbb{E}_{a \sim p}[(a-\mu(p))(a-\mu(p))^{\top}]$. 
For any $a$, define the lifted action $\pmb{a} = (a, 1)^{\top}$ and the lifted covariance matrix
$\widehat{\Cov}(p) = \mathbb{E}_{a \sim p}[\pmb{a}\pmb{a}^{\top}] =  \mathbb{E}_{a \sim p}\begin{bmatrix} aa^{\top} & a\\ a^{\top} & 1 \end{bmatrix} = \begin{bmatrix} \Cov(p) + \mu(p)\mu(p)^{\top} & \mu(p)\\ \mu(p)^{\top} & 1 \end{bmatrix}$. We use \textbf{bold} matrices to denote matrices in the lifted space (e.g., in \pref{alg: FTRL} and \pref{def: additional notation}).

\section{Follow-the-Regularized-Leader with the Log-Determinant Barrier}

In this section, we present our main algorithm, \pref{alg: FTRL}. This algorithm can be viewed as instantiating an individual Follow-The-Regularized-Leader (FTRL) algorithm on each action set (\pref{line: FTRL}), with all FTRLs sharing the same loss vectors. This perspective has been taken by previous works \cite{neu2020efficient, olkhovskaya2023first} and simplifies the understanding of the problem. The rationale comes from the following calculation due to \cite{neu2020efficient}: for any policy $\pi$ that may depend on $\calF_{t-1}$, 
\begin{align*}
    \E_t\left[ \langle \pi(\calA_t), y_t
 \rangle\right] &= \E_{\calA_t}\left[\E_{y_t} \left[\langle \pi(\calA_t), y_t
 \rangle~|~\calF_{t-1}\right]\right] = \E_{\calA_0}\left[\E_{y_t} \left[\langle \pi(\calA_0), y_t
 \rangle~|~\calF_{t-1}\right]\right] = \E_t\left[ \langle \pi(\calA_0), y_t
 \rangle\right]
\end{align*}
where $\calA_0$ is a sample drawn from $D$ independent of all interaction history. This allows us to calculate the regret as 
\begin{align}
     \E\left[ \sum_{t=1}^T \langle \pi_t(\calA_t) - \pi(\calA_t), y_t \rangle \right] = \E\left[ \sum_{t=1}^T \langle \pi_t(\calA_0) - \pi(\calA_0), y_t \rangle \right]  \label{eq: ghost view}
\end{align}
where $\pi_t$ is the policy used by the learner at time $t$. Note that this view does not require the learner to simultaneously ``run'' an algorithm on every action set since the learner only needs to calculate the policy on $\calA$ whenever $\calA_t=\calA$. 
In the regret analysis, in view of \pref{eq: ghost view}, it suffices to consider a single fixed action set $\calA_0$ drawn from $D$ and bound the regret on it, even though the learner may never execute the policy on it. This $\calA_0$ is called a ``ghost sample'' in \cite{neu2020efficient}.

\begin{algorithm}[t]
    \caption{Logdet-FTRL for linear contextual bandits}
    \label{alg: FTRL}
    \textbf{Definitions}: $F(\pmb{H}) = -\log\det{(\pmb{H})}, \eta_t = \frac{1}{64d\sqrt{t}}, \alpha_t = \frac{d}{\sqrt{t}}, \beta_t = \frac{100(d+1)^3\log(3T)}{t-1}$. \\
    %\nl Initialize $\hat{\gamma}_0 - \eta \pmb{\hatSigma}_0^{-1} = 0$ and $\hat{\gamma}_1 = 0$\\
    \nl  \For{$t=1, 2, \ldots$}{
    \nl For all $\calA$,  define $\pmb{H}_t^{\calA} = \argmin\limits_{\pmb{H}\in \mathcal{H}^{\calA}} \sum_{s=1}^{t-1}\langle \pmb{H}, \hat{\gamma}_s - \alpha_s \pmb{\hat\Sigma}_s^{-1}\rangle + \frac{F(\pmb{H})}{\eta_t} $.   \label{line: FTRL} \\
    \nl For all $\calA$,  define $p_t^{\calA}\in\Delta(\calA)$ such that $\pmb{H}^\calA_t = \widehat{\Cov}(p_t^{\calA})$.  \label{line: define pt} \\
    \nl  Receive $\calA_t$ and sample $a_t\sim p_t^{\calA_t}$. \label{line: receive At}\\
    \nl      Observe $\ell_t\in[-1,1]$ with  $\E[\ell_t] = a_t^\top y_t$ and construct $\hat{y}_t = \hat\Sigma_t^{-1}(a_t -\hat{x}_t)\ell_t$, 
    where \label{line: loss esimator construct}
          \begin{align}
    \hat{x}_t &= \emp \E_{a\sim p_t^{\calA_{\tau}}} [ a ], \ \  
              \hatH_t = \emp \E_{a\sim p_t^{\calA_{\tau}}}\left[(a-\hat{x}_t)(a-\hat{x}_t)^\top\right], \ \  
              \hatSigma_t = \hatH_t + \beta_t I.  \nonumber %\label{eq: several def}
         \end{align}
    \nl      Define $\pmb{\hatH}_{t} = \frac{1}{t-1}\sum_{\tau = 1}^{t-1}\pmb{H}_t^{\calA_{\tau}}$ and $\pmb{\hatSigma}_t = \pmb{\hatH}_t + \beta_t  \pmb{I}$ and 
    $
         \hat{\gamma}_t = \begin{bmatrix}0& \frac{1}{2}\hat{y}_t\\ \frac{1}{2}\hat{y}_t^\top&0 \end{bmatrix}
    $. 
    \\
    (If $t=1$, define $\hatSigma_t^{-1} $ and $\pmb{\hatSigma}_t^{-1}$ as zeros). 
    }
\end{algorithm}

\subsection{The lifting idea and the execution of \pref{alg: FTRL}}

Our algorithm is built on the \alglog algorithm developed by \cite{zimmert2022return} for high-probability adversarial linear bandits, which lifts the original $d$-dimensional problem over the feature space to a $(d+1)\times(d+1)$ one over the covariance matrix space, with the regularizer being the negative log-determinant function. In our case, we instantiate an individual \alglog on each action set. The motivation behind \cite{zimmert2022return} to lift the problem to the space of covariance matrix is that it casts the problem to one in the positive orthant, which allows for an easier way to construct the \emph{bonus} term that is crucial to compensate the variance of the losses, enabling a high-probability bound in their case. 
In our case, we use the same technique to introduce the bonus term, but the goal is to compensate the \emph{bias} resulting from the estimation error in the covariance matrix (see \pref{sec: bonus term}). This bias only appears in our contextual case but not in the linear bandit problem originally considered in \cite{zimmert2022return}. 

As argued previously, we can focus on the learning problem over a fixed action set $\calA$, and our algorithm operates in the lifted space of covariance matrices $\mathcal{H}^{\calA} = \{\widehat{\Cov}(p): p \in \Delta(\calA) \} \subset \mathbb{R}^{(d+1)\times (d+1)}$. For this space, we define the lifted loss $\gamma_t = \begin{bmatrix}0& \frac{1}{2}y_t\\ \frac{1}{2}y_t^\top &0 \end{bmatrix} \in \mathbb{R}^{(d+1)\times(d+1)}$ so that $\langle \widehat{\Cov}(p),\gamma_t \rangle = \mathbb{E}_{a\sim p} [a^{\top}y_t] = \langle \mu(p), y_t\rangle $ and thus the loss value in the lifted space (i.e., $\langle \widehat{\Cov}(p),\gamma_t \rangle$) is the same as that in the original space (i.e., $\langle \mu(p), y_t\rangle$).  

In each round $t$, the FTRL on $\calA$ outputs a lifted covariance matrix $\pmb{H}_t^{\calA}\in\calH^{\calA}$ that corresponds to a probability distribution $p_t^{\calA}\in \Delta(\calA)$ such that $\widehat{\Cov}(p_t^{\calA}) = \pmb{H}_t^{\calA}$ (\pref{line: FTRL} and \pref{line: define pt}). Upon receiving $\calA_t$, the learner samples an action from $p_t^{\calA_t}$ and the agent constructs the loss estimator $\hat{y}_t$ (\pref{line: loss esimator construct}). Similarly to the construction of $\gamma_t$, we define the lifted loss estimator $\hat{\gamma}_t = \begin{bmatrix}0& \frac{1}{2}\hat{y}_t\\ \frac{1}{2}\hat{y}_t^\top&0 \end{bmatrix}$ which makes $\langle \widehat{\Cov}(p),\hat{\gamma}_t \rangle = \mathbb{E}_{a\sim p} [a^{\top}\hat{y}_t] = \langle \mu(p), \hat{y}_t\rangle $. The lifted loss estimator, along with the \emph{bonus} term $-\alpha_t \pmb{\hatSigma_t}^{-1}$, is then fed to the FTRL on all $\calA$'s. The purpose of the bonus term will be clear in \pref{sec: bonus term}. 

In the rest of this section, we use the following notation in addition to those defined in \pref{alg: FTRL}.  
\begin{definition}\label{def: additional notation}
   Define $
       x_t^\calA=\E_{a\sim p^{\calA}_t}[a], \  x_t=\E_{\calA\sim D}[x_t^{\calA}],  \ \  
       H_t^{\calA} = \E_{a\sim p^{\calA}_t}[(a-\hat{x}_t)(a-\hat{x}_t)^\top], \     H_t= \E_{\calA\sim D}[H^{\calA}_t], \  \pmb{H}_t = \E_{\calA\sim D}[\pmb{H}^\calA_t].  
   $
   Let $p_\star^{\calA}\in\Delta(\calA)$ be the action distribution used by the benchmark policy on $\calA$, and define 
   $
        u^{\calA} = \E_{a\sim p_\star^\calA}[a], \  u = \E_{\calA\sim D}[u^\calA], \   \pmb{U}^{\calA} = \E_{a\sim p_\star^\calA}[\pmb{a}\pmb{a}^\top], \  \pmb{U} = \E_{\calA\sim D}[\pmb{U}^{\calA}]. 
   $
   Notice that the $x_t^{\calA}$ and $u^{\calA}$ defined here is equivalent to the $\pi_t(\calA)$ and $\pi(\calA)$ in \pref{eq: ghost view}, respectively.  
\end{definition}

\subsection{The construction of loss estimators and feature covariance matrix estimators} \label{sec: estimator contruction}
Our goal is to make $\hat{y}_t$ in \pref{line: loss esimator construct} an estimator of $y_t$ with controllable bias and variance. If the context distribution is known (as in \cite{neu2020efficient}), then a standard unbiased estimator of $y_t$ is 
\begin{align}
    \hat{y}_t = \Sigma_t^{-1} a_t\ell_t, \qquad \text{where}\quad \Sigma_t = \E_{\calA\sim D}\E_{a\sim p_t^{\calA}}\left[ aa^\top \right]. \label{eq: ideal yt hat}
\end{align}
To see its unbiasedness, notice that $\E[a_t\ell_t] = \E_{\calA\sim D
} \E_{a\sim p^{\calA}_t}[aa^\top y_t]$ and thus $\E[\hat{y}_t]=y_t$. This $\hat{y}_t$, however, can have a variance that is inversely related to the smallest eigenvalue of the covariance matrix $\hatSigma_t$, which can be unbounded in the worst case. This is the main reason why \cite{neu2020efficient} does not achieve the optimal bound, and requires the bias-variance-tradeoff techniques in \cite{dai2023refined} to close the gap. %Notice that we intentionally offset the actions by its mean $x_t$, which is actually not necessary for constructing an unbiased estimator (that is, \pref{eq: ideal yt hat} is still unbiased if we replace all $a-x_t$ in it by $a$), but this will be crucial in our algorithm to get a smaller variance. 
When the context distribution is unknown but the learner has access to a simulator \citep{luo2021policy, dai2023refined, sherman2023improved, kong2023improved}, the learner can draw free contexts to estimate the covariance matrix $\hatSigma_t$ up to a very high accuracy without interacting with the environment, making the problem close to the case of known context distribution. 

Challenges arise when the learner has no knowledge about the context distribution and there is no simulator. In this case, there are two natural ways to estimate the covariance matrix under the current policy. One is to draw new samples from the environment, treating the environment like a simulator. This approach is essentially taken by all previous work studying linear models in the ``S-A'' category. 
However, this is very expensive, and it causes the simulator-equipped bound $\sqrt{T}$ in \cite{dai2023refined} to deteriorate to the simulator-free bound $T^{\nicefrac{5}{6}}$ at best (see \pref{app: comparison} for details). The other is to use the contexts received in time $1$ to $t$ to estimate the covariance matrix under the policy at time $t$. 
This demands a very high efficiency in reusing the contexts samples, and existing ways of constructing the covariance matrix and the accompanied analysis by \cite{dai2023refined, sherman2023improved} are insufficient to achieve the near-optimal bound even with context reuse. This necessitates our tighter construction of the covariance matrix estimator and tighter concentration bounds for it.  

Our construction of the loss estimator (\pref{line: loss esimator construct}) is 
\begin{align}
     \hat{y}_t = \hat{\Sigma}_t^{-1}(a_t-\hat{x}_t)\ell_t \qquad \text{where}\quad \hatSigma_t =
\E_{\calA\sim \hat{D}_t}\E_{a\sim p_t^{\calA}}\left[(a-\hat{x}_t)(a-\hat{x}_t)^\top\right] + \beta_t I \label{eq: est yt hat}
     %\hat{x}_t = \emp \E_{a\sim p_t^{\calA_{\tau}}} [ a ], \quad 
      %        \hatSigma_t = \emp \E_{a\sim p_t^{\calA_{\tau}}}\left[(a-\hat{x}_t)(a-\hat{x}_t)^\top\right] + \beta_t I
\end{align}
where $\hat{D}_t=\text{Uniform}\{\calA_1, \calA_2, \ldots, \calA_{t-1}\}$, $\hat{x}_t = \E_{\calA\sim \hat{D}_t}, \E_{a\sim p_t^{\calA}}[a]$, and $\beta_t=\otil(d^3/t)$.  Comparing \pref{eq: est yt hat} with \pref{eq: ideal yt hat}, we see that besides using the empirical context distribution $\hat{D}_t$ in place of the ground truth $D$ and adding a small term $\beta_t I$ to control the smallest eigenvalue of the covariance matrix, we also centralize the features by $\hat{x}_t$, an estimation of the mean features under the current policy. The centralization is important in making the bias $y_t-\hat{y}_t$ appear in a nice form that can be compensated by a bonus term. 
The estimator might seem problematic on first sight, because $p_t^{\calA}$ is strongly dependent on $\hat{D}_t$, which rules out canonical concentration bounds.
We circumvent this issue by leveraging the special structure of $p_t$ in \pref{alg: FTRL}, which allows for a union bound over a sufficient covering of all potential policies (\pref{app: union bounds}).
The analysis on the bias of this loss estimator is also non-standard, which is the key to achieve the near-optimal bound . In the next two subsections, we explain how to bound the \emph{bias} of this loss estimator (\pref{sec: bias term}), and how the \emph{bonus} term can be used to compensate the bias (\pref{sec: bonus term}). 

\subsection{The bias of the loss estimator} \label{sec: bias term}
Since the true loss vector is $y_t$ and we use the loss estimator $\hat{y}_t$ in the update, there is a bias term emerging in the regret bound at time $t$: 
\begin{align*}
   \E_{t} \left[ \langle x_t^{\calA_0} - u^{\calA_0}, y_t - \hat{y}_t \rangle \right]= \E_t\left[ \langle x_t - u, y_t - \hat{y}_t \rangle\right] = \E_t\left[\left( x_t - u\right)^\top \left( I - \hatSigma_t^{-1} (a_t-\hat{x}_t)a_t^\top \right)y_t\right]
\end{align*}
where definitions of $x_t^{\calA}, u^{\calA}, x_t, u$ can be found in \pref{def: additional notation}, and we use the definition of $\hat{y}_t$ in \pref{eq: est yt hat} in the last equality. 
Now taking expectation over $\calA_t$ and $a_t$ conditioned on $\calF_{t-1}$, we can further bound the expectation in the last expression by 
\begin{align}
    &(x_t - u)^\top \left(I - \hatSigma_t^{-1} H_t \right) y_t - (x_t - u)^\top \hatSigma_t^{-1}\left(x_t - \hat{x}_t\right)\hat{x}_t^\top y_t  \nonumber\\
    &\leq \|x_t-u\|_{\hatSigma_t^{-1}} \|(\hatSigma_t - H_t)y_t\|_{\hatSigma_t^{-1}} + \|x_t-u\|_{\hatSigma_t^{-1}} \|x_t-\hat{x}_t\|_{\hatSigma_t^{-1}} \label{eq: bias term decompose}
\end{align}
(see \pref{def: additional notation} for the definition of $H_t$). The two terms $\|(\hatSigma_t - H_t)y_t\|_{\hatSigma_t^{-1}}$ and $\|x_t-\hat{x}_t\|_{\hatSigma_t^{-1}}$ in \pref{eq: bias term decompose} are related to the error between the empirical context distribution $\hat{D}_t=\text{Uniform}\{\calA_1, \ldots, \calA_{t-1}\}$ and the true distribution $D$. 
We handle them through novel analysis and bound both of them by $\otil\big(\sqrt{d^3/t}\big)$. See \pref{lem: local norm concentration for vectors 2}, \pref{lem: local norm concentration for matrix 2}, \pref{lem: local norm concentration for vectors}, and \pref{lem: local norm concentration for matrix} for details. The techniques we use in these lemmas surpass those in \cite{dai2023refined, sherman2023improved}. 
As a comparison, a similar term as $\|(\hatSigma_t - H_t)y_t\|_{\hatSigma_t^{-1}}$ is also presented in Eq.~(16) of \cite{dai2023refined} and Lemma B.5 of \cite{sherman2023improved} when bounding the bias. While their analysis uses off-the-shelf matrix concentration inequalities, our analysis expands this expression by its definition, and applies concentration inequalities for \emph{scalars} on individual entries. Overall, our analysis is more tailored for this specific expression. Previous works ensure that this term can be bounded by $\order(\sqrt{\beta})$ after collecting $\order(\beta^{-2})$ new samples (Lemma 5.1 of \cite{dai2023refined} and Lemma B.1 of \cite{sherman2023improved}), we are able to bound it by $\order(1/\sqrt{t})$ only using $t$ samples that the learner received up to time $t$. This essentially improves their $\order(\beta^{-2})$ sample complexity bound to $\order(\beta^{-1})$. See \pref{app: comparison} for detailed comparison with  \cite{dai2023refined} and \cite{sherman2023improved}.

Now we have bounded the regret due to bias of $\hat{y}_t$ by the order of $\sqrt{d^3/t}\|x_t-u\|_{\hatSigma_t^{-1}} $. The next problem is how to mitigate this term. This is also a problem in previous work \citep{luo2021policy, dai2023refined, sherman2023improved}, and it has become clear that this can be handled by incorporating \emph{bonus} in the algorithm. 

\subsection{The bonus term}\label{sec: bonus term}
To handle a bias term in the form of $\|x_t-u\|_{\hatSigma_t^{-1}}$, we resort to the idea of \emph{bonus}. 
%From the analysis in \pref{sec: bias term}, if we only use $\hat{y}_t$ to update the FTRL algorithms, then the regret bound would be of the following form (omitting $d$ dependence): 
%\begin{align*}
%    \Reg &= \E\left[\sum_{t=1}^T (x_t - u)^\top \hat{y}_t\right] + \E\left[\sum_{t=1}^T (x_t - u)^\top (y_t - \hat{y}_t)\right] 
%    \leq \otil(\sqrt{T}) + \otil\left(\E\left[\sum_{t=1}^T \frac{1}{\sqrt{t}}\|x_t-u\|_{\hatSigma_t^{-1}} \right]\right)
%\end{align*}
%where we assume for now that the FTRL algorithm can give us a $\sqrt{T}$ bound for the first term. The second term that comes from the bias, however, cannot be bounded by $\sqrt{T}$ in general, and can be arbitrarily large. The same issue has appeared also in previous works \cite{luo2021policy, dai2023refined, sherman2023improved}, and the solution is to incorporate an exploration bias in the update. 
To illustrate this, suppose that instead of feeding $\hat{y}_t$ to the FTRLs, we feed $\hat{y}_t - b_t$ for some $b_t$. Then this would give us a regret bound of the following form: 
\begin{align}
    \Reg &= \E\left[\sum_{t=1}^T \langle x_t - u, \hat{y}_t - b_t\rangle\right] + \E\left[\sum_{t=1}^T \langle x_t - u,  y_t - \hat{y}_t\rangle\right] + \E\left[\sum_{t=1}^T \langle x_t - u,  b_t\rangle \right] \nonumber \\
    &\lesssim \otil(d^2\sqrt{T}) + \E\left[\sum_{t=1}^T \sqrt{\frac{d^3}{t}}\|x_t-u\|_{\hatSigma_t^{-1}}\right]  + \E\left[\sum_{t=1}^T \langle x_t-u,  b_t\rangle\right]   \label{eq: change of measure}
\end{align}
where we assume that FTRL can give us $\otil(d^2\sqrt{T})$ bound for the loss sequence $\hat{y}_t-b_t$. Our hope here is to design a $b_t$ such that $\langle x_t-u, b_t\rangle$ provides a negative term that can be used to cancel the bias term $\sqrt{d^3/t}\|x_t-u\|_{\hatSigma_t^{-1}}$ in the following manner: 
\begin{align}
     \text{bias} + \text{bonus} = \sum_{t=1}^T \left(\sqrt{\frac{d^3}{t}}\|x_t-u\|_{\hatSigma_t^{-1}} + \langle x_t - u,  b_t\rangle\right) \lesssim \otil(d^2\sqrt{T}).  \label{eq: desired property}
\end{align}
which gives us a $\otil(d^2\sqrt{T})$ overall regret by \pref{eq: change of measure}.  
%which, after combined with \pref{eq: change of measure}, gives us a $\otil(d^2\sqrt{T})$ bound. 
This approach relies on two conditions to be satisfied. First, we have to find a $b_t$ that makes \pref{eq: desired property} hold. Second, we have to ensure that the FTRL algorithm achieves a $\otil(d^2\sqrt{T})$ bound under the loss sequence $\hat{y}_t - b_t$.

To meet the first condition, we take inspiration from \cite{zimmert2022return} and lift the problem to the space of covariance matrix in $\mathbb{R}^{(d+1)\times(d+1)}$. Considering the bonus term $\alpha_t \pmb{\hatSigma}_t^{-1}$ in the lifted space, we have
\begin{align}
     \langle \pmb{H}_t - \pmb{U}, \alpha_t  \pmb{\hatSigma}_t^{-1} \rangle = \alpha_t \tr(\pmb{H}_t\pmb{\hatSigma}^{-1}_t) -  \alpha_t  \tr(\pmb{U}\pmb{\hatSigma}_t^{-1})  \label{eq: lifted space bonus}
\end{align}
Using \pref{lem: main concentration} and \pref{cor: special lifted trace bound}, we can upper bound \pref{eq: lifted space bonus} by $\order\left(d \alpha_t \right) -\frac{\alpha_t}{4}\|u-\hat{x}_t\|_{\hat{\Sigma}_t^{-1}}^2$. This gives  %Though the negative part does not match the bias $\sqrt{\frac{d^3}{t}}\|x_t-u\|_{\hatSigma_{t}^{-1}}$, cancellation still happens since 
\vspace*{-5pt}
\begin{align*}
    \text{bias} + \text{bonus} &\leq \sum_{t=1}^T \left(\sqrt{\frac{d^3}{t}}\|x_t-u\|_{\hatSigma_t^{-1}} +  d\alpha_t - \frac{\alpha_t}{4}\|\hatx_t-u\|_{\hatSigma_t^{-1}}^2 \right) \\
    &\leq \otil(d^2\sqrt{T}) + \sum_{t=1}^T \sqrt{\frac{d^3}{t}}\|x_t-\hatx_t\|_{\hatSigma_t^{-1}} + \sum_{t=1}^T \left(\sqrt{\frac{d^3}{t}}\|\hatx_t-u\|_{\hatSigma_t^{-1}}  - \frac{\alpha_t}{4}\|\hatx_t-u\|_{\hatSigma_t^{-1}}^2\right).
\end{align*}
%By \pref{lem:matrix concentration}, the first term can be upper bounded by $\E \left[\sum_{t=1}^T 8\eta \tr(\pmb{H}_t\pmb{H}^{-1}_t)\right] = 8\eta(d+1)T=\order(\sqrt{T})$. By \pref{cor: special lifted trace bound}, the second term can be upper bounded by $-\E\left[\sum_{t=1}^T \eta \tr(\pmb{U}\pmb{\hatSigma}_t^{-1})\right] \leq -\frac{\eta}{3}\|u-\hat{x}_t\|_{\hat{\Sigma}_t^{-1}}^2 \leq \frac{\eta}{6}\|u-x_t\|^2_{\hatSigma_t^{-1}} + \frac{\eta}{3}\|x_t-\hat{x}_t\|_{\hatSigma_t^{-1}}^2$. Overall, the right-hand side of \pref{eq: lifted space bonus} is upper bounded by 
%\begin{align*}
%     &\order\left(d\sum_{t=1}^T \alpha_t \right) -\frac{\alpha_t}{4}\|u-\hat{x}_t\|_{\hat{\Sigma}_t^{-1}}^2 
%     \leq \otil(d^2\sqrt{T}) - \sum_{t=1}^T \frac{\alpha_t}{4}\|u-x_t\|^2_{\hatSigma_t^{-1}} + \sum_{t=1}^T \frac{\alpha_t}{4} \|x_t-\hat{x}_t\|^2_{\hatSigma_t^{-1}} \\
%     &\leq \otil(d^2\sqrt{T})  - \sum_{t=1}^T \sqrt{\frac{d^3}{t}}\|u-x_t\|_{\hatSigma_t^{-1}} + \sum_{t=1}^T 
% \frac{d^3}{\alpha_t t} \leq \otil(d^2\sqrt{T})  - \sum_{t=1}^T \sqrt{\frac{d^3}{t}}\|u-x_t\|_{\hatSigma_t^{-1}}, 
%\end{align*}
%where in the second inequality we use \pref{lem: local norm concentration for vectors} and AM-GM. 
Using \pref{lem: local norm concentration for vectors} to bound the second term above by $\otil(\sum_t d^3/t)=\otil(d^3)$, and AM-GM to bound the third term by $\otil(\sum_t d^3/(t\alpha_t))=\otil(d^2\sqrt{T})$, we get \pref{eq: desired property}, through the help of lifting. 

To meet the second condition, we have to analyze the regret of FTRL under the loss $\hat{y}_t-b_t$. The key is to show that the bonus $\alpha_t \pmb{\hatSigma}_t^{-1}$ introduces small \emph{stability term} overhead. Thanks to the use of the logdet regularizer and its self-concordance property, the extra stability term introduced by the bonus can indeed be controlled by the order $\sqrt{T}$. The key analysis is in \pref{lem: bonus stability}. 

Previous works rely on exponential weights \citep{luo2021policy, dai2023refined, sherman2023improved} rather than \alglog, which comes with the following drawbacks. 1) In \cite{luo2021policy, sherman2023improved} where exponential weights is combined with standard loss estimators, the bonus introduces large stability term overhead. Therefore, their bound can only be $T^{\nicefrac{2}{3}}$ at best even with simulators. 2) In \cite{dai2023refined} where exponential weights is combined with magnitude-reduced loss estimators, the loss estimator for action $a$ can no longer be represented as a simple linear function $a^\top \hat{y}_t$. Instead, it becomes a complex non-linear function. This restricts the algorithm's potential to leverage linear optimization oracle over the action set and achieve computational efficiency. 

\subsection{Overall regret analysis}\label{sec: overall regret}
With all the algorithmic elements discussed above, now we give a formal statement for our regret guarantee and perform a complete regret analysis. 
Our main theorem is the following. 
\begin{theorem}
     \pref{alg: FTRL} ensures $\Reg\leq \order(d^2\sqrt{T}\log T)$. 
\end{theorem}
\begin{proof}[Proof sketch]
Let $\calA_0$ be drawn from $D$ independently from all the interaction history between the learner and the environment. Recalling the definitions in \pref{def: additional notation}, we have 
\begin{align*}
&\Reg=\mathbb{E}\left[\sum_{t=1}^{T} \langle a_t-u^{\calA_t}, y_t \rangle\right] = \mathbb{E}\left[\sum_{t=1}^{T} \langle \pmb{H}_t^{\calA_t}-\pmb{U}^{\calA_t}, \gamma_t \rangle\right] =  \mathbb{E}\left[\sum_{t=1}^{T} \langle \pmb{H}_t^{\calA_0}-\pmb{U}^{\calA_0}, \gamma_t \rangle\right]
\\
&\le 
\underbrace{\mathbb{E}\left[\sum_{t=1}^{T} \langle \pmb{H}_t^{\calA_0}-\pmb{U}^{\calA_0}, \gamma_t - \hat{\gamma}_t \rangle\right]}_{\textbf{Bias}}+
 \underbrace{\mathbb{E}\left[\sum_{t=1}^{T} \langle\pmb{H}_t^{\calA_0}-\pmb{U}^{\calA_0}, \alpha_t\pmb{\hatSigma}^{-1}_t \rangle\right]}_{\textbf{Bonus}} 
 + 
 \underbrace{\mathbb{E}\left[\sum_{t=1}^{T} \langle\pmb{H}_t^{\calA_0}-\pmb{U}^{\calA_0},\hat{\gamma}_t -  \alpha_t\pmb{\hatSigma}^{-1}_t \rangle\right]}_{\textbf{FTRL-Reg}} 
 \end{align*}
 %By the standard FTRL analysis (\pref{lem: FTRL guarantee}), we can further bound 
 %\begin{align*}
 %&\textbf{FTRL-Reg} \leq 
 %\underbrace{\mathbb{E}\left[\frac{F(\pmb{U}^{\calA_0}) - \min_{\pmb{H} \in \mathcal{H}^{\calA_0}}F(\pmb{H})}{\eta_T}\right]}_{\textbf{Penalty}} \\
 %& + \underbrace{\mathbb{E}\left[\sum_{t=1}^{T}\max\limits_{\pmb{H} \in \mathcal{H}^{\calA_0}}\langle \pmb{H}_t^{\calA_0} - \pmb{H}, \hat{\gamma}_{t} \rangle - \frac{D(\pmb{H}, \pmb{H}_t^{\calA_0})}{2\eta_t}\right]}_{\textbf{Stability-1}} + \underbrace{\mathbb{E}\left[\sum_{t=1}^{T}\max\limits_{\pmb{H} \in \mathcal{H}^{\calA_0}}\langle \pmb{H}_t^{\calA_0} - \pmb{H}, -\alpha_t\pmb{\hatSigma}^{-1}_t \rangle - \frac{D(\pmb{H}, \pmb{H}_t^{\calA_0})}{2\eta_t}\right]}_{\textbf{Stability-2}} 
%\end{align*}
Each term can be bounded as follows:
\begin{itemize}[itemsep=-0.5em, topsep=0em, partopsep=0em]
    \item $\textbf{Bias} \leq  \order(d^2\sqrt{T}\log T) + \frac{1}{4}\sum_{t=1}^T \alpha_t\|u - x_t\|_{\hatSigma_t^{-1}}^2$ (discussed in \pref{sec: bias term}).  
    \item $\textbf{Bonus}\leq \order(d^2\sqrt{T}\log T) - \frac{1}{4}\sum_{t=1}^T \alpha_t \|u-x_t\|_{\hatSigma_t^{-1}}^2$ (discussed in \pref{sec: bonus term}).  
     \item $\textbf{FTRL-Reg} \leq \order(d^2\sqrt{T}\log T)$. 
    %\item $\textbf{Stability-1}\lesssim \sqrt{T}$. See \pref{lem: stability term for loss}. 
    %\item $\textbf{Stability-2}\lesssim \sqrt{T}$. See \pref{lem: bonus stability}. 
\end{itemize}
Combining all terms gives the desired bound.  The complete proof is provided in \pref{app: regret analysis}. 
\end{proof}

\subsection{Handling Misspecification}\label{sec: misspecification}
In this subsection, we show how our approach naturally handles the case when the expectation of the loss cannot be exactly realized by a linear function but with a misspecification error. In this case, we assume that the expectation of the loss is given by $\E[\ell_t|a_t=a]=f_t(a)$ for some $f_t: \mathbb{R}^d\rightarrow [-1,1]$, and the realized loss $\ell_t$ still lies in $[-1,1]$. We define the following notion of misspecification (slightly more refined than that in \cite{neu2020efficient}):
\begin{assumption}[misspecification] 
    $
         \sqrt{\frac{1}{T}\sum_{t=1}^T \inf_{y\in \mathbb{B}_2^d}\sup_{\calA\in\supp(D)}\sup_{a\in\calA}(f_t(a) - \langle a, y \rangle)^2} \leq \eps. 
    $
\end{assumption}
%In words, in each round, suppose that $\langle a, y_t\rangle$ is the best linear function that approximates $f_t(a)$, and the maximum squared error $\max_{a\in\calA}|f_t(a)-\langle a,y_t\rangle|^2$ under the context distribution $D$ is $\eps_t^2$. Then our assumption is $\sqrt{\frac{1}{T}\sum_{t=1}^T \eps_t^2}\leq \eps$. 

Based on previous discussions, the design idea of \pref{alg: FTRL} is to 1) identify the bias of the loss estimator, and 2) add necessary bonus to compensate the bias. When there is misspecification, this design idea still applies. The difference is that now the loss estimator $\hat{y}_t$ potentially has more bias due to misspecification. Therefore, the bias becomes larger by an amount related to $\eps$. Consequently, we need to enlarge bonus (raising $\alpha_t$) to compensate it. Due to the larger bonus, we further need to tune down the learning rate $\eta_t$ to make the algorithm stable. Overall, to handle misspecification, when $\eps$ is known, it boils down to using the same algorithm (\pref{alg: FTRL}) with adjusted $\alpha_t$ and $\eta_t$. The case of unknown $\eps$ can be handled by the standard meta-learning technique \emph{Corral} \citep{agarwal2017corralling, foster2020adapting,luo2022corralling}. We defer all details to \pref{app: misspecification} and only state the final bound here. 
\begin{theorem}\label{thm: misspeicifcation thm}
    Under misspecification, there is an algorithm ensuring $\Reg\leq \otil(d^2\sqrt{T} + \sqrt{d}\eps T)$, without knowing $\eps$ in advance.  
\end{theorem}

\section{Linear EXP4}
To tighten the $d$-dependence in the regret bound, we can use the computationally inefficient algorithm EXP4 \citep{auer2002nonstochastic}. The original regret bound for EXP4 has a polynomial dependence on the number of actions, but here we take the advantage of the linear structure to show a bound that only depends on the feature dimension $d$. The algorithm is presented in \pref{alg: linear exp4}. 
\begin{algorithm}[H]
    \caption{Linear EXP4}
    \label{alg: linear exp4}
    \textbf{input}: $\Pi,\eta,\gamma$. \\
     \For{$t=1, 2, \ldots$}{
         Receive $\calA_t\subset \mathbb{R}^d$. \\
         Construct $\nu_t\in\Delta(\calA_t)$ such that
         $
              \max_{a\in\calA_t} \|a\|^2_{G_t^{-1}} \leq d
         $, 
         where $G_t=\E_{a\sim \nu_t}[aa^\top]$. 
         Set 
         \begin{align*}
              P_{t,\pi} = \frac{\exp\left(-\eta\sum_{s=1}^{t-1}\hatell_{s,\pi}\right)}{\sum_{\pi'\in\Pi} \exp\left(-\eta\sum_{s=1}^{t-1}\hatell_{s,\pi'}\right) } 
         \end{align*}
         and define $p_{t,a} = \sum_{\pi\in\Pi} P_{t,\pi} \ind\{\pi(\calA_t)=a\} $\,. \\
         Sample $a_t\sim\tildep_t= (1-\gamma)p_t+\gamma\nu_t$ and receive $\ell_t\in[-1,1]$ with $\E[\ell_t] = \langle a_t, y_t\rangle$. \\
         Construct $\forall \pi\in\Pi
         $: 
         $
             \hatell_{t,\pi}=   \inner{\pi(\calA_t),\tildeH_t^{-1}a_t\ell_t}
         $,
         where
         $
             \tildeH_t = \E_{a\sim\tildep_t}[aa^\top]\,.
         $
     }
\end{algorithm}

To run \pref{alg: linear exp4}, we restrict ourselves to a finite policy class. The policy class we use in the algorithm is the set of linear policies defined as 
\begin{align}
     \Pi = \left\{\pi_\theta:~ \theta\in\Theta, ~~\pi_\theta(\calA) = \argmin_{a\in\calA} a^\top \theta \right\}   \label{eq: definition of Pi}
\end{align}
where $\Theta$ is an $1$-net of $[-T,T]^{d}$. The next theorem shows that this suffices to give us near-optimal bounds for our problem. The proof is given in \pref{app: EXP4}. 

%Assume for now we are dealing with a finite-sized policy class $\Pi$, then we can run \pref{alg: linear exp4}, a linearized version of EXP4. We deal with general policy set in \pref{sec: reduction from finite to infinite}. 

\begin{theorem}\label{thm: exp4 guarantee}
With $\gamma = 2d\sqrt{(\log T)/T}$ and $\eta=\sqrt{(\log T)/T}$, \pref{alg: linear exp4} with the policy class defined in \pref{eq: definition of Pi} guarantees $\Reg = \order\left(d\sqrt{T\log T}\right)$.

\end{theorem}

Note that this result technically also holds in the ``A-A'' category with respect to the policy class defined in \pref{eq: definition of Pi}.
However, this policy class is \emph{not} necessarily a sufficient cover of all policies of interest when the contexts and losses are adversarial.
%In the S-A setting following the same argument at \citet{hanna2022contexts}, we can construct a policy cover $\Pi$ of size $\ln\Pi=O(d\ln(T))$ such that 

%\begin{align*}
%\max_{\ell,\theta\in B_2}\min_{\pi\in\Pi}\E_{A\sim D}|\inner{\pi_\theta(A)-\pi(A),\ell}|\leq \frac{1}{T}\,.
%\end{align*}
%Combining this with linear exp4 achieves directly optimal $O(d\sqrt{T\ln T})$ regret for S-A. 

\section{Conclusions}
We derived the first efficient algorithm that obtains $\sqrt{T}$ regret in contextual linear bandits with stochastic action sets in the absence of a simulator or prior knowledge on the distribution. As a side result, we obtained the first computationally efficient $\poly(d)\sqrt{T}$ algorithm for adversarial sleeping bandits with general stochastic arm availabilities.
We believe the techniques in this paper will be useful for improving results for simulator-free linear MDPs as well.

%%%%%%%%%%%%%%%%%%%%%%%%%%%%%%%%%%%%%%%%%%%%%%%%%%%%%%%%%%%%
\newpage
\bibliography{ref}
\bibliographystyle{plainnat}

\newpage
\appendix
\appendixpage

{
\startcontents[section]
\printcontents[section]{l}{1}{\setcounter{tocdepth}{2}}
}

\newpage
\section{Summary of Notation}
We summarize the notations that have been defined in \pref{alg: FTRL} and \pref{def: additional notation}. 
 
% \rho_t =   \frac{(d+1)\log(\frac{(d+1)}{\delta}) + (d+1)^3\log(48\eta(d+1)(t-1))}{t-1}
     \begin{align*}
            \beta_t &=  \Theta\left(\frac{(d+1)^3\log(T/\delta)}{t-1}\right)\\
             \hat{x}_t &= \emp \E_{a\sim p_t^{\calA_{\tau}}} [ a ]  \\
             \hat{H}_t &= \emp \E_{a\sim p_t^{\calA_{\tau}}}\left[(a-\hat{x}_t)(a-\hat{x}_t)^\top\right]\\
             \pmb{\hat{H}}_t &= \emp\E_{a\sim p_t^{\calA_{\tau}}} \begin{bmatrix}aa^\top &a \\ a^\top &1\end{bmatrix} 
             = \begin{bmatrix}
                  \hat{H}_t + \hat{x}_t  \hat{x}_t^\top  & \hat{x}_t \\
                  \hat{x}_t^\top & 1 
             \end{bmatrix} \\
             \hatSigma_t &= \hat{H}_t + \beta_t I \\
             \pmb{\hatSigma}_t &= \pmb{\hat{H}}_t + \beta_t  \pmb{I}
             = \begin{bmatrix}
                \hatSigma_t + \hat{x}_t  \hat{x}_t^\top  & \hat{x}_t \\
                  \hat{x}_t^\top  & 1 + \beta_t
             \end{bmatrix} \\
        x_t &= \E_{\calA\sim\calD}\E_{a\sim p^\calA_t}[a] \\
        H_t &= \E_{\calA\sim \calD}\E_{a\sim p_t^{\calA}}\left[(a-{\hat{x}_t})(a-{\hat{x}_t})^\top\right] \\
        \pmb{H}_t &= \E_{\calA\sim \calD} \E_{a\sim p_t^{\calA}} \begin{bmatrix}
            aa^\top & a \\
            a^\top & 1
        \end{bmatrix}
    \end{align*}

\section{Auxiliary Lemmas}
\begin{lemma}[FTRL regret bound, Lemma 18 of \cite{dann2023best}]\label{lem: FTRL guarantee}
   Let $\Omega\subset \mathbb{R}^d$ be a convex set, $g_1, \ldots, g_T  \in \mathbb{R}^d$, and $\eta_1,\ldots, \eta_T >0$.  Then the FTRL update 
   \begin{align*}
       w_t = \argmin_{w\in\Omega}\left\{ \left\langle w, \sum_{\tau=1}^{t-1} g_\tau\right\rangle + \frac{1}{\eta_t}\psi(w)\right\}
   \end{align*}
   ensures for any $u\in\Omega$ and $\eta_0>0$, 
   \begin{align*}
       &\sum_{t=1}^T  \langle w_t - u, g_t\rangle  \\
       &\leq \underbrace{\frac{\psi(u) - \min_{w\in\Omega}\psi(w)}{\eta_0} + \sum_{t=1}^T (\psi(u) - \psi(w_t))\left(\frac{1}{\eta_t} - \frac{1}{\eta_{t-1}}\right) }_{\textbf{\textup{Penalty}}} + \underbrace{\sum_{t=1}^T \left(\max_{w\in\Omega} \langle w_t - w, g_t\rangle - \frac{D_{\psi}(w,w_t)}{\eta_t}\right)}_{\textbf{\textup{Stability}}}. 
   \end{align*}
   When $\eta_0, \eta_1, \ldots, \eta_T$ is non-increasing, the penalty term can further be upper bounded by 
   \begin{align*}  
       \textbf{\textup{Penalty}}\leq \frac{\psi(u)-\min_{w\in\Omega}\psi(w)}{\eta_T}. 
   \end{align*} 
\end{lemma}

\begin{lemma}[Bernstein's inequality] \label{lem: Bernstein} Let $X_1, \cdots, X_n$ be iid random variables; let $\mathbb{E}[X]$ be the expectation and $\Var(X)$ be the variance of these random variables. If for any $i$, $|X_i - \mathbb{E}[X_i]| \le R$, then with probability of at least $1-\delta$,
\begin{equation*}
    \left|\frac{1}{n}\sum_{i=1}^n X_i - \mathbb{E}[X]\right| \le \sqrt{\frac{4\Var(X)\log{\frac{2}{\delta}}}{n}} + \frac{4R\log{\frac{2}{\delta}}}{3n}. 
\end{equation*}
\end{lemma}

\begin{lemma}[Hoeffding's inequality]\label{lem: Hoeffding}
 Let $X_1, \cdots, X_n$ be iid random variables; let $a \le X_i \le b$ and let $\mathbb{E}[X]$ be the expectation. Then with probability of at least $1-\delta$,
\begin{equation*}
    \left|\frac{1}{n}\sum_{i=1}^n X_i - \mathbb{E}[X]\right| \le (b-a)\sqrt{\frac{1}{2n}\log(\frac{2}{\delta})}
\end{equation*}
\end{lemma}

Given $F(X) = -\log{\det(X)}$, $D^2F(X) = X^{-1} \otimes X^{-1}$ where $\otimes$ is the Kronecker product. For any matrix $A = \begin{bmatrix}
    a_1 & a_2 & \cdots & a_n
\end{bmatrix}$, let $\text{vec}(A) = \begin{bmatrix}
    a_1 \\ \vdots \\ a_n
\end{bmatrix}$ which vectorizes matrix $A$ to a column vector by stacking the columns $A$. The second order directional derivative for $F$ is $D^2F(X)[A, A] = \text{vec}(A)^T\left(X^{-1} \otimes X^{-1} \right)\text{vec}(A) = \tr(A^\top X^{-1}AX^{-1})$. We define $\|A\|_{\nabla^2F(X)} = \sqrt{\tr(A^\top X^{-1}AX^{-1})}$ and $\|A\|_{\nabla^{-2}F(X)} = \sqrt{\tr(A^\top XAX)}$. It is a pseudo-norm, and more discussion can be found in Appendix D of \cite{zimmert2022pushing}. In the following analysis, we will only use one property of this pseudo-norm which is similar to the Holder inequality. 

\begin{lemma} \label{lem:matrix norm holder}
For any two symmetric matrices $A,B$ and positive definite matrix $X$,
\begin{equation*}
\langle A, B \rangle \le \|A\|_{\nabla^{2}F(X)}\|B\|_{\nabla^{-2}F(X)}
\end{equation*}
\end{lemma}

\begin{proof}
Since $(X \otimes X)^{-1} = X^{-1} \otimes X^{-1}$, from Holder inequality, we have
\begin{align*}
    \langle A, B\rangle = \langle \text{vec}(A), \text{vec}(B) \rangle \le \|\text{vec}(A)\|_{X^{-1} \otimes X^{-1} } \|\text{vec}(B)\|_{(X^{-1} \otimes X^{-1})^{-1} } = \|A\|_{\nabla^{2}F(X)}\|B\|_{\nabla^{-2}F(X)}
\end{align*}

% From Von Neumann's trace inequality, for any $d \times d$ matrices $P,Q$, we have $|\tr(PQ)| \le \sum_{i=1}^d \sigma_i(P) \sigma_i(Q)$ where $\sigma_i(P)$ and $\sigma_i(Q)$ is the $i$-th largest singular value of $P$ and $Q$. Let $\lambda_i(P^\top P)$ and $\lambda_i(Q^\top Q)$ be the $i$-th largest eigenvalue of $P^{\top}P$ and $Q^{\top}Q$, respectively. From the definition of singular value, we have $|\tr(PQ)| \le \sum_{i=1}^d \sqrt{\lambda_i(P^\top P) \lambda_i(Q^\top Q)}$. Thus, we have
% \begin{align*}
% \langle A, B \rangle^2 &= \tr(AX^{-1}XB)^2
% \\&\le \left(\sum_{i=1}^d \sqrt{\lambda_i(X^{-1}AAX^{-1}) \lambda_i(XBBX)}\right)^2
% \\&\le \left(\sum_{i=1}^d \lambda_i(X^{-1}AAX^{-1}) \right) \left(\sum_{i=1}^d \lambda_i(XBBX) \right)\tag{Cauchy-Schwarz inequality}
% \\&= \tr(X^{-1}AAX^{-1})\tr(XBXB)
% \\&= \|A\|_{\nabla^{2}F(X)}^2\|B\|_{\nabla^{-2}F(X)}^2
% \end{align*}
\end{proof}

\section{Concentration Inequalities} \label{app:concentration}
The goal of this section is to show \pref{lem: local norm concentration for vectors} and \pref{lem: local norm concentration for matrix}, which are key to bound the bias term. 
We first introduce a useful lemma from \cite{dai2023refined}, which will be used later to prove our concentration bounds.

\subsection{General Concentration Inequalities}
\begin{lemma}[Lemma A.4 in \cite{dai2023refined}] \label{lem: A.4} Let $H_1,H_2,\ldots,H_n$ be i.i.d. PSD matrices such that $\E[H_i]=H$, $H_i\preceq I$ almost surely and $H\succeq \frac{1}{dn}\log \frac d\delta I$. Then with probability $1-\delta$,
\begin{equation*}
\frac 1n\sum_{i=1}^n H_i-H \succeq -\sqrt{\frac dn\log \frac d\delta}H^{1/2}
\end{equation*}
\end{lemma}

\begin{corollary} \label{cor: useful concentration}
Let $H_1,H_2,\ldots,H_n$ be i.i.d. PSD matrices such that $\E[H_i]=H$ and  $H_i\preceq cI$ almost surely for some positive constant $c$. Let $\hatH = \frac{1}{n}\sum_{i=1}^n H_i$, then with probability $1-\delta$,
\begin{equation}
\hat{H} + \frac{3c}{2}\cdot\frac{d}{n}\log\left(\frac{d}{\delta}\right)I \succeq \frac{1}{2}H
\label{eq: approx concentration}
\end{equation}
\end{corollary}

\begin{proof}
A simple corollary of \pref{lem: A.4} under the condition of \pref{lem: A.4} is that 
\begin{align}
    &\frac{1}{n}\sum_{i=1}^n H_i - H \succeq -\sqrt{\frac{d}{n}\log \frac{d}{\delta}} H^{1/2} \succeq -\frac{1}{2}H - \frac{d}{2n}\log\left(\frac{d}{\delta}\right)I  \nonumber \\
    \Rightarrow\ \  & \frac{1}{n}\sum_{i=1}^n H_i + \frac{d}{2n}\log\left(\frac{d}{\delta}\right)I  \succeq \frac{1}{2}H,   \label{eq: in coro} 
\end{align}
where we use that $H^{\frac{1}{2}}\preceq \frac{k}{2}H + \frac{1}{2k}$ for any $k>0$. 

Now consider the condition of this corollary. We first consider the case where $\frac{d}{n}\log(\frac{d}{\delta})\le 1$. In this case, we apply \pref{eq: in coro} with $H_i'=\frac{1}{2c}H_i + \frac{d}{2n}\log(\frac{d}{\delta})I$, which satisfies the condition for \pref{eq: in coro} to hold. This gives 
\begin{align*}
    &\frac{1}{n}\sum_{i=1}^n \left(\frac{1}{2c}H_i + \frac{d}{2n}\log\left(\frac{d}{\delta}\right)I\right) + \frac{d}{2n}\log\left(\frac{d}{\delta}\right)I \succeq \frac{1}{2}\left(\frac{1}{2c}H + \frac{d}{2n}\log\left(\frac{d}{\delta}\right)I\right) \\
    \Rightarrow \ \ & \hat{H} + \frac{3c}{2}\cdot\frac{d}{n}\log\left(\frac{d}{\delta}\right)I \succeq \frac{1}{2}H  
 \end{align*}
with probability at least $1-\delta$. 
When $\frac{d}{n}\log(\frac{d}{\delta})>1$. \pref{eq: approx concentration} is trivial because $\frac{1}{2}H\preceq \frac{c}{2}I\preceq \frac{c}{2}\cdot \frac{d}{n}\log(\frac{d}{\delta})I$.

\end{proof}

\subsection{Concentration Inequalities under a Fixed Policy $p$}
In this subsection, we establish concentration bounds for a \emph{fixed} policy $p$ (with $p^{\calA}\in\Delta(\calA)$ denoting the action distribution it uses over $\calA$) over i.i.d. contexts. The results in this subsection are preparation for \pref{app: union bounds} where we take union bounds over policies. 

%we consider a fixed policy denoted as $p$, with $p^{\calA}\in\Delta(\calA)$ specifies the distribution over actions the policy uses on $\calA$. 

The setting and notation to be used in this subsection are defined in \pref{def: simple definition}. 
\begin{definition}\label{def: simple definition}
     Let $\{\calA_1, \ldots, \calA_n\}$ be i.i.d. context samples drawn from $D$. Let $\hat{D}$ be the uniform distribution over $\{\calA_1, \ldots, \calA_n\}$. 

     Over this set of context samples, define for any policy $p$, 
    \begin{align*}
          x(p)&= \E_{\calA\sim D} \E_{a\sim p^\calA} [a], \\
          \hat{x}(p) &= \E_{\calA\sim \hat{D}} \E_{a\sim p^\calA} [a], \\
          H(p) &= \E_{\calA\sim D}\E_{a\sim p^\calA} \left[(a-\hatx(p))(a-\hatx(p))^\top \right], \\
          \hatH(p) &= \E_{\calA\sim \hat{D}}\E_{a\sim p^\calA} \left[(a-\hatx(p))(a-\hatx(p))^\top \right], \\
          \pmb{H}(p) &= \E_{\calA\sim D}\E_{a\sim p^\calA} \left[\pmb{a}\pmb{a}^\top \right], \\
          \pmb{\hatH}(p) &= \E_{\calA\sim \hat{D}}\E_{a\sim p^\calA} \left[\pmb{a}\pmb{a}^\top \right], \\
          \hatSigma(p) &= \hatH(p) + \beta I, \\
          \pmb{\hatSigma}(p) &= \pmb{\hatH}(p) + \beta \pmb{I},
    \end{align*}
    where $\beta =  \frac{5d\log(6d/\delta)}{n}$. 
\end{definition}

%In this subsection, since we consider a fixed policy $p$, we will omit $p$ in all of the above notations in the proof for simplicity.

\begin{lemma} \label{lem: hatH > H lemma}
Under the setting of \pref{def: simple definition}, for any fixed $p$, with probability at least $1-\delta$, 
    \begin{align*}
        \hat{H}(p) + \frac{4d\log(6d/\delta)}{n} I &\succeq \frac{1}{2}H(p), \\
        \pmb{\hat{H}}(p) + \frac{3d\log(d/\delta)}{n} \pmb{I} &\succeq \frac{1}{2}\pmb{H}(p).  
    \end{align*}
\end{lemma}

\begin{proof}
In this proof, we use $\hat{x}, x, \hat{H}, H, \pmb{\hatH}, \pmb{H}$ to denote $\hat{x}(p), x(p), \hat{H}(p), H(p), \pmb{\hatH}(p), \pmb{H}(p)$  since $p$ is fixed throughout the proof. 

Since $\|a\| \le 1$, $\pmb{H} \preceq  2I$ and $\pmb{\hatH} \preceq 2I$. Thus, we can directly apply \pref{cor: useful concentration} with $c = 2$ to get with probability $1-\frac{\delta}{3}$
\begin{equation*}
    \pmb{\hat{H}} + \frac{3d\log(3d/\delta)}{n} \pmb{I} \succeq \frac{1}{2}\pmb{H}.  
\end{equation*}
To prove the first inequality, we first decompose $H$ and $\hatH$
\begin{align}
     H &= \E_{\calA\sim D}\E_{a\sim p^\calA} \left[(a-\hatx)(a-\hatx)^\top \right]   \nonumber \\
     &= \E_{\calA\sim D}\E_{a\sim p^\calA} \left[(a-x + x-\hatx)(a-x + x-\hatx)^\top \right]  \nonumber 
 \\
     &= \E_{\calA\sim D}\E_{a\sim p^\calA} \left[(a-x)(a-x)^\top \right] + (x-\hatx)(x-\hatx)^\top \tag{because $\E_{\calA\sim D}\E_{a\sim p^\calA}(a-x)=0$}  \\
     \  \label{eq: H formula 1} \\
     \hat{H} &= \E_{\calA\sim \hat{D}}\E_{a\sim p^\calA} \left[(a-\hatx)(a-\hatx)^\top \right]  \nonumber  \\
     &= \E_{\calA\sim \hat{D}}\E_{a\sim p^\calA} \left[(a-x + x-\hatx)(a-x + x-\hatx)^\top \right]  \nonumber \\
     &= \E_{\calA\sim \hat{D}}\E_{a\sim p^\calA} \left[(a-x)(a-x)^\top \right] - (x-\hatx)(x-\hatx)^\top \tag{because $\E_{\calA\sim \hat{D}}\E_{a\sim p^\calA}(a-x)=\hatx-x$} \\
      \label{eq: H formula 2}
\end{align}
From Hoeffding inequality (\pref{lem: Hoeffding}) and union bound, with probability $1-\frac{\delta}{3}$, for all $k \in [d]$, we have
\begin{equation*}
    |\e_k^\top x - \e_k^\top \hat{x}| \le \sqrt{\frac{1}{2n}\log\left(\frac{6d}{\delta}\right)}, 
\end{equation*}
which implies that $\e_k^\top (x-\hatx)(x-\hatx)^\top \e_k \leq \frac{1}{2n}\log(\frac{6d}{\delta})$ for all $k$, and thus 
\begin{equation}
    (x-\hatx)(x-\hatx)^\top \preceq \frac{1}{2n}\log\left(\frac{6d}{\delta}\right)I. \label{eq: H formula 3} 
\end{equation}
By directly applying \pref{cor: useful concentration} with $c = 2$, we get with probability at least $1-\frac{\delta}{3}$,
\begin{equation*}
    \E_{\calA\sim \hat{D}}\E_{a\sim p^\calA} \left[(a-x)(a-x)^\top \right] + \frac{3d\log(3d/\delta)}{n} I \succeq \frac{1}{2}\E_{\calA\sim D}\E_{a\sim p^\calA} \left[(a-x)(a-x)^\top \right] 
\end{equation*}
Further using \pref{eq: H formula 1}, \pref{eq: H formula 2} and \pref{eq: H formula 3}, we get with probability at least $1-\frac{2\delta}{3}$, 
\begin{equation*}
\hat{H} + \frac{4d\log(6d/\delta)}{n} I \succeq \frac{1}{2}H
\end{equation*}
Taking union bound for both inequality finishes the proof.
\end{proof}

\begin{lemma}\label{lem: local norm concentration for vectors 2}
Under the setting of \pref{def: simple definition}, for any fixed policy $p$, with probability at least $1-\order(\delta)$, 
\begin{align*}
     \|x(p) - \hat{x}(p)\|^2_{\hatSigma(p)^{-1}} & \leq \order\left(\frac{d\log(d/\delta)}{n}\right)
\end{align*}
\end{lemma}
\begin{proof}
In this proof, we use $\hat{x}, x, \hat{H}, H, \pmb{\hatH}, \pmb{H}, \hatSigma, \pmb{\hatSigma}$ to denote $\hat{x}(p)$, $x(p)$, $\hat{H}(p)$, $H(p)$, $\pmb{\hatH}(p)$, $\pmb{H}(p)$, $\hatSigma(p)$, $\pmb{\hatSigma}(p)$  since $p$ is fixed throughout the proof. 

We first rewrite $H$.
\begin{align}
 H &= \E_{\calA\sim D}\E_{a\sim p^\calA} \left[(a-\hatx)(a-\hatx)^\top \right]   \nonumber \\
     &= \E_{\calA\sim D}\E_{a\sim p^\calA} \left[(a-x + x-\hatx)(a-x + x-\hatx)^\top \right]  \nonumber 
 \\
     &= \E_{\calA\sim D}\E_{a\sim p^\calA} \left[(a-x)(a-x)^\top \right] + (x-\hatx)(x-\hatx)^\top \tag{because $\E_{\calA\sim D}\E_{a\sim p^\calA}(a-x)=0$} \\
     \label{eq: H formula 12}
\end{align}
To simplify analysis, we perform diagonalization. 
Suppose that $\E_{\calA\sim D}\E_{a\sim p^\calA}[(a-x)(a-x)^\top]$ admits the following eigen-decomposition: 
\begin{align*}
    \E_{\calA\sim D}\E_{a\sim p^\calA}[(a-x)(a-x)^\top] = V\Lambda V^\top 
\end{align*}
where $V$ is an orthogonal matrix and $\Lambda$ is a diagonal matrix. 
By \pref{lem: hatH > H lemma} and the definition of $\beta$ in \pref{def: simple definition}, we have with probability $1-\delta$, 
\begin{align*}
    \hatSigma &\succeq \frac{1}{2}H + \rho I \succeq \frac{1}{2}V\Lambda V^\top + \rho I
\end{align*}
with some $\rho=\Theta\left(\frac{d\log(d/\delta)}{n}\right)$, 
where the second inequality is by \pref{eq: H formula 12}. 
Thus, 
\begin{align*}
 \|x -\hat{x}\|^2_{\hatSigma^{-1}}  &= (x - \hat{x})^{\top} \hatSigma^{-1}  (x - \hat{x}) 
 \\&\le (x - \hat{x})^{\top} \left(\frac{1}{2}V\Lambda V^\top + \rho I\right)^{-1}  (x - \hat{x}) 
 \\&= (\hat{x} -x)^{\top}V \left(\frac{1}{2}\Lambda + \rho I\right)^{-1} V^\top (\hat{x} - x). 
\end{align*}
Define 
\begin{equation*}
\Delta_k  = \e_k^\top V^\top(\hat{x} - x) = \frac{1}{n}\sum_{i=1}^n \underbrace{\e_k^\top V^\top \E_{a\sim p^{\calA_i}} [a]}_{\textbf{Define as $Z_k^{(i)}$}}- \underbrace{\e_k^\top V^\top \E_{\calA\sim D} \E_{a\sim p^\calA} [a]}_{\textbf{Define as $Z_k$}}
\end{equation*}
Since $\E_{\calA_i\sim D}\left[Z_k^{(i)}\right] = Z_k$, by  Bernstein's inequality, with probability at least $1-\delta$, we have 
\begin{align}
    |\Delta_{k}|  \leq \order\left(\sqrt{\frac{\Var(Z_k^{(i)})\log(d/\delta)}{n}} + \frac{\log(d/\delta)}{n}\right) \label{eq: apply bern}
\end{align}
for all $k$, where 
\begin{equation*}
    \Var(Z_k^{(i)}) = \E_{\calA\sim D}\left[\left(\e_k^\top V^\top \E_{a \sim p^{\calA}}[a] - \e_k^\top V^\top x\right)^2\right].
\end{equation*}

On the other hand, 
\begin{align*}
    \Lambda_{kk} &= \e_k^\top \E_{\calA\sim D}\E_{a\sim p^\calA}[V^\top(a-x)(a-x)^\top V ]\e_k
    \\&=  \E_{\calA\sim D}\E_{a\sim p^\calA}\left[\left(\e_k^\top V^\top a- \e_k^\top V^\top x\right)^2\right]. 
\end{align*}
From Jensen's inequality, 
\begin{equation*}
\Lambda_{kk} = \E_{\calA\sim D}\E_{a\sim p^\calA}\left[\left(\e_k^\top V^\top a- \e_k^\top V^\top x\right)^2\right] \ge \E_{\calA\sim D}\left[\left(\e_k^\top V^\top \E_{a\sim p^\calA}[a]- \e_k^\top V^\top x\right)^2\right] = \Var(Z_k^{(i)}) 
\end{equation*}

Thus, 
\begin{align*}
\|x -\hat{x}\|^2_{\hatSigma^{-1}}  &\le (\hat{x} -x)^{\top}V \left(\frac{1}{2}\Lambda + \rho I\right)^{-1} V^\top (\hat{x} - x) 
\\&= \sum_{k=1}^d \frac{(\Delta_k)^2}{\frac{1}{2}\Lambda_{kk} + \rho}
\\&\le \order\left( \frac{\log(d/\delta)}{n}\sum_{k=1}^d \frac{ \Var(Z_k^{(i)}) + \frac{\log(d/\delta)}{n}}{\Lambda_{kk} + \rho}\right) \tag{by \pref{eq: apply bern}}
\\&\le \order\left( \frac{d\log(d/\delta)}{n} \right). \tag{$\Lambda_{kk} \ge  \Var(Z_k^{(i)}) $ and $\rho = \Theta( \frac{d\log(d/\delta)}{n})$}
\end{align*}

\end{proof}

\begin{lemma} \label{lem: local norm concentration for matrix 2}
Under the setting of \pref{def: simple definition}, for any fixed policy $p$, with probability at least $1-\order(\delta)$, 
\begin{align*}
     \|(\hatSigma(p) - H(p))y\|^2_{\hatSigma(p)^{-1}}  &\leq \order\left(\frac{d\log(d/\delta)}{n} \right)
\end{align*}
for any $y\in\mathbb{B}^d_2$. 
\end{lemma}

% \order\left(\frac{d^3\log{\frac{d\eta t}{\delta}}}{t}\right)
\begin{proof}
In this proof, we use $\hat{x}, x, \hat{H}, H, \pmb{\hatH}, \pmb{H}, \hatSigma, \pmb{\hatSigma}$ to denote $\hat{x}(p)$, $x(p)$, $\hat{H}(p)$, $H(p)$, $\pmb{\hatH}(p)$, $\pmb{H}(p)$, $\hatSigma(p)$, $\pmb{\hatSigma}(p)$  since $p$ is fixed throughout the proof.

First, we re-write $H$ and $\hat{H}$: 
\begin{align}
     H &= \E_{\calA\sim D}\E_{a\sim p^\calA} \left[(a-\hatx)(a-\hatx)^\top \right]   \nonumber \\
     &= \E_{\calA\sim D}\E_{a\sim p^\calA} \left[(a-x + x-\hatx)(a-x + x-\hatx)^\top \right]  \nonumber 
 \\
     &= \E_{\calA\sim D}\E_{a\sim p^\calA} \left[(a-x)(a-x)^\top \right] + (x-\hatx)(x-\hatx)^\top \tag{because $\E_{\calA\sim D}\E_{a\sim p^\calA}(a-x)=0$}  \\
     \  \label{eq: H formula} \\
     \hat{H} &= \E_{\calA\sim \hat{D}}\E_{a\sim p^\calA} \left[(a-\hatx)(a-\hatx)^\top \right]  \nonumber  \\
     &= \E_{\calA\sim \hat{D}}\E_{a\sim p^\calA} \left[(a-x + x-\hatx)(a-x + x-\hatx)^\top \right]  \nonumber \\
     &= \E_{\calA\sim \hat{D}}\E_{a\sim p^\calA} \left[(a-x)(a-x)^\top \right] - (x-\hatx)(x-\hatx)^\top \tag{because $\E_{\calA\sim \hat{D}}\E_{a\sim p^\calA}(a-x)=\hatx-x$}
\end{align}
Then, by definition (in \pref{def: simple definition}) and the calculation above, 
\begin{align*}
    &\hatSigma- H \\
    &= \hat{H} - H + \beta I \\
    &= \underbrace{\frac{1}{n}\sum_{i=1}^n \E_{a\sim p^{\calA_i}} \left[(a-x)(a-x)^\top \right] - \E_{\calA\sim D}\E_{a\sim p^\calA} \left[(a-x)(a-x)^\top \right]}_{\text{define this as }\Gamma}   -2(x-\hatx)(x-\hatx)^\top + \beta I. 
\end{align*}
Using $\|a+b+c\|^2\leq 3\|a\|^2 + 3\|b\|^2 + 3\|c\|^2$, we have 
\begin{align}
    \|(\hatSigma-H)y\|^2_{\hatSigma^{-1}}
    &\leq 3\|\Gamma y\|_{\hatSigma^{-1}}^2 + 12\|(x-\hatx)(x-\hatx)^\top y\|^2_{\hatSigma^{-1}} + \beta^2 \|y\|^2_{\hatSigma^{-1}}  \nonumber \\
    &\leq 3\|\Gamma y\|_{\hatSigma^{-1}}^2 + 12\|x-\hatx\|^2_{\hatSigma^{-1}} + \order(\beta).   \label{eq: three decompose}
\end{align}
The second and third term are bounded by $\order\left(\frac{d\log(d/\delta)}{n}\right)$ using \pref{lem: local norm concentration for vectors 2} and the definition of $\beta$, with probability at least $1-\order(\delta)$. Below, we further deal with the first term. To simplify analysis, we perform diagonalization. 
Suppose that $\E_{\calA\sim D}\E_{a\sim p^\calA}[(a-x)(a-x)^\top]$ admits the following eigen-decomposition: 
\begin{align*}
    \E_{\calA\sim D}\E_{a\sim p^\calA}[(a-x)(a-x)^\top] = V\Lambda V^\top 
\end{align*}
where $V$ is an orthogonal matrix and $\Lambda$ is a diagonal matrix. 
Then 
\begin{align}
\left\|\Gamma y\right\|_{\hatSigma^{-1}}^2 
&= y^\top\Gamma \hatSigma^{-1}\Gamma y
=(V^\top y)^\top (V^\top \Gamma V)(V^\top \hatSigma V)^{-1} (V^\top \Gamma V) (V^\top y).  \label{eq: tmp equation}
\end{align}
Below, we further deal with the $V^\top \Gamma V$ and $V^\top \Lambda V$ terms in \pref{eq: tmp equation}. By \pref{lem: hatH > H lemma}, with probability at least $1-\delta$, 
\begin{align*}    
    \hatSigma \succeq \frac{1}{2} H + \rho I \succeq \frac{1}{2}V\Lambda V^\top + \rho I,    
\end{align*}
for some $\rho=\Theta\left(\frac{d\log(d/\delta)}{n}\right)$, where we use \pref{eq: H formula} in the second inequality. Therefore, 
\begin{align}
    V^\top \hatSigma V \succeq \frac{1}{2}\Lambda + \rho I.   \label{eq: lower bounding middle term}
\end{align}

Next, denote $\Delta=V^\top \Gamma V$. By definition, it can be written as the following:  
\begin{align*}
    \Delta 
    = \frac{1}{n}\sum_{i=1}^n \underbrace{\E_{a\sim p^{\calA_i}} \left[V^\top (a-x)(a-x)^\top V\right]}_{\text{defining this as } \Lambda^{(i)}} - \underbrace{\E_{\calA\sim D}\E_{a\sim p^\calA} \left[V^\top (a-x)(a-x)^\top V\right]}_{=\Lambda} 
\end{align*}
with $\Lambda^{(i)}$ being i.i.d. samples with mean $\E[\Lambda^{(i)}]=\Lambda$. 
%\begin{align*}
%\\&\le (V^\top y)^\top\left(D_t + \beta_tI\right)\left(\frac{1}{8}H_t + \rho_t I \right)^{-1}\left(D_t + \beta_tI\right)y_t\tag{\pref{lem: main concentration}}
%\\&\le 2y^\top D_t\left(\frac{1}{8}H_t + \rho_t I \right)^{-1}D_t y_t + 2\beta_t^2 y_t^\top\left(\frac{1}{8}H_t + \rho_t I \right)^{-1}y_t\tag{For all positive definite $A,B,C$, $(A+B)C(A+B) \preceq 2ACA + 2BCB$}
%\\&\le 2\tr\left(D_t\left(\frac{1}{8}H_t + \rho_t I \right)^{-1}D_t\right) + 2\beta_t^2\lambda_{\max}\left(\left(\frac{1}{8}H_t + \rho_t I\right)^{-1}\right)\tag{$\|y\|_2\le 1$}
%\end{align*} 
While these are $d\times d$ matrices, we will apply concentration inequalities to individual entries. 

Let $\lambda_{ikh} = \e_k^\top \Lambda^{(i)} \e_h$ be the $(k,h)$-th entry of $\Lambda^{(i)}$. Notice that $\E[\lambda_{ikh}]=\e_k^\top \Lambda \e_h=\Lambda_{kh}$, the $(k,h)$-th entry of $\Lambda$.  

By Bernstein's inequality, with probability at least $1-\delta$, we have 
\begin{align}
    %|\Delta_{kk}| 
    %&= \left|\frac{1}{n}\sum_{i=1}^n  (\lambda_{ikk} - \Lambda_{kk})\right| \leq  \sqrt{\frac{\Var(\lambda_{ikk})\log(1/\delta)}{n}} + \frac{\log(1/\delta)}{n}.   \\
    |\Delta_{kh}| &= \left|\frac{1}{n}\sum_{i=1}^n (\lambda_{ikh} - \Lambda_{kh})\right| \leq \order\left(\sqrt{\frac{\Var(\lambda_{ikh})\log(d/\delta)}{n}} + \frac{\log(d/\delta)}{n}\right).   \label{eq: Bernstein result}
\end{align}
With the manipulations and notations above, we continue to bound \pref{eq: tmp equation} by 
\begin{align*}
     \|\Gamma y\|^2_{\hatSigma^{-1}}& = y'^\top \Delta(V^\top \hatSigma V)^{-1}\Delta y'    \tag{let $y'=V^\top y$}\\
     &\leq 2y'^\top \Delta \left(\Lambda + \rho I\right)^{-1} \Delta y'  \tag{by \pref{eq: lower bounding middle term}}\\ 
     &\leq 2\tr\left(\Delta \left(\Lambda + \rho I\right)^{-1} \Delta\right)  
\end{align*}
By direct expansion and the fact that $\Lambda$ is diagonal, 
\begin{align}
    \tr\left(\Delta \left(\Lambda + \rho I\right)^{-1} \Delta\right) &= \sum_{k=1}^d  
    \left(\Delta \left(\Lambda + \rho I\right)^{-1} \Delta\right)_{kk} \nonumber \\
    &= \sum_{k=1}^d \sum_{h=1}^d  \frac{\Delta_{kh}\Delta_{hk}}{\Lambda_{hh} + \rho} \nonumber\\
    &\leq \order\left(\sum_{k=1}^d\sum_{h=1}^d \frac{1}{\Lambda_{hh} + \rho} \left(\frac{\Var(\lambda_{ikh})\log(d/\delta)}{n} + \frac{\log^2(d/\delta)}{n^2}\right)\right) \tag{by \pref{eq: Bernstein result}}\\
    &\leq \order\left(\sum_{k=1}^d\sum_{h=1}^d \frac{1}{\Lambda_{hh}+\rho}\frac{\E(\lambda_{ikh}^2)\log(d/\delta)}{n} + \frac{d^2\log^2(d/\delta)}{\rho n^2} \right)\label{eq: trance tmp} 
\end{align}
By definition, 
\begin{align*}
    \lambda_{ikh} = \E_{a\sim p^{\calA_i}} \left[\e_k V^\top (a-x)(a-x)^\top  V \e_h\right] 
\end{align*}
and thus
\begin{align*}
    \sum_{k=1}^d \lambda_{ikh}^2 &\leq \E_{a\sim p^{\calA_i}} \left[\sum_{k=1}^d\left(\e_k V^\top (a-x)(a-x)^\top V \e_h\right)^2\right] \\
    &= \E_{a\sim p^{\calA_i}} \left[\sum_{k=1}^d \e_h^\top V^\top (a-x)(a-x)^\top V \e_k \e_k^\top V^\top (a-x)(a-x)^\top V \e_h\right] \\
    &= \E_{a\sim p^{\calA_i}}\left[\e_h^\top V^\top (a-x)(a-x)^\top (a-x)(a-x)^\top V \e_h\right] \\
    &\leq \E_{a\sim p^{\calA_i}}\left[\e_h^\top V^\top (a-x)(a-x)^\top V\e_h\right] \\
    &= \lambda_{ihh}
\end{align*}
and $\sum_{k=1}^d \E[\lambda_{ikh}^2] \leq \E[\lambda_{ihh}]=\Lambda_{hh}$. 
Continuing from \pref{eq: trance tmp} and using that $\rho=\Theta\left(\frac{d\log(d/\delta)}{n}\right)$,  
\begin{align*}
    \tr\left(\Delta(\Lambda+\rho I)^{-1} \Delta\right)  \leq \order\left(\sum_{h=1}^d \frac{\Lambda_{hh}\log(d/\delta)}{(\Lambda_{hh}+\rho)n} + \frac{d^2\log^2(d/\delta)}{n^2}\right) \leq \order\left(\frac{d\log(d/\delta)}{n}\right). 
\end{align*}
This gives a bound on $\|\Gamma y\|^2_{\hatSigma^{-1}}$ and finishes the proof after combining \pref{eq: three decompose}.

\end{proof}

\subsection{Union Bound over Policies}\label{app: union bounds}

In \pref{lem: hatH > H lemma}, \pref{lem: local norm concentration for vectors 2}, and 
\pref{lem: local norm concentration for matrix 2}, we have obtained the desired concentration inequalities \emph{under a fixed policy $p$. }
In this subsection, we proceed to take union bound over \emph{all policies} that are possibly used by \pref{alg: FTRL}.

The set of policies that could be generated by \pref{alg: FTRL} is the following: 
\begin{align*}   
    \mathbf{P} = \left\{p:~ \widehat{\Cov}(p^{\calA}) = \argmin_{\pmb{H}\in\calH^{\calA}} \left\{\left\langle \pmb{H}, \pmb{Z}\right\rangle + F(\pmb{H}) \right\}, \text{for\ } \pmb{Z}\in \calZ \right\}
\end{align*}
where $\calZ=[-T^2, T^2]^{(d+1)\times (d+1)} \cap \mathbb{S}$ with $\mathbb{S}$ denoting the set of symmetric matrices. To see this, notice that \pref{alg: FTRL} at round $t$ corresponds to the policy defined above with $\pmb{Z} = \eta_t\sum_{s=1}^{t-1}(\hat{\gamma}_s - \alpha_s \pmb{\hatSigma}_s^{-1})$.

Our goal is to construct a $\epsilon$-cover $\mathbf{P}'$ so that every policy $p\in\mathbf{P}$ can find a policy $p'\in\mathbf{P}'$ making $- \epsilon I \preceq \widehat{\Cov}(p^{\calA}) - \widehat{\Cov}(p'^{\calA}) \preceq \epsilon I$ on \emph{every} action set $\calA$. The size of such a cover is bounded in the Proposition below.  %If $\|\widehat{\Cov}(p^{\calA})-\widehat{\Cov}(p'^{\calA})\|_F$ is small on every action set $\calA$, all vectors/matrices defined in \pref{def: simple definition} induced by $p$ and $p'$ are also close. This further allows us to show properties in \pref{lem: hatH > H lemma}, \pref{lem: local norm concentration for vectors 2}, and 
%\pref{lem: local norm concentration for matrix 2} \emph{for all policies possibly used by the algorithm}. This will be carefully justified in \pref{lem: main concentration}, \pref{lem: local norm concentration for vectors} and \pref{lem: local norm concentration for matrix}. 

\begin{lemma}\label{prop: cover}
      There exists an $\epsilon$-cover $\mathbf{P}'$ of $\mathbf{P}$ with size $\log |\mathbf{P}'|=\order\left(d^2\log\frac{d}{\epsilon}\right)$ such that for any $p\in\mathbf{P}$, there exists an $p'\in\mathbf{P'}$ satisfying 
      \begin{align*}
           \left\| \widehat{\Cov}(p^{\calA}) - \widehat{\Cov}(p'^{\calA}) \right\|_F\leq \epsilon
      \end{align*}
      for all $\calA$. 
\end{lemma}
\begin{proof}
It is straightforward to construct an $\frac{\epsilon}{4}$-cover $\calC$ for $\calZ=[-T^2, T^2]^{(d+1)\times (d+1)} \cap \mathbb{S}$ in Frobenius norm with size $|\calC|= (\frac{24(d+1)^2}{\epsilon})^{(d+1)^2}$ (Exercise 27.6 of \cite{lattimore2020bandit}).  Now define $\mathbf{P}'$ as  
\begin{align}   
    \mathbf{P}' = \left\{p:~ \widehat{\Cov}(p^{\calA}) = \argmin_{\pmb{H}\in\calH^{\calA}} \left\{\left\langle \pmb{H}, \pmb{Z}\right\rangle + F(\pmb{H}) \right\}, \text{for\ } \pmb{Z}\in \calC \right\}  \label{eq: policy cover}
\end{align}

Below, we show that this is a $\epsilon$-cover for $\mathbf{P}$. 
%\begin{align*}
%    \pmb{Z} \rightarrow \argmin_{\pmb{H}\in\calH^{\calA}} \left\{\left\langle \pmb{H}, \pmb{Z}\right\rangle + F(\pmb{H})\right\}
%\end{align*}
%is $C$-Lipshitz for any $\calA$. Below, we show that $C\leq 4$. 

Consider two policies $p_1$ and $p_2$ defined as the following: 
\begin{align*}
\widehat{\Cov}(p_1^{\calA}) &=  \argmin_{\pmb{H}\in\calH^{\calA}} \left\{\left\langle \pmb{H}, \pmb{Z}_1\right\rangle + F(\pmb{H})\right\}
\\\widehat{\Cov}(p_2^{\calA})  &=  \argmin_{\pmb{H}\in\calH^{\calA}} \left\{\left\langle \pmb{H}, \pmb{Z}_2\right\rangle + F(\pmb{H})\right\}
\end{align*}
with $\|\pmb{Z}_1 - \pmb{Z}_2\|_F \le \frac{\epsilon}{4}$. Consider an arbitrary $\calA$ and define $\pmb{H}_1 = \widehat{\Cov}(p_1^{\calA})$, $\pmb{H}_2 = \widehat{\Cov}(p_2^{\calA})$. Below we show $\|\pmb{H}_1 - \pmb{H}_2\|_F\leq \epsilon$.  %By our assumption on the action set, it holds that $\pmb{H}_1 \preceq 2\pmb{I}$ and $\pmb{H}_2 \preceq 2\pmb{I}$. 
%\end{proof}
%\begin{proposition} \label{prop:hatH close}
%If $\|\pmb{Z}_1 - \pmb{Z}_2\|_F \le \epsilon$, then 
%\begin{align*}
%\|\pmb{H}_1 - \pmb{H}_2\|_F \le 4\epsilon \quad \text{and} \quad -4 \epsilon I \preceq \pmb{H}_1 - \pmb{H}_2 \preceq 4 \epsilon I
%\end{align*}
%\end{proposition}
%\begin{proof}

Since $F(\pmb{H})$ is convex for $\pmb{H}$, from the first-order optimality condition for convex function, we have 
\begin{align*}
\left\langle \pmb{H}_1, \pmb{Z}_1\right\rangle +  F(\pmb{H}_1) &\le \left\langle \pmb{H}_2, \pmb{Z}_1\right\rangle +  F(\pmb{H}_2) - D_F(\pmb{H}_2, \pmb{H}_1) 
\\& = \left\langle \pmb{H}_2, \pmb{Z}_2 \right\rangle + \left\langle \pmb{H}_2, \pmb{Z}_1 - \pmb{Z}_2 \right\rangle +   F(\pmb{H}_2) - D_F(\pmb{H}_2, \pmb{H}_1) 
\\ \left\langle \pmb{H}_2, \pmb{Z}_2\right\rangle +  F(\pmb{H}_2) &\le \left\langle \pmb{H}_1, \pmb{Z}_2\right\rangle +F(\pmb{H}_1) - D_F(\pmb{H}_1, \pmb{H}_2) 
\\&= \left\langle \pmb{H}_1, \pmb{Z}_1\right\rangle + \left\langle \pmb{H}_1, \pmb{Z}_2 - \pmb{Z}_1\right\rangle +  F(\pmb{H}_1) - D_F(\pmb{H}_1, \pmb{H}_2) 
\end{align*}
Adding up these the two inequalities, we get
\begin{equation*}
2\min\{D_F(\pmb{H}_1, \pmb{H}_2),  D_F(\pmb{H}_2, \pmb{H}_1) \} \le D_F(\pmb{H}_1, \pmb{H}_2) +  D_F(\pmb{H}_2, \pmb{H}_1) \le \left\langle \pmb{Z}_1 - \pmb{Z}_2, \pmb{H}_2- \pmb{H}_1 \right\rangle
\end{equation*}
Since the second order directional derivative for $F$ is $D^2F(\pmb{H})[\pmb{X}, \pmb{X}] = \tr(\pmb{X}\pmb{H}^{-1}\pmb{X}\pmb{H}^{-1})$ for any symmetric matrix $\pmb{X}$, from the Taylor series, there exists $\pmb{H}'$ that is a  line segment between $\pmb{H}_1$ and $\pmb{H}_2$ such that
\begin{align*}
\|\pmb{H}_1 - \pmb{H}_2\|_{\nabla^{2}F(\pmb{H}')} ^2 &=  2\min\{D_F(\pmb{H}_1, \pmb{H}_2),  D_F(\pmb{H}_2, \pmb{H}_1) \} \le \left\langle \pmb{Z}_1 - \pmb{Z}_2, \pmb{H}_2- \pmb{H}_1 \right\rangle 
\\& \le \|\pmb{Z}_1 - \pmb{Z}_2\|_{\nabla^{-2}F(\pmb{H}')}  \|\pmb{H}_1 - \pmb{H}_2\|_{\nabla^{2}F(\pmb{H}')}\tag{\pref{lem:matrix norm holder}}
\end{align*}
Thus we have $\|\pmb{H}_1 - \pmb{H}_2\|_{\nabla^{2}F(\pmb{H}')} \le \|\pmb{Z}_1 - \pmb{Z}_2\|_{\nabla^{-2}F(\pmb{H}')}$. Since $\|a\|_2 \le 1$,  $\pmb{H}' \preceq 2\pmb{I}$. The left-hand side and right-hand side can be bounded as follows,
\begin{align*}
&\|\pmb{H}_1 - \pmb{H}_2\|_{\nabla^{2}F(\pmb{H}')} = \sqrt{\tr\left((\pmb{H}_1 - \pmb{H}_2)(\pmb{H}')^{-1}(\pmb{H}_1 - \pmb{H}_2)(\pmb{H}')^{-1}\right)} \ge \frac{1}{2}\|\pmb{H}_1 - \pmb{H}_2\|_F
\\ &\|\pmb{Z}_1 - \pmb{Z}_2\|_{\nabla^{-2}F(\pmb{H}')} = \sqrt{\tr\left((\pmb{Z}_1 - \pmb{Z}_2) \pmb{H}'(\pmb{Z}_1 - \pmb{Z}_2)\pmb{H}'\right)} \le 2\|\pmb{Z}_1 - \pmb{Z}_2\|_F \le \frac{\epsilon}{2}
\end{align*}
Combining the three inequalities above, we conclude that
\begin{equation*}
 \|\pmb{H}_1 - \pmb{H}_2\|_F 
\leq 2\|\pmb{H}_1 - \pmb{H}_2\|_{\nabla^2 F(\pmb{H}')} \leq 
2\|\pmb{Z}_1 - \pmb{Z}_2\|_{\nabla^{-2} F(\pmb{H}')} \le 4\|\pmb{Z}_1 - \pmb{Z}_2\|_F \le \epsilon. 
\end{equation*}
%This implies that all eigenvalues of $\pmb{H}_1 - \pmb{H}_2$ are bounded in $[-\epsilon, \epsilon]$, which further implies
%\begin{align*}
%     \max\limits_i\{|\lambda_i\left(\pmb{H}_1 - \pmb{H}_2\right)|\} \le \|\pmb{H}_1 - \pmb{H}_2\|_F \leq \epsilon
%\end{align*}
%which implies
%\begin{equation*}
%    \lambda_{\max}(\pmb{H}_1 - \pmb{H}_2) \le \epsilon \quad \text{and} \quad \lambda_{\min}(\pmb{H}_1 - \pmb{H}_2) \ge -4 \epsilon
%\end{equation*}
%Thus, we have
\begin{equation*}
    -\epsilon \pmb{I} \preceq \pmb{H}_1 - \pmb{H}_2 \preceq \epsilon \pmb{I}. 
\end{equation*}
\end{proof}

\iffalse
\begin{proposition}\label{prop:Ex close}
If $\|\pmb{Z}_1 - \pmb{Z}_2\|_F \le \epsilon$, then for any context $\calA$, 
\begin{equation*}
    \|\E_{a \sim p_1^{\calA}}[a] - \E_{a \sim p_2^{\calA}}[a]\|_F \le 2\epsilon
\end{equation*}
\end{proposition}
\begin{proof}
From \pref{prop:hatH close}, we have $\|\pmb{H}_1 - \pmb{H}_2\|_F \le 4\epsilon$
\begin{align*}
\|\pmb{H}_1 - \pmb{H}_2\|_F &= \|  \mathbb{E}_{a \sim p_1^{\calA}}[aa^\top] - \mathbb{E}_{a \sim p_2^{\calA}}[aa^\top] \|_F + 2\|\mathbb{E}_{a \sim p_1^{\calA}}[a] - \mathbb{E}_{a \sim p_2^{\calA}}[a]\|_F \le 4\epsilon
\end{align*}
Thus we have
\begin{equation*}
\|\mathbb{E}_{a \sim p_1^{\calA}}[a] - \mathbb{E}_{a \sim p_2^{\calA}}[a]\|_F  \le 2\epsilon
\end{equation*}
\end{proof}
\fi

% where $x_1,x_2$ can be arbitrary vector with $\|x_1\|_2 \le 1$, $\|x_1\|_2 \le 1$ and $\|x_1 - x_2\|_2 \le 2\epsilon$
\begin{lemma} \label{prop:H close}
Suppose that $p, p'$ are two policies such that for all action set $\calA$, 
\begin{align}
    \left\| \widehat{\Cov}(p^{\calA}) - \widehat{\Cov}(p'^{\calA}) \right\|_F\leq \epsilon    \label{eq: close 0}
\end{align}
Then all quantities defined in \pref{def: simple definition} under 
$p$ and $p'$ are close. That is, 
\begin{gather}
     \|x(p) - x(p')\|\leq \epsilon     \label{eq: close 1}\\
    \|\hat{x}(p) - \hat{x}(p')\|\leq \epsilon    \label{eq: close 2}\\
    \| H(p) - H(p') \|_F \leq 7\epsilon \label{eq: close 3}\\
    \| \hatH(p) - \hatH(p')\|_F\leq 7\epsilon \label{eq: close 4}\\
    \| \pmb{H}(p) - \pmb{H}(p')\|_F\leq \epsilon  \label{eq: close 5}\\
    \| \pmb{\hatH}(p) -  \pmb{\hatH}(p') \|_F\leq \epsilon  \label{eq: close 6}\\
   \|\hatSigma(p) - \hatSigma(p')\|_F \leq 7\epsilon  \label{eq: close 7}\\
 \| \pmb{\hatSigma}(p)  - \pmb{\hatSigma}(p')\|_F\leq \epsilon  \label{eq: close 8}
\end{gather}

\iffalse
\begin{gather*}
     \left\|\E_{a\sim p^\calA}[a] - \E_{a\sim p'^\calA}[a]\right\| \leq \epsilon,  \\
      -4\epsilon I - \|v-v'\|^2 I \preceq \E_{a\sim p^\calA}[(a-v)(a-v)^\top] - \E_{a\sim p'^\calA}[(a-v')(a-v')^\top] \preceq 4\epsilon I + \|v-v'\|^2 I
\end{gather*}

For any context $\calA$, let
\begin{align*}
H_1 &= \mathbb{E}_{a \sim p_1^{\calA}}\left[(a- \hat{x}(p_1))(a- \hat{x}(p_1))^\top\right]
\\H_2 &= \mathbb{E}_{a \sim p_2^{\calA}}\left[(a- \hat{x}(p_2))(a- \hat{x}(p_2) )^\top\right]
\end{align*}
If $\|\pmb{Z}_1 - \pmb{Z}_2\|_F \le \epsilon$, then
\begin{align*}
    \|H_1 - H_2 \|_F \le 12\epsilon \quad \text{and} \quad  -12 \epsilon I \preceq H_1 - H_2 \preceq 12 \epsilon I
\end{align*}
\fi
\end{lemma}

\begin{proof}

\pref{eq: close 5} and \pref{eq: close 6} are direct consequences of \pref{eq: close 0} since $\pmb{H}(p)$ and $\pmb{\hatH}(p)$ are expectations of $\widehat{\Cov}(p^\calA)$ over distributions over $\calA$. \pref{eq: close 8} is directly implied by \pref{eq: close 6} because $\pmb{\hatSigma}(p) = \pmb{\hatH}(p) + \beta \pmb{I}$. 

To show \pref{eq: close 1} and \pref{eq: close 2}, observe that by the definition of $x(p)$ and $\pmb{H}(p)$, 
\begin{align*}
     \pmb{H}(p) = \E_{\calA\sim D} \E_{a\sim p^\calA} \begin{bmatrix}
          a a^\top  & a \\
          a^\top   & 1 
     \end{bmatrix} &= 
     \begin{bmatrix}
          \E_{\calA\sim D}\E_{a\sim p^\calA}[a a^\top]  & \E_{\calA\sim D}\E_{a\sim p^\calA}[a] \\
          \E_{\calA\sim D}\E_{a\sim p^\calA}[a^\top]   & 1 
     \end{bmatrix} \\
     &= 
     \begin{bmatrix}
          \E_{\calA\sim D}\E_{a\sim p^\calA}[a a^\top]  & x(p) \\
          x(p)^\top   & 1 
     \end{bmatrix}
\end{align*}
Therefore, $\|x(p) - x(p')\|\leq \|\pmb{H}(p) - \pmb{H}(p')\|_F \leq \epsilon$. Similarly, $\|\hatx(p) - \hatx(p')\|\leq \|\pmb{\hatH}(p) - \pmb{\hatH}(p')\|_F \leq \epsilon$. 

If remains to show \pref{eq: close 3}, \pref{eq: close 4} and \pref{eq: close 7}. Next, we show \pref{eq: close 3}: 
\begin{align}
     &H(p) - H(p')  \nonumber  \\
     &= \E_{\calA\sim D} \left[ \E_{a\sim p^{\calA}}[(a - \hatx(p))(a - \hatx(p))^\top ] - \E_{a\sim p'^{\calA}}[(a - \hatx(p'))(a - \hatx(p'))^\top ] \right]   \nonumber \\
     &= \E_{\calA\sim D} \Big[ \E_{a\sim p^{\calA}}[aa^\top] - \E_{a\sim p'^{\calA}}[aa^\top] \Big]  \nonumber  \\
     &\qquad - x(p)\hatx(p)^\top -  \hatx(p) x(p)^\top   + x(p')\hatx(p')^\top +  \hatx(p') x(p')^\top  \tag{using $\E_{\calA\sim D}\E_{a\sim p^\calA}[a] = x(p)$}\\
     &\qquad \qquad + \hatx(p)\hatx(p)^\top - \hatx(p')\hatx(p')^\top   \label{eq: H - H'}
\end{align}
Using the property
\begin{align*}
     \|ab^\top - cd^\top\|_F \leq \|ab^\top - cb^\top\|_F + \|cb^\top - cd^\top\|_F \leq \|a-c\|\|b\| + \|c\|\|b-d\|
\end{align*}
we continue from \pref{eq: H - H'} and bound 
\begin{align*}
    &\|H(p) - H(p')\|_F \\
    &\leq \|\pmb{H}(p) - \pmb{H}(p')\|_F + 2(\|\hatx(p) - \hatx(p') \| + \|x(p) - x(p') \|) + \|\hatx(p) - \hatx(p') \| + \|\hatx(p) - \hatx(p') \| \\
    &\leq 7\epsilon. 
\end{align*}

\pref{eq: close 4} can be shown in the same manner, which further implies \pref{eq: close 7} by the definition of $\hatSigma(p)$. 

\end{proof}

\begin{lemma} \label{lem: main concentration}
With probability $1-\delta$, for all $t = 1, \cdots, T$, 
\begin{align*}
    \hatH_t + \frac{50(d+1)^3\log(3T/\delta)}{t-1}I  \succeq \frac{1}{2}H_t,
\\ \pmb{\hatH}_t + \frac{50(d+1)^3\log(3T/\delta)}{t-1}\pmb{I} \succeq \frac{1}{2}\pmb{H}_t.
\end{align*}
\end{lemma}

\begin{proof} 
Notice that $\hat{H}_t, \pmb{\hatH}_t, H_t, \pmb{H}_t$ corresponds to $\hat{H}(p_t), \pmb{\hatH}(p_t), H(p_t), \pmb{H}(p_t)$ defined in \pref{def: simple definition} with $n=t-1$. To show the lemma, our strategy is to argue the following two facts: 1) the two desired inequalities hold for all policies in the cover $\mathbf{P}'$ (defined in \pref{eq: policy cover}) with high probability. This is simply by applying \pref{lem: hatH > H lemma} with an union bound over policies in $\mathbf{P}'$. 2) $p_t$ is sufficiently close to the nearest element in $\mathbf{P}'$ so the desired inequalities still approximately hold. 

By \pref{prop: cover}, we can find $p'\in\mathbf{P}'$ such that for all $\calA$, 
\begin{align*}
    \left\| \widehat{\Cov}(p^{\calA}_t) - \widehat{\Cov}(p'^{\calA}) \right\|_F \leq \epsilon. 
\end{align*}
By \pref{prop:H close}, it holds that 
\begin{align}
     \| H(p_t) - H(p') \|_F &\leq 7\epsilon,  \quad  \| \hatH(p_t) - \hatH(p')\|_F \leq 7\epsilon  \label{eq: come from prop 3} \\
     \| \pmb{H}(p_t) - \pmb{H}(p') \|_F &\leq \epsilon,   \quad \|\pmb{\hatH}(p_t) - \pmb{\hatH}(p') \|_F\leq \epsilon   \label{eq: come from prop 32}
\end{align}

%because $\pmb{H}(p)$ and $\pmb{\hatH}(p)$ are both expectations of $\widehat{\Cov}(p^{\calA})$ over some distributions over $\calA$. By \pref{prop:H close}, we also have 

On the other hand, using \pref{lem: hatH > H lemma} and union bound, with probability $1-\delta$, we have
\begin{align}
        \hat{H}(p') + \frac{4d\log(6d|\mathbf{P}'|/\delta)}{n} I &\succeq \frac{1}{2}H(p'),   \label{eq: combine 1}\\
        \pmb{\hat{H}}(p') + \frac{3d\log(d|\mathbf{P}'|/\delta)}{n} \pmb{I} &\succeq \frac{1}{2}\pmb{H}(p').  \label{eq: combine 2}
\end{align}  
Combining \pref{eq: combine 1} and \pref{eq: come from prop 3}, we get 
\begin{align*}
    \hat{H}(p_t) + 7\epsilon I + \frac{4d\log(6d|\mathbf{P}'|/\delta)}{n} I  \succeq \hat{H}(p') + \frac{4d\log(6d|\mathbf{P}'|/\delta)}{n} I  \succeq \frac{1}{2}H(p') \succeq \frac{1}{2}H(p_t) - \frac{7}{2}\epsilon I 
\end{align*}
which implies the first inequality in the lemma by plugging in the choice of $\epsilon  = \frac{1}{T^3}$ and the upper bound of $\log|\mathbf{P}'|$ in \pref{prop:H close}. The second inequality in the lemma can be obtained similarly by combining \pref{eq: come from prop 32} and \pref{eq: combine 2}. 

\end{proof}

\begin{lemma}\label{lem: local norm concentration for vectors}
With probability of at least $1-\delta$, for all $t = 1, \cdots, T$, 
\begin{align*}
     \|x_t-\hat{x}_t\|^2_{\hatSigma_t^{-1}} & \leq \order\left(\frac{d^3\log{(dT/\delta)}}{t}\right)
\end{align*}
\end{lemma}
\begin{proof}
Notice that $x_t, \hat{x}_t, \hatSigma_t$ corresponds to $x(p_t), \hat{x}(p_t), \hatSigma(p_t)$ defined in \pref{def: simple definition} with $n=t-1$. To show the lemma, our strategy is to argue the following two facts: 1) the two desired inequalities hold for all policies in the cover $\mathbf{P}'$ with high probability. This is simply by applying \pref{lem: local norm concentration for vectors 2} with an union bound over policies in $\mathbf{P}'$. 2) $p_t$ is sufficiently close to the nearest element in $\mathbf{P}'$ so the desired inequalities still approximately hold. 

By \pref{prop: cover}, we can find $p'\in\mathbf{P}'$ such that for all $\calA$, 
\begin{align*}
   \left\| \widehat{\Cov}(p^{\calA}_t) - \widehat{\Cov}(p'^{\calA}) \right\|_F\leq \epsilon. 
\end{align*}
By \pref{prop:H close}, we have 
\begin{align}
     \|x(p') - x(p_t)\| \leq \epsilon, \quad \|\hatx(p') - \hatx(p_t)\| \leq \epsilon, \quad \|\hatSigma(p') - \hatSigma(p_t)\|_F \leq  7\epsilon  \label{eq: key inequ 2}
\end{align}
Thus, 
\begin{align*}
    &\|x(p_t) - \hatx(p_t)\|^2_{\hatSigma(p_t)^{-1}} \\
    &= \left(\|x(p_t) - \hatx(p_t)\|^2_{\hatSigma(p_t)^{-1}} - \|x(p') - \hatx(p')\|^2_{\hatSigma(p')^{-1}}\right) + \|x(p') - \hatx(p')\|^2_{\hatSigma(p')^{-1}} \\
    &\leq \left(\|x(p_t) - \hatx(p_t)\|^2_{\hatSigma(p_t)^{-1}} - \|x(p') - \hatx(p')\|^2_{\hatSigma(p')^{-1}}\right) + \order\left(\frac{d\log(d|\mathbf{P}'|/\delta)}{t-1}\right)   \tag{by \pref{lem: local norm concentration for vectors 2} with an union bound over $\mathbf{P}'$} \\
    &= \theta_t^\top \hatSigma(p_t)^{-1} \theta_t - \theta'^\top \hatSigma(p')^{-1} \theta' + \order\left(\frac{d\log(d|\mathbf{P}'|/\delta)}{t-1}\right) \tag{define $\theta_t = x(p_t) - \hatx(p_t)$ and $\theta' = x(p') - \hatx(p')$} \\
    &= (\theta_t - \theta')^\top \hatSigma(p_t)^{-1} \theta_t + \theta'^\top \Big(\hatSigma(p_t)^{-1} - \hatSigma(p')^{-1}\Big)\theta_t + \theta'^\top \hatSigma(p')^{-1}(\theta_t - \theta') + \order\left(\frac{d\log(d|\mathbf{P}'|/\delta)}{t-1}\right) \\
    &\leq (\theta_t - \theta')^\top \Big(\hatSigma(p_t)^{-1}\theta_t + \hatSigma(p')^{-1}\theta'\Big) + \theta'^\top \hatSigma(p')^{-1}\Big( \hatSigma(p') - \hatSigma(p_t) \Big)\hatSigma(p_t)^{-1}\theta_t + \order\left(\frac{d\log(d|\mathbf{P}'|/\delta)}{t-1}\right)
\end{align*}
The first two terms above can be bounded by the order of $\order(\epsilon t^2)$ by \pref{eq: key inequ 2}. Using the choice $\epsilon = \frac{1}{T^3}$ and recalling that $\log|\mathbf{P}'| = \order(d^2 \log(d/\epsilon))$ finishes the proof. 

\end{proof}

\begin{lemma} \label{lem: local norm concentration for matrix}
With probability of at least $1-\delta$, for all $t=1,2,\ldots, T$, 
\begin{align*}
     \|(\hatSigma_t - H_t)y_t\|^2_{\hatSigma_t^{-1}}  &\leq \order\left(\frac{d^3\log{(dT/\delta)}}{t}\right)
\end{align*}
\end{lemma}

% \order\left(\frac{d^3\log{\frac{d\eta t}{\delta}}}{t}\right)
\begin{proof}
Notice that $x_t, \hat{x}_t, \hatSigma_t$ corresponds to $x(p_t), \hat{x}(p_t), \hatSigma(p_t)$ defined in \pref{def: simple definition} with $n=t-1$. To show the lemma, our strategy is to argue the following two facts: 1) the two desired inequalities hold for all policies in the cover $\mathbf{P}'$ with high probability. This is simply by applying \pref{lem: local norm concentration for vectors 2} with an union bound over policies in $\mathbf{P}'$. 2) $p_t$ is sufficiently close to the nearest element in $\mathbf{P}'$ so the desired inequalities still approximately hold. 

By \pref{prop: cover}, we can find $p'\in\mathbf{P}'$ such that for all $\calA$, 
\begin{align*}
   \left\| \widehat{\Cov}(p^{\calA}_t) - \widehat{\Cov}(p'^{\calA}) \right\|_F\leq \epsilon. 
\end{align*}
By \pref{prop:H close}, we have 
\begin{align}
     \|x(p') - x(p_t)\| \leq \epsilon, \quad \|\hatx(p') - \hatx(p_t)\| \leq \epsilon, \quad \|\hatSigma(p') - \hatSigma(p_t)\|_F \leq  7\epsilon  \label{eq: key inequ}
\end{align}
Thus, for any $\|y_t\|_2 \le 1$,
\begin{align*}
    & \|(\hatSigma(p_t) - H(p_t))y_t\|^2_{\hatSigma(p_t)^{-1}}  \\
    &= \left(\|(\hatSigma(p_t) - H(p_t))y_t\|^2_{\hatSigma(p_t)^{-1}}  - \|(\hatSigma(p') - H(p'))y_t\|^2_{\hatSigma(p')^{-1}} \right) +\|(\hatSigma(p') - H(p'))y_t\|^2_{\hatSigma(p')^{-1}} \\
    &\leq\left(\|(\hatSigma(p_t) - H(p_t))y_t\|^2_{\hatSigma(p_t)^{-1}}  - \|(\hatSigma(p') - H(p'))y_t\|^2_{\hatSigma(p')^{-1}} \right) + \order\left(\frac{d\log(d|\mathbf{P}'|/\delta)}{t-1}\right)   \tag{by \pref{lem: local norm concentration for matrix 2} with an union bound over $\mathbf{P}'$} \\
    &= \theta_t^\top \hatSigma(p_t)^{-1} \theta_t - \theta'^\top \hatSigma(p')^{-1} \theta' + \order\left(\frac{d\log(d|\mathbf{P}'|/\delta)}{t-1}\right) \tag{define $\theta_t = (\hatSigma(p_t) - H(p_t))y_t$ and $\theta' = (\hatSigma(p') - H(p'))y_t$} \\
    &= (\theta_t - \theta')^\top \hatSigma(p_t)^{-1} \theta_t + \theta'^\top \Big(\hatSigma(p_t)^{-1} - \hatSigma(p')^{-1}\Big)\theta_t + \theta'^\top \hatSigma(p')^{-1}(\theta_t - \theta') + \order\left(\frac{d\log(d|\mathbf{P}'|/\delta)}{t-1}\right) \\
    &\leq (\theta_t - \theta')^\top \Big(\hatSigma(p_t)^{-1}\theta_t + \hatSigma(p')^{-1}\theta'\Big) + \theta'^\top \hatSigma(p')^{-1}\Big( \hatSigma(p') - \hatSigma(p_t) \Big)\hatSigma(p_t)^{-1}\theta_t + \order\left(\frac{d\log(d|\mathbf{P}'|/\delta)}{t-1}\right)
\end{align*}

%Since for any policy $p$, $\hatSigma(p) \preceq \order(1+ \frac{d\log(\frac{d}{\delta})}{n} )I$. $\|\theta_t\|_2, \|\theta'\|_2 \le \order(1 + \frac{d\log(\frac{d}{\delta})}{n})I$. 
The first two terms above can be bounded by the order of $\order(\epsilon t^2)$ by \pref{eq: key inequ}. Plugging in the choice of $\epsilon =  \frac{1}{T^3}$ and recalling that $\log|\mathbf{P}'|=\order(d^2 \log(d/\epsilon))$ finishes the proof.

\end{proof}

\section{Regret Analysis}\label{app: regret analysis}

Consider the regret decomposition in \pref{sec: overall regret}. 
\begin{align*}
&\Reg(u)=\mathbb{E}\left[\sum_{t=1}^{T} \left\langle a_t-u^{\calA_t}, y_t \right\rangle\right] = \mathbb{E}\left[\sum_{t=1}^{T} \left\langle \pmb{H}_t^{\calA_t}-\pmb{U}^{\calA_t}, \gamma_t \right\rangle\right] =  \mathbb{E}\left[\sum_{t=1}^{T} \left\langle \pmb{H}_t^{\calA_0}-\pmb{U}^{\calA_0}, \gamma_t \right\rangle\right]
\\
&\le 
\underbrace{\mathbb{E}\left[\sum_{t=1}^{T} \left\langle \pmb{H}_t^{\calA_0}-\pmb{U}^{\calA_0}, \gamma_t - \hat{\gamma}_t \right\rangle\right]}_{\textbf{Bias}}+
 \underbrace{\mathbb{E}\left[\sum_{t=1}^{T} \left\langle\pmb{H}_t^{\calA_0}-\pmb{U}^{\calA_0}, \alpha_t\pmb{\hatSigma}^{-1}_t \right\rangle\right]}_{\textbf{Bonus}} 
 + 
 \underbrace{\mathbb{E}\left[\sum_{t=1}^{T} \left\langle\pmb{H}_t^{\calA_0}-\pmb{U}^{\calA_0},\hat{\gamma}_t -  \alpha_t\pmb{\hatSigma}^{-1}_t \right\rangle\right]}_{\textbf{FTRL-Reg}} 
 \end{align*}
 where $\calA_0$ is drawn from $D$ and is independent from the interaction between the learning and the environment. Recall that our algorithm is FTRL: 
 \begin{align*}
      \pmb{H}_t^{\calA_0} = \argmin\limits_{\pmb{H}\in \mathcal{H}^{\calA_0}} \left\{ \sum_{s=1}^{t-1}\left\langle \pmb{H}, \hat{\gamma}_s - \alpha_s \pmb{\hat\Sigma}_s^{-1}\right\rangle + \frac{F(\pmb{H})}{\eta_t} \right\}.
 \end{align*}
The \textbf{FTRL-Reg} term can be handled by the standard FTRL analysis (\pref{lem: FTRL guarantee}). In order to deal with the issue that $F$ can be unbounded on the boundary of $\mathcal{H}^{\calA_0}$, we apply \pref{lem: FTRL guarantee} with the regret comparator $\overline{\pmb{U}}^{\calA_0}$ defined as 
\begin{align*}
    \overline{\pmb{U}}^{\calA_0} = \left(1-\frac{1}{T^2}\right) \pmb{U}^{\calA_0} + \frac{1}{T^2} \pmb{H}^{\calA_0}_*
\end{align*}
where $\pmb{H}^{\calA_0}_*\triangleq \argmin_{\pmb{H} \in \mathcal{H}^{\calA_0}}F(\pmb{H})$. 
Thus, 
\begin{align}
&\textbf{FTRL-Reg}   \nonumber  \\
&\leq \mathbb{E}\left[\sum_{t=1}^{T} \left\langle\pmb{H}_t^{\calA_0}-\overline{\pmb{U}}^{\calA_0},\hat{\gamma}_t -  \alpha_t\pmb{\hatSigma}^{-1}_t \right\rangle\right] + \mathbb{E}\left[\sum_{t=1}^{T} \left\langle\overline{\pmb{U}}^{\calA_0}-\pmb{U}^{\calA_0},\hat{\gamma}_t -  \alpha_t\pmb{\hatSigma}^{-1}_t \right\rangle\right] \nonumber  \\
&\leq 
\underbrace{\mathbb{E}\left[\frac{F(\overline{\pmb{U}}^{\calA_0}) - \min_{\pmb{H} \in \mathcal{H}^{\calA_0}}F(\pmb{H})}{\eta_T}\right]}_{\textbf{Penalty}}  + \underbrace{\mathbb{E}\left[\sum_{t=1}^{T}\max\limits_{\pmb{H} \in \mathcal{H}^{\calA_0}}\left\langle \pmb{H}_t^{\calA_0} - \pmb{H}, \hat{\gamma}_{t} \right\rangle - \frac{D(\pmb{H}, \pmb{H}_t^{\calA_0})}{2\eta_t}\right]}_{\textbf{Stability-1}} \nonumber   \\ 
&\qquad + \underbrace{\mathbb{E}\left[\sum_{t=1}^{T}\max\limits_{\pmb{H} \in \mathcal{H}^{\calA_0}}\left\langle \pmb{H}_t^{\calA_0} - \pmb{H}, -\alpha_t\pmb{\hatSigma}^{-1}_t \right\rangle - \frac{D(\pmb{H}, \pmb{H}_t^{\calA_0})}{2\eta_t}\right]}_{\textbf{Stability-2}} + \underbrace{\mathbb{E}\left[\sum_{t=1}^{T} \left\langle\overline{\pmb{U}}^{\calA_0}-\pmb{U}^{\calA_0},\hat{\gamma}_t -  \alpha_t\pmb{\hatSigma}^{-1}_t \right\rangle\right]}_{\textbf{Error}}   \label{eq: FTRL decomposition}
\end{align}

In the rest of this section, we bound the following terms individually: \textbf{Bias}, \textbf{Bonus}, \textbf{Penalty}, \mbox{\textbf{Stability-1}}, \mbox{\textbf{Stability-2}}, \textbf{Error}. 

For any $t=2, \cdots, T$, let $\calE_{t-1}$ be the event that the high-probability event in \pref{lem: main concentration}, \pref{lem: local norm concentration for vectors}, and \pref{lem: local norm concentration for matrix} happens for all $1,\cdots, t-1$ and $\overline{\calE_{t-1}}$ be the opposite event of $\calE_{t-1}$(i.e. any of these three lemmas fails for any $1, \cdots, t-1$). We have $\calP[\calE_{t-1}] = 1-\order(\delta)$ and $\calP[\overline{\calE_{t-1}}] = \order(\delta)$. Let $\mathbb{E}\left[ \cdot ~|~ \calE_{t-1} \right]$ be the conditional expectation that event $\calE_{t-1}$ happens and let $\E_{t}^{\calE} = \E[\cdot ~|~ \calF_{t-1}, \calE_{t-1}]$
 
 % and  define $\mathbb{E}_{t}^{\calE} = \left[ \cdot ~|~ \calF_{t-1}, \calE \right]$ \CW{Should define $\calE_{t-1}$ as the event that the good event happen in $1,\ldots, t-1$ and only condition on this. }

\subsection{Bounding the Bias term} \label{app: Bound the fifth term}

\begin{lemma}\label{lem: handling bias}
% \CW{The original $\textbf{Bias}$ is defined without conditioning on $\calE$}
% With probability $1-\delta$
\begin{align*}
    \textbf{\textup{Bias}} = \E\left[\sum_{t=1}^T \left\langle \pmb{H}_t^{\calA_0} - \pmb{U}^{\calA_0}, \gamma_t - \hat{\gamma}_t \right\rangle \right] \leq   \frac{1}{4}\sum_{t=1}^T \alpha_t \|x_t-u\|_{\hatSigma_t^{-1}}^2 + \order\left(\delta T^2 + \sum_{t=1}^T \frac{d^3\log(T/\delta)}{\alpha_t t}\right) 
\end{align*}
%$\left\langle \pmb{H}_t^{\calA_0} - \pmb{U}^{\calA_0}, \gamma_t - \hat{\gamma}_t \right\rangle \le   O\left(\frac{d^3\log{\frac{d\eta t}{\delta}}}{\alpha_t t}\right) + \frac{\alpha_t}{3} \|x_t-u\|_{\hatSigma_t^{-1}}^2$
\end{lemma}

\begin{proof}
For any $t$, we have
\begin{align*}
    &\E_t^{\calE}\left[ \left\langle \pmb{H}_t^{\calA_0} - \pmb{U}^{\calA_0}, \gamma_t - \hat{\gamma}_t \right\rangle \right] \\ 
    &= \E_t^{\calE}\left[ \left\langle \pmb{H}_t - \pmb{U}, \gamma_t - \hat{\gamma}_t \right\rangle \right]   \tag{taking expectation over $\calA_0$}  \\
    &= \E_t^{\calE}\left[\left\langle x_t - u,  y_t - \hat{y}_t \right\rangle \right]  \tag{by the definition of lifting} \\
    &= \E_t^{\calE}\left[ (x_t - u)^\top\left(y_t - \hatSigma_t^{-1}(a_t-\hat{x}_t)a_t^\top y_t \right) \right] \tag{by the definition of $\hat{y}_t$} \\
    &= \E_t^{\calE}\left[ (x_t - u)^\top \left(y_t - \hatSigma_t^{-1}(a_t-\hat{x}_t)(a_t - \hat{x}_t)^\top y_t\right)\right] - \E_t^{\calE}\left[ (x_t - u)^\top \hatSigma_t^{-1}(a_t-\hat{x}_t)\hat{x}_t^\top y_t \right] \\
%Assume $\pmb{H}_t^{\calA_0} = \widehat{\Cov}(p_t^{\calA_0})$, $x_t^{\calA_0} = \mathbb{E}_{a \sim p_t^{\calA_0}}[a]$, $\pmb{U}^{\calA_0} = \widehat{\Cov}(p_u^{\calA_0})$, and $u^{\calA_0} = \mathbb{E}_{a \sim p_u^{\calA_0}}[a]$. We have
%\begin{align*}
%    \left\langle \pmb{H}_t^{\calA_0} - \pmb{U}^{\calA_0}, \gamma_t - \hat{\gamma}_t \right\rangle &= \left\langle x_t^{\calA_0} - u^{\calA_0}, y_t - \hat{y}_t \right\rangle 
%    = (x_t^{\calA_0} - u^{\calA_0})^\top \left(y_t - \hatSigma_t^{-1}(a_t-\hat{x}_t)a_t^\top y_t\right) \\
%    &= (x_t^{\calA_0} - u^{\calA_0})^\top \left(y_t - \hatSigma_t^{-1}(a_t-\hat{x}_t)(a_t - \hat{x}_t)^\top y_t\right) - (x_t^{\calA_0} - u^{\calA_0})^\top \hatSigma_t^{-1}(a_t-\hat{x}_t)\hat{x}_t^\top y_t
%\end{align*}
%Taking expectation over $\calA_t$ and $a_t$: 
    &=\E_t^{\calE}\left[ (x_t - u)^\top \left(I - \hatSigma_t^{-1} \E_{\calA\sim \calD}\E_{a_t\sim p_t^{\calA}}\left[(a_t-\hat{x}_t)(a_t-\hat{x}_t)^\top\right] \right) y_t\right] \\
    &\qquad - \E_t^{\calE}\left[(x_t - u)^\top \hatSigma_t^{-1}\left(\E_{\calA\sim \calD}\E_{a_t\sim p_t^{\calA}}[a_t] - \hat{x}_t\right)\hat{x}_t^\top y_t \right]   \tag{taking expectation over $\calA_t$ and $a_t$}  \\
    &= \E_t^{\calE}\left[ (x_t - u)^\top \hatSigma_t^{-1}\left(\hatSigma_t - H_t\right) y_t\right]- \E_t^{\calE}\left[ (x_t - u)^\top \hatSigma_t^{-1}\left(x_t - \hat{x}_t\right)\hat{x}_t^\top y_t\right] \tag{by the definition of $H_t$ and $x_t$}
    \\&\leq \E_t^{\calE}\left[(x_t - u)^\top \hatSigma_t^{-1} \left(\hatSigma_t - H_t\right) y_t\right] +  \E_t^{\calE} \left[\left|(x_t - u)^\top \hatSigma_t^{-1}\left(x_t - \hat{x}_t\right)\right|\right] \tag{$|\hat{x}_t^\top  y_t| \le 1$}
    \\&\leq \E_t^{\calE}\left[ 
  \|x_t-u\|_{\hatSigma_t^{-1}} \left(\|(\hatSigma_t - H_t)y_t\|_{\hatSigma_t^{-1}} + \|x_t-\hat{x}_t\|_{\hatSigma_t^{-1}} \right)\right] \tag{Cauchy-Schwarz}
 %\\ &\leq \order\left(\E_t\left[\sqrt{\frac{d^3\log (dt/\delta)}{t}}\|x_t-u\|_{\hatSigma_t^{-1}}\right]\right) + \E\left[ \textbf{Bias} ~|~ \overline{\calE_t}\right]\tag{\pref{lem: local norm concentration for matrix} and \pref{lem: local norm concentration for vectors}}
 % \\ &\le \frac{1}{4}\E\left[\sum_{t=1}^T \alpha_t \|x_t-u\|_{\hatSigma_t^{-1}}^2\right] + \order\left(\sum_{t=1}^T \frac{d^3\log(dt/\delta)}{\alpha_t t} \right) + \E\left[ \textbf{Bias} ~|~ \overline{\calE_t}\right] \tag{AM-GM inequality}
 \\&\leq \order\left(\sqrt{\frac{d^3\log (T/\delta)}{t}}\|x_t-u\|_{\hatSigma_t^{-1}}\right)\tag{\pref{lem: local norm concentration for matrix} and \pref{lem: local norm concentration for vectors} given $\calE_{t-1}$}
 \\&\le \frac{\alpha_t}{4} \|x_t-u\|_{\hatSigma_t^{-1}}^2 + \order\left(\frac{d^3\log(T/\delta)}{\alpha_t t} \right) \tag{AM-GM inequality}
 \end{align*}
% For any $t = 1, \cdots, T$, we have
% \begin{align*}
% \E_t^{\calE}\left[ \left\langle \pmb{H}_t^{\calA_0} - \pmb{U}^{\calA_0}, \gamma_t - \hat{\gamma}_t \right\rangle ~|~ \calE_{t}\right] &\le \E_t^{\calE}\left[ 
%   \|x_t-u\|_{\hatSigma_t^{-1}} \left(\|(\hatSigma_t - H_t)y_t\|_{\hatSigma_t^{-1}} + \|x_t-\hat{x}_t\|_{\hatSigma_t^{-1}} \right) ~|~ \calE_t \right] 
% \\&\leq \order\left(\E_t^{\calE}\left[\sqrt{\frac{d^3\log (dt/\delta)}{t}}\|x_t-u\|_{\hatSigma_t^{-1}}\right]\right)\tag{\pref{lem: local norm concentration for matrix} and \pref{lem: local norm concentration for vectors}}
%  \\&\le \frac{\alpha_t}{4} \|x_t-u\|_{\hatSigma_t^{-1}}^2 + \order\left(\frac{d^3\log(dt/\delta)}{\alpha_t t} \right) \tag{AM-GM inequality}
% \end{align*}
On the other hand, since $\hatSigma_t \succeq \frac{1}{t}I \succeq \frac{1}{T}I$,  for any $t = 1, \cdots, T$, 
\begin{equation*}
\|\hat{y}_t\|_2 = \|\Sigma_t^{-1}(a_t - \hat{x}_t)a_t^\top y_t\|_2 \le \|\Sigma_t^{-1}(a_t - \hat{x}_t)\|_2 \le \order(T)
\end{equation*}

Thus, we have trivial bound
\begin{align*}
\E_t\left[ \left\langle \pmb{H}_t^{\calA_0} - \pmb{U}^{\calA_0}, \gamma_t - \hat{\gamma}_t \right\rangle ~\Big|~  \overline{\calE_{t-1}} \right] = \E_t\left[ \left\langle \pmb{H}_t - \pmb{U}, \gamma_t - \hat{\gamma}_t \right\rangle ~|~  \overline{\calE_{t-1}} \right] = \E_t\left[\left\langle x_t - u,  y_t - \hat{y}_t \right\rangle ~|~ \overline{\calE_{t-1}} \right] \le \order(T)
\end{align*}

 Therefore, we have
 \begin{align*}
    \textbf{\textup{Bias}} &= \E\left[\sum_{t=1}^T \left\langle \pmb{H}_t^{\calA_0} - \pmb{U}^{\calA_0}, \gamma_t - \hat{\gamma}_t \right\rangle \right] 
    \\&= \E\left[\sum_{t=1}^T \E_t\left[ \left\langle \pmb{H}_t^{\calA_0} - \pmb{U}^{\calA_0}, \gamma_t - \hat{\gamma}_t \right\rangle \right] \right]
    \\&=  \E\left[\sum_{t=1}^T \E_t\left[ \left\langle \pmb{H}_t^{\calA_0} - \pmb{U}^{\calA_0}, \gamma_t - \hat{\gamma}_t \right\rangle ~\Big|~ \calE_{t-1}\right]\ind\{\calE_{t-1}\}\right] + \E\left[ \sum_{t=1}^T \E_t\left[ \left\langle \pmb{H}_t^{\calA_0} - \pmb{U}^{\calA_0}, \gamma_t - \hat{\gamma}_t \right\rangle ~\Big|~ \overline{\calE_{t-1}} \right]\ind\{\overline{\calE_{t-1}}\}\right] 
 \\&\le \frac{1}{4}\sum_{t=1}^T \alpha_t \|x_t-u\|_{\hatSigma_t^{-1}}^2 + \order\left(\sum_{t=1}^T \frac{d^3\log(T/\delta)}{\alpha_t t} + \delta T^2\right) 
\end{align*}
\end{proof}

\subsection{Bounding the Bonus term} \label{app: bound the second term}
We first prove the following useful technique lemma to bound the inner product of lifted matrices.
\begin{lemma}\label{lem: lifted trace bound}
Let $\pmb{G} = \begin{bmatrix} G+gg^{\top}&g \\g^{\top}&1\end{bmatrix}$,  $\pmb{H} = \begin{bmatrix} H+hh^{\top}&h \\h^{\top}&1\end{bmatrix}$ where $G$ and $H$ are positive semi-definite, and $\pmb{H}' = \pmb{H} + vv^{\top}$ where $v = \begin{bmatrix}
                  0\\ \sqrt{\beta}
             \end{bmatrix} \in \mathbb{R}^{d+1}$. Then we have
\begin{enumerate}
    \item $\tr\left(\pmb{H}^{-1}\pmb{G}\right) =  \tr(H^{-1}G) + \|g-h\|_{H^{-1}}^2 + 1$
    \item $\tr\left((\pmb{H}')^{-1}\pmb{G}\right) \ge \frac{1}{2\left(1+\frac{\beta}{1+\beta} \|h\|^2_{H^{-1}}\right)}\|g-h\|^2_{H^{-1}} - \frac{\beta^2}{(1+\beta)^2} \|h\|^2_{H^{-1}}$
\end{enumerate}
\end{lemma}

\begin{proof}
From Theorem 2.1 of \cite{lu2002inverses}, for any block matrix $R = \begin{bmatrix}
    A & B\\C&D
\end{bmatrix}$ if $A$ is invertible and its Schur complement $S_A = D-CA^{-1}B$ is invertible, then 
\begin{equation*}
    R^{-1} = \begin{bmatrix}
        A^{-1} + A^{-1}BS_A^{-1}CA^{-1} & -A^{-1}BS_A^{-1}
        \\-S_A^{-1}CA & S_A^{-1}
    \end{bmatrix}
\end{equation*}

Using above equation, for the first equation, Since $(H+hh^{\top})^{-1} = H^{-1} - \frac{H^{-1}hh^{\top}H^{-1}}{1+h^{\top}H^{-1}h}$. The inverse Schur complement of $H+hh^\top$ is $1+h^{\top}H^{-1}h$. Thus 
\begin{equation*}
    \pmb{H}^{-1} = \begin{bmatrix} (I+H^{-1}hh^{\top})(H+hh^{\top})^{-1} &-H^{-1}h \\ -h^{\top}H^{-1} & 1 + h^{\top}H^{-1}h\end{bmatrix} = \begin{bmatrix} H^{-1} &-H^{-1}h \\ -h^{\top}H^{-1} & 1 + h^{\top}H^{-1}h\end{bmatrix}
\end{equation*}
and
\begin{align*}
\tr(\pmb{H}^{-1}\pmb{G}) &= \tr\left( H^{-1}G + H^{-1}gg^\top - H^{-1}hg^\top\right) - h^\top H^{-1} g + 1 + h^\top H^{-1} h
\\&= \tr\left( H^{-1}G\right)  + g^\top H^{-1}g- 2g^\top H^{-1}h + h^\top H^{-1} h + 1 
\\&= \tr(H^{-1}G) + \|g-h\|_{H^{-1}}^2 + 1.
\end{align*}

For the second equation, observe that 
\begin{align*}
    \pmb{H}' = \begin{bmatrix}
         H + hh^\top & h \\
         h^\top & 1+\beta
    \end{bmatrix}
    = (1+\beta)\begin{bmatrix}
         \frac{1}{1+\beta}(H + hh^\top) & \frac{1}{1+\beta}h \\
         \frac{1}{1+\beta} h^\top & 1
    \end{bmatrix} 
    = (1+\beta) \begin{bmatrix}
         H' + h'h'^\top  & h' \\
         h'^\top & 1
    \end{bmatrix}
\end{align*}
where $h'=\frac{1}{1+\beta}h$ and $H'=\frac{1}{1+\beta}H + (\frac{1}{1+\beta} - \frac{1}{(1+\beta)^2})hh^\top = \frac{1}{1+\beta}H + \frac{\beta}{(1+\beta)^2}hh^\top\succeq 0$. 

Applying the first equality, we have 
\begin{align*}
    \tr((\pmb{H}')^{-1}\pmb{G}) = \frac{1}{1+\beta}\left(\tr((H')^{-1}G) + \|g-h'\|^2_{H'^{-1}}+1\right) \geq \frac{1}{1+\beta} \|g-h'\|^2_{H'^{-1}}. 
\end{align*}
Below, we continue to lower bound this term. By the same formula above, we have 
\begin{align*}
    H'^{-1} = \left(\frac{1}{1+\beta}H + \frac{\beta}{(1+\beta)^2}hh^\top\right)^{-1} = (1+\beta) H^{-1} - \frac{\beta H^{-1}hh^\top H^{-1} }{1 + \frac{\beta}{1+\beta}h^\top H^{-1}h}. 
\end{align*}
Thus
\begin{align*}
    &\frac{1}{1+\beta}\|g-h'\|^2_{H'^{-1}} \\
    &\geq \frac{1}{2(1+\beta)} \|g-h\|_{H'^{-1}}^2 - \frac{1}{1+\beta}\|h-h'\|_{H'^{-1}}^2 \tag{using $\|a+b\|^2\leq 2\|a\|^2 + 2\|b\|^2$}\\
    &= \frac{1}{2}(g-h)^\top \left( H^{-1} - \frac{\frac{\beta}{1+\beta}H^{-1}hh^\top H^{-1}}{1+\frac{\beta}{1+\beta} h^\top H^{-1}h}\right)(g-h) - (h-h')^\top\left( H^{-1} - \frac{\frac{\beta}{1+\beta}H^{-1}hh^\top H^{-1}}{1+\frac{\beta}{1+\beta} h^\top H^{-1}h}\right)(h-h') \\
    &\geq \frac{1}{2}\|g-h\|_{H^{-1}}^2 - \frac{\frac{\beta}{1+\beta} ((g-h)^\top H^{-1}h)^2}{2\left(1+\frac{\beta}{1+\beta} \|h\|^2_{H^{-1}}\right)} - \frac{\beta^2}{(1+\beta)^2} \|h\|^2_{H^{-1}}   \tag{using $h-h'=\frac{\beta}{1+\beta}h$} \\
    &\geq \frac{1}{2}\|g-h\|^2_{H^{-1}} - \frac{\frac{\beta}{1+\beta} \|h\|^2_{H^{-1}}}{2\left(1+\frac{\beta}{1+\beta} \|h\|^2_{H^{-1}}\right)}\|g-h\|^2_{H^{-1}} - \frac{\beta^2}{(1+\beta)^2} \|h\|^2_{H^{-1}}  \tag{Cauchy-Schwarz} \\
    &= \frac{1}{2\left(1+\frac{\beta}{1+\beta} \|h\|^2_{H^{-1}}\right)}\|g-h\|^2_{H^{-1}} - \frac{\beta^2}{(1+\beta)^2} \|h\|^2_{H^{-1}}. 
\end{align*}

\end{proof}

\noindent  Using \pref{lem: lifted trace bound}, we are able to show \pref{cor: special lifted trace bound} which bound part of the second term.
\begin{corollary} \label{cor: special lifted trace bound}
    $\tr(\pmb{U}\pmb{\hat\Sigma}^{-1}_t) \ge  \frac{1}{4}\|u-\hat{x}_t\|^2_{\hatSigma_t^{-1}} - \frac{1}{4}$. 
\end{corollary}

% \frac{1}{3}\|u^{\calA_0} -  \hat{x}_t\|_{\hatSigma_t^{-1}}^2

\begin{proof}
From \pref{lem: lifted trace bound}, we have 
\begin{equation*}
   \tr(\pmb{U}\pmb{\hat\Sigma}^{-1}_t) \ge  \frac{1}{2\left(1+\frac{\beta_t}{1+\beta_t} \|\hat{x}_t\|^2_{\hatSigma_t^{-1}}\right)}\|u-\hat{x}_t\|^2_{\hatSigma_t^{-1}} - \frac{\beta_t^2}{(1+\beta_t)^2} \|\hat{x}_t\|^2_{\Sigma_t^{-1}}. 
\end{equation*}

Since $\hatSigma_t \succeq \beta_tI$, $\hatSigma_t^{-1} \preceq \frac{1}{\beta_t}I$. Since $\|\hat{x}_t\|_2 \le 1$, we have $\|\hat{x}_t\|_{\hatSigma_t^{-1}}^2 \le \frac{1}{\beta_t}$. Then 
\begin{align*}
   \tr(\pmb{U}\pmb{\hat\Sigma}^{-1}_t) &\ge  \frac{1}{2\left(1+\frac{1}{1+\beta_t}\right)}\|u-\hat{x}_t\|^2_{\hatSigma_t^{-1}} - \frac{\beta_t}{(1+\beta_t)^2}
   \\ &\ge \frac{1}{4}\|u-\hat{x}_t\|^2_{\hatSigma_t^{-1}} - \frac{\beta_t}{(2\sqrt{\beta_t})^2}\tag{$\beta_t \ge 0$} \\
   &= \frac{1}{4}\|u-\hat{x}_t\|^2_{\hatSigma_t^{-1}} - \frac{1}{4}. 
\end{align*}

% $$\tr(\pmb{U}^{\calA_0}\pmb{\hat\Sigma}^{-1}_t) \ge \frac{1}{1+\beta_t(1+\hat{x}^{\top}\hatSigma_t^{-1}\hat{x})} \|u^{\calA_0} -  \hat{x}_t\|_{\hatSigma_t^{-1}}$$
% Since $\hatSigma_t \succeq \beta_t I$, we have $\hatSigma_t^{-1} \preceq \frac{1}{\beta_t} I$. Thus 
% $\beta_t(1+\hat{x}^{\top}\hatSigma_t^{-1}\hat{x}) \le \beta_t(1 + \frac{1}{\beta_t}\hat{x}^{\top}\hat{x}) \le \beta_t + 1$ and we have $\tr(\pmb{U}^{\calA_0}\pmb{\hat\Sigma}^{-1}_t) \ge \frac{1}{\beta_t + 2}\|u^{\calA_0} -  \hat{x}_t\|_{\hatSigma_t^{-1}}$. Given $\beta_t = \le 1$, we have $\tr(\pmb{U}^{\calA_0}\pmb{\hatSigma}^{-1}_t) \ge \frac{1}{3}\|u^{\calA_0} -  \hat{x}_t\|_{\hatSigma_t^{-1}}^2$
\end{proof}

\begin{lemma}\label{lem: bonus term}
% \CW{Revise this to handle high probability event more properly (like \pref{lem: handling bias})}
\begin{align*}
\textbf{\textup{Bonus}} 
&= \mathbb{E}\left[\sum_{t=1}^T\left\langle \pmb{H}_t^{\calA_0} - \pmb{U}^{\calA_0}, \alpha_t\pmb{\hat\Sigma}^{-1}_t  \right\rangle \right] \\
&\le 2(d+2)\sum_{t=1}^T\alpha_t - \frac{1}{4}\sum_{t=1}^T \alpha_t  \|u - x_t\|_{\hatSigma_t^{-1}}^2 + \order\left( \sum_{t=1}^T\frac{d^3\alpha_t\log{(T/\delta)}}{t} + \delta T\sum_{t=1}^T \alpha_t \right).
\end{align*}
\end{lemma}

\begin{proof}
For any $t$, we have
\begin{align*}
&\mathbb{E}_t^{\calE}\left[\left\langle \pmb{H}_t^{\calA_0} - \pmb{U}^{\calA_0}, \alpha_t\pmb{\hat\Sigma}^{-1}_t  \right\rangle \right] \\
&= \mathbb{E}_t^{\calE}\left[\tr\left(\alpha_t\left(\pmb{H}_t - \pmb{U}\right)\pmb{\hat\Sigma}^{-1}_t\right) \right]   \tag{taking expectation over $\calA_0$}
\\
&=  \mathbb{E}_t^{\calE}\left[\alpha_t \tr\left(\pmb{H}_t\pmb{\hat\Sigma}^{-1}_t\right)  - \alpha_t\tr\left(\pmb{U}\pmb{\hat\Sigma}^{-1}_t\right) \right]
\\
&\le \alpha_t \tr\left(\mathbb{E}_t^{\calE}\left[\pmb{H}_t\right]\pmb{\hat\Sigma}^{-1}_t\right)  - \mathbb{E}_t^{\calE}\left[\frac{\alpha_t}{4}\|u -  \hat{x}_t\|_{\hatSigma_t^{-1}}^2 \right] + \frac{1}{4}\alpha_t\tag{\pref{cor: special lifted trace bound}}
%\\&\le 2\alpha_t (d+1) - \mathbb{E}_t^{\calE}\left[\frac{\alpha_t}{4}  \|u^{\calA_0} - u\|_{\hatSigma_t^{-1}}^2 + \frac{\alpha_t}{4}  \|u - \hat{x}_t\|_{\hatSigma_t^{-1}}^2 - \frac{\alpha_t}{2} \left\langle u^{\calA_0} - u, u - \hat{x}_t \right\rangle \right] \tag{$\pmb{\hat\Sigma}_t \succeq \frac
%{1}{2}\mathbb{E}_{\calA_0 \sim D}[\pmb{H}_t^{\calA_0}]$ given $\calE_{t-1}$}
\\&\le 2\alpha_t(d+2) - \E_t^{\calE}\left[\frac{\alpha_t}{4}  \|u - \hat{x}_t\|_{\hatSigma_t^{-1}}^2\right]
\\&\le 2\alpha_t(d+2) - \E_t^{\calE}\left[\frac{\alpha_t}{4}  \|u - x_t\|_{\hatSigma_t^{-1}}^2 - \frac{\alpha_t}{4} \|\hat{x}_t- x_t\|_{\hatSigma_t^{-1}}^2 \right]
\\&\le 2\alpha_t(d+2) - \frac{\alpha_t}{4}  \|u - x_t\|_{\hatSigma_t^{-1}}^2 + \order\left(\frac{d^3\alpha_t\log{(T/\delta)}}{t}\right)
\tag{\pref{lem: local norm concentration for vectors}}
\end{align*}
On the other hand, since $\pmb{\hat\Sigma}_t\succeq \frac{1}{t}\pmb{I}  \succeq \frac{1}{T}\pmb{I}$, we have trivial bound
\begin{align*}
\mathbb{E}_t\left[\left\langle \pmb{H}_t^{\calA_0} - \pmb{U}^{\calA_0}, \alpha_t\pmb{\hat\Sigma}^{-1}_t \right\rangle ~\Big|~ \overline{\calE_{t-1}}  \right] &\le \order(\alpha_t T)
\end{align*}

Therefore, we have
 \begin{align*}
   \textbf{\textup{Bonus}} &= \mathbb{E}\left[\sum_{t=1}^T\left\langle \pmb{H}_t^{\calA_0} - \pmb{U}^{\calA_0}, \alpha_t\pmb{\hat\Sigma}^{-1}_t  \right\rangle \right] 
    \\&= \E\left[\sum_{t=1}^T \E_t\left[ \left\langle \pmb{H}_t^{\calA_0} - \pmb{U}^{\calA_0}, \alpha_t\pmb{\hat\Sigma}^{-1}_t  \right\rangle \right] \right]
    \\&=  \E\left[\sum_{t=1}^T \E_t\left[ \left\langle \pmb{H}_t^{\calA_0} - \pmb{U}^{\calA_0}, \alpha_t\pmb{\hat\Sigma}^{-1}_t \right\rangle ~\Big|~ \calE_{t-1}\right]\ind\{\calE_{t-1}\}\right] + \E\left[ \sum_{t=1}^T \E_t\left[ \left\langle \pmb{H}_t^{\calA_0} - \pmb{U}^{\calA_0}, \alpha_t\pmb{\hat\Sigma}^{-1}_t \right\rangle ~\Big|~ \overline{\calE_{t-1}} \right]\ind\{\overline{\calE_{t-1}}\}  \right]
    \\&\le 2(d+2)\sum_{t=1}^T\alpha_t - \frac{(1-\delta)}{4}\sum_{t=1}^T \alpha_t  \|u - x_t\|_{\hatSigma_t^{-1}}^2 + \order\left( \sum_{t=1}^T\frac{d^3\alpha_t\log(T/\delta)}{t} + \delta T\sum_{t=1}^T \alpha_t \right) 
    \\&\leq 2(d+2)\sum_{t=1}^T\alpha_t - \frac{1}{4}\sum_{t=1}^T \alpha_t  \|u - x_t\|_{\hatSigma_t^{-1}}^2 + \order\left( \sum_{t=1}^T\frac{d^3\alpha_t\log{(T/\delta)}}{t} + \delta T \sum_{t=1}^T \alpha_t \right)
\end{align*}
\end{proof}

\subsection{Bounding the Penalty term}
\begin{lemma}\label{lem:bound penalty} 
$\overline{\pmb{U}}^{\calA_0}$, we have
\begin{equation*}
\frac{F(\overline{\pmb{U}}^{\calA_0}) - \min_{\pmb{H} \in \mathcal{H}^{\calA_0}}F(\pmb{H})}{\eta_T} \le \frac{2d\log(T)}{\eta_T}
\end{equation*}
\end{lemma}

\begin{proof}
Since $ \overline{\pmb{U}}^{\calA_0} = \left(1-\frac{1}{T^2}\right) \pmb{U}^{\calA_0} + \frac{1}{T^2} \pmb{H}^{\calA_0}_*$, we have $\overline{\pmb{U}}^{\calA_0} \succeq \frac{1}{T^2} \pmb{H}^{\calA_0}_*$. Then 

% Define $\mathcal{U}^{\calA_0} = (1-\frac{1}{T})\mathcal{H}^{\calA_0} + \frac{1}{T}\pmb{H}_0^*$.
% We have $\mathcal{U}^{\calA_0}  \in \mathcal{H}^{\calA_0}$ because $\mathcal{H}^{\calA_0}$ is convex and $\pmb{H}_0^* \in \mathcal{H}^{\calA_0}$. 

% Thus, for all $\pmb{U}' \in \mathcal{U}^{\calA_0}$ and $\pmb{H}_0 \in \mathcal{H}^{\calA_0}$, 

% since $\pmb{U}' \succeq\frac{1}{T}\pmb{H}_0^*$, we have
\begin{equation*}
   \frac{F(\overline{\pmb{U}}^{\calA_0}) - \min_{\pmb{H} \in \mathcal{H}^{\calA_0}}F(\pmb{H})}{\eta_T} = \frac{1}{\eta_T}\log{\frac{\det(\pmb{H}^{\calA_0}_*)}{\det(\overline{\pmb{U}}^{\calA_0})}} \le \frac{2d\log(T)}{\eta_T}.
\end{equation*}

% Thus, given $\eta_0 = 1$, the Penalty term is bounded by
% \begin{equation*}
% d\log(T) + \sum_{t=1}^T \left(\frac{1}{\eta_t} - \frac{1}{\eta_{t-1}}\right)d\log(T)
% \end{equation*}
 
% For any $\pmb{U}^{\calA_0}_1 \in \mathcal{H}^{\calA_0}$, there is a $\pmb{U}^{\calA_0}_2 \in \mathcal{U}^{\calA_0}$ such that $\pmb{U}_2^{\calA_0} - \pmb{U}^{\calA_0}_1 = \frac{1}{T}(\pmb{H}_0^*  - \pmb{U}_1^{\calA_0})$. Thus, we have $\left\langle \pmb{H}_t^{\calA_0}-\pmb{U}_1^{\calA_0}, \gamma_t \right\rangle = \left\langle \pmb{H}_t^{\calA_0}-\pmb{U}_2^{\calA_0}, \gamma_t \right\rangle +  \left\langle \pmb{U}_2^{\calA_0} - \pmb{U}^{\calA}_1, \gamma_t \right\rangle =  \left\langle \pmb{H}_t^{\calA_0}-\pmb{U}_2^{\calA_0}, \gamma_t \right\rangle +  \left\langle \frac{1}{T}(\pmb{H}_0^*  - \pmb{U}_1^{\calA_0}), \gamma_t \right\rangle$ where the bound of the first term will be discussed later given 
%  $\pmb{U}_2^{\calA_0} \succeq \frac{1}{T}\pmb{H}_0^*$ and the second term is in the order of $\order(\frac{1}{T})$, which will not affect the final regret. 
\end{proof}

\subsection{Bounding the Stability-1 term}\label{app: bound the third term}
\cite{zimmert2022return} gave a useful identity to bound the Bregman divergence. We restate it in \pref{lem:bregman bound} for completeness.
\begin{lemma} \label{lem:bregman bound}
Let $\pmb{G} = \begin{bmatrix} G+gg^{\top}&g \\g^{\top}&1\end{bmatrix}$ and $\pmb{H} = \begin{bmatrix} H+hh^{\top}&h \\h^{\top}&1\end{bmatrix}$, we have 
\begin{equation*}
D(\pmb{G}, \pmb{H}) = D(G, H) +  \|g-h\|_{H^{-1}}^2 \ge  \|g-h\|_{H^{-1}}^2
\end{equation*}
\end{lemma}

\begin{proof}
\begin{align*}
D(\pmb{G}, \pmb{H}) &= F(\pmb{G}) - F(\pmb{H}) - \left\langle \nabla F(\pmb{H}), \pmb{G} - \pmb{H}\right\rangle
\\&= \log\left(\frac{\det(\pmb{H})}{\det(\pmb{G})}\right) + \tr(\pmb{H}^{-1}(\pmb{G} - \pmb{H}))
\\&= \log\left(\frac{\det(\pmb{H})}{\det(\pmb{G})}\right) + \tr(\pmb{H}^{-1}\pmb{G}) - d -1
\\&= \log\left(\frac{\det(\pmb{H})}{\det(\pmb{G})}\right) + \tr(\pmb{H}^{-1}\pmb{G}) - d -1
\\&= \log\left(\frac{\det(H)}{\det(G)}\right) + \tr(H^{-1}G) + \|g-h\|_{H^{-1}}^2 - d  \tag{\pref{lem: lifted trace bound}}
\\&= D(G, H) +  \|g-h\|_{H^{-1}}^2
\\&\ge  \|g-h\|_{H^{-1}}^2
\end{align*}
\end{proof}

\begin{lemma}\label{lem: stability term for loss}
For any  $\pmb{H} \in \mathcal{H}^{\calA_0}$, we have
\begin{equation*}
\textbf{\textup{Stability-1}} = \mathbb{E}\left[ \sum_{t=1}^T \left\langle \pmb{H}_t^{\calA_0} - \pmb{H}, \hat{\gamma}_{t} \right\rangle - \frac{D(\pmb{H}, \pmb{H}_t^{\calA_0})}{2\eta_t}  \right] \le 2d \sum_{t=1}^T \eta_t + \order(\delta T^2)
\end{equation*}
\end{lemma}

\begin{proof}
%First, assume that $\calE_{t-1}$ holds. 

%Recall that $\pmb{H}_t^{\calA_0} = \widehat{\Cov}(p_t^{\calA_0})$,  $H_t^{\calA_0} = \mathbb{E}_{a \sim p_t^{\calA_0}}[(a - \hat{x}_t)(a - \hat{x}_t)^\top]$, and $\Cov(p_t^{\calA_0}) = \mathbb{E}_{a \sim p_t^{\calA_0}}[(a - \mu(p_t^{\calA_0}))(a - \mu(p_t^{\calA_0}))^\top]$. We have
%\begin{align*}
%\lambda_{\max}\left(H_t^{\calA_0}\right) &= \max\limits_{\|x\|_2 \le 1} x^\top H_t^{\calA_0}x =  \mathbb{E}_{a \sim p_t^{\calA_0}}\left[\|a - \hat{x}_t\|_2^2\right] \ge \mathbb{E}_{a \sim p_t^{\calA_0}}\left[\|a - \mu(p_t^{\calA_0})\|_2^2\right]  = \lambda_{\max}\left(\Cov(p_t^{\calA_0})\right)
%\end{align*}
%Thus, given that $\hatSigma_t \succeq \frac{1}{2}\mathbb{E}_{\calA_0 \sim D}[H_t^{\calA_0}]$, we have $\hatSigma_t \succeq \frac{1}{2}\mathbb{E}_{\calA_0 \sim D}[\Cov(p_t^{\calA_0})]$.

\noindent Recall that $\pmb{H}^{\calA_0}_t=\widehat{\Cov}(p^{\calA_0}_t)$ and $\widehat{\Cov}(p) = \begin{bmatrix} \Cov(p) + \mu(p)\mu(p)^{\top} & \mu(p)\\ \mu(p)^{\top} & 1 \end{bmatrix}$, we have

\begin{align*}\left\langle \pmb{H}_t^{\calA_0} - \pmb{H}, \hat{\gamma}_{t} \right\rangle - \frac{D(\pmb{H}, \pmb{H}_t^{\calA_0})}{2\eta_t} &\le \left\langle x_t^{\calA_0}- \mu(p), \hat{y}_t \right\rangle - \frac{\|\mu(p)- x_t^{\calA_0}\|_{\Cov(p_t^{\calA_0})^{-1}}^2}{2\eta_t} \tag{\pref{lem:bregman bound}}
\\&\le \|x_t^{\calA_0} - \mu(p)\|_{\Cov(p_t^{\calA_0})^{-1}}\|\hat{y}_t\|_{\Cov(p_t^{\calA_0})}- \frac{\|\mu(p)- x_t^{\calA_0}\|_{\Cov(p_t^{\calA_0})^{-1}}^2}{2\eta_t}
\\&\le \frac{\eta_t}{2}\|\hat{y}_t\|_{\Cov(p_t^{\calA_0})}^2 \tag{AM-GM inequality}
\\&=  \frac{\eta_t}{2}\|\hat{\Sigma}_t^{-1}(a_t- \hat{x}_t)\ell_t \|_{\Cov(p_t^{\calA_0})}^2
\\&\le \frac{\eta_t}{2}(a_t- \hat{x}_t)^{\top}\hat{\Sigma}_t^{-1}\Cov(p_t^{\calA_0})\hat{\Sigma}_t^{-1}(a_t- \hat{x}_t)\tag{$|\ell_t| \le 1$}
\\& =  \frac{\eta_t}{2} \tr\left((a_t- \hat{x}_t)(a_t- \hat{x}_t)^{\top}\hat{\Sigma}_t^{-1}\Cov(p_t^{\calA_0})\hat{\Sigma}_t^{-1}\right)
\end{align*}
Since $\E_{\calA\sim\calD}\E_{a\sim p^{\calA}}\left[(a- \hat{x}_t)(a- \hat{x}_t)^{\top} \right]= H_t$, taking expectations over $\calA_t$, $a_t$ and $\calA_0$ conditioned on $\calE_{t-1}$, we have
\begin{align*}
\E_t^{\calE}\left[\left\langle \pmb{H}_t^{\calA_0} - \pmb{H}, \hat{\gamma}_{t} \right\rangle - \frac{D(\pmb{H}, \pmb{H}_t^{\calA_0})}{2\eta_t} \right] 
&\le  \E_t^{\calE} \left[ \frac{\eta_t}{2}\tr\left((a_t- \hat{x}_t)(a_t- \hat{x}_t)^{\top}\hat{\Sigma}_t^{-1}\Cov(p_t^{\calA_0})\hat{\Sigma}_t^{-1}\right) \right] 
\\&= \E_t^{\calE} \left[ \frac{\eta_t}{2}\tr\left(H_t\hatSigma_t^{-1}\E_{\calA_0 \sim D }\left[\Cov(p_t^{\calA_0})\right]\hatSigma_t^{-1}\right) \right]. 
\end{align*}
Notice that given $\calE_{t-1}$, 
\begin{align*}
    \hatSigma_t \succeq \frac{1}{2}H_t  = \frac{1}{2} \E_{\calA\sim D}[\Cov(p_t^\calA)]   + \frac{1}{2}(\hatx_t - x_t)(\hatx_t - x_t)^\top  \succeq \frac{1}{2}\E_{\calA\sim D}[\Cov(p_t^\calA)]
\end{align*}
Hence we continue to upper bound the last expression by
\begin{align*}
\E_t^{\calE} \left[ \eta_t\tr\left(H_t\hatSigma_t^{-1}\hatSigma_t\hatSigma_t^{-1}\right) \right]
\le \E_t^{\calE} \left[ \eta_t\tr\left(H_t\hatSigma_t^{-1}\right) \right] \leq 2\eta_t d. 
\end{align*}
On the other hand, since $\hat\Sigma_t\succeq \frac{1}{t}I  \succeq \frac{1}{T}I$, we have trivial bound
\begin{align*}
\E_t\left[\left\langle \pmb{H}_t^{\calA_0} - \pmb{H}, \hat{\gamma}_{t} \right\rangle - \frac{D(\pmb{H}, \pmb{H}_t^{\calA_0})}{2\eta_t} ~\bigg|~ \overline{\calE_{t-1}} \right] &\le \order(T)
\end{align*}
Combining everything, we get 
\begin{align*}
    \textbf{\textup{Stability-1}} &= \mathbb{E}\left[ \sum_{t=1}^T \left\langle \pmb{H}_t^{\calA_0} - \pmb{H}, \hat{\gamma}_{t} \right\rangle - \frac{D(\pmb{H}, \pmb{H}_t^{\calA_0})}{2\eta_t} \right]
    \\&= \E\left[\sum_{t=1}^T \E_t\left[\left\langle \pmb{H}_t^{\calA_0} - \pmb{H}, \hat{\gamma}_{t} \right\rangle - \frac{D(\pmb{H}, \pmb{H}_t^{\calA_0})}{2\eta_t} \right]\right]  
    \\&= \E\left[\sum_{t=1}^T \E_t\left[\left\langle \pmb{H}_t^{\calA_0} - \pmb{H}, \hat{\gamma}_{t} \right\rangle - \frac{D(\pmb{H}, \pmb{H}_t^{\calA_0})}{2\eta_t} ~\bigg|~ \calE_{t-1}\right]\ind\{\calE_{t-1}\}\right]
    \\& \qquad +  \E\left[ \sum_{t=1}^T \E_t\left[\left\langle \pmb{H}_t^{\calA_0} - \pmb{H}, \hat{\gamma}_{t} \right\rangle - \frac{D(\pmb{H}, \pmb{H}_t^{\calA_0})}{2\eta_t} ~\bigg|~ \overline{\calE_{t-1}}\right]\ind\{\overline{\calE_{t-1}}\}\right]
    \\&\le 2d \sum_{t=1}^T \eta_t + \order(\delta T^2).
\end{align*}

% Taking expectation over $\calA_t$ and $a_t$: 
% \begin{align*}
%     &\frac{\eta_t}{8} \tr(\hatSigma_t^{-1}\E_{\calA_t\sim\calD}\E_{a_t\sim p^{\calA_t}}\left[(a_t- \hat{x}_t)^{\top}(a_t- \hat{x}_t)\right]\hatSigma_t^{-1}\Cov(p_t^{\calA_0})) \\
%     &\leq \frac{\eta_t}{8} \tr(\hatSigma_t^{-1} H_t\hatSigma_t^{-1}\Cov(p_t^{\calA_0})) \\
%     &\leq \eta_t \tr(\hatSigma_t^{-1}\Cov(p_t^{\calA_0})) \tag{$\hatSigma_t \succeq \frac{1}{8}H_t$}
% \end{align*}
% Taking expectation over $\calA_0$: 
% \begin{align*}
%      \eta_t\tr(\hatSigma_t^{-1}\mathbb{E}_{\calA_0 \sim D}[\Cov(p_t^{\calA_0})]) \leq 8\eta_t d \tag{$\hatSigma_t \succeq \frac{1}{8}\mathbb{E}_{\calA_0 \sim D}[\Cov(p_t^{\calA_0})]$} 
% \end{align*}
\end{proof}

\subsection{Bounding the Stability-2 term} \label{app: bound the fourth term}
% For any matrix $A$ and $X$, define $\|A\|_{X} = \tr(XAXA)$. 
% \begin{lemma} \label{lem:matrix norm holder}
% For any two symmetric matrices $A,B$ and positive definite matrix $X$,  $\left\langle A, B \right\rangle^2 \le \|A\|_{X}\|B\|_{X^{-1}}$.
% \end{lemma}

% \begin{proof}
% From Von Neumann's trace inequality, for any $d \times d$ matrices $P,Q$, we have $|\tr(PQ)| \le \sum_{i=1}^d \sigma_i(P) \sigma_i(Q)$ where $\sigma_i(P)$ and $\sigma_i(Q)$ is the $i$-th largest singular value of $P$ and $Q$. Let $\lambda_i(P^\top P)$ and $\lambda_i(Q^\top Q)$ be the $i$-th largest eigenvalue of $P^{\top}P$ and $Q^{\top}Q$. From the definition of singular value, we have $|\tr(PQ)| \le \sum_{i=1}^d \sqrt{\lambda_i(P^\top P) \lambda_i(Q^\top Q)}$. Thus, we have
% \begin{align*}
% \left\langle A, B \right\rangle^2 &= \tr(AXX^{-1}B)^2
% \\&\le \left(\sum_{i=1}^d \sqrt{\lambda_i(XAAX) \lambda_i(X^{-1}BBX^{-1})}\right)^2
% \\&\le \left(\sum_{i=1}^d \lambda_i(XAAX) \right) \left(\sum_{i=1}^d \lambda_i(X^{-1}BBX^{-1}) \right)\tag{Cauchy-Schwarz inequality}
% \\&= \tr(XAXA)\tr(X^{-1}BX^{-1}B)
% \\&= \|A\|_{X}\|B\|_{X^{-1}}
% \end{align*}
% \end{proof}

\noindent Note that \pref{lem:matrix norm holder} does not require matrix $A,B$ to be positive semi-definite. We will use it to prove the following lemma based on Lemma 34 in \cite{dann2023blackbox}.

\begin{lemma}\label{lem: bonus stability}
If $\eta_t\alpha_t \le \frac{1}{64t}$, then 
\begin{equation*}
    \textbf{\textup{Stability-2}} = \mathbb{E}\left[\sum_{t=1}^T \max\limits_{\pmb{H} \in \mathcal{H}^{\calA_0}}\left\langle \pmb{H}_t^{\calA_0} - \pmb{H}, -\alpha_t\pmb{\hat\Sigma}^{-1}_t \right\rangle - \frac{D(\pmb{H}, \pmb{H}_t^{\calA_0})}{2\eta_t} \right] \le d\sum_{t=1}^T \alpha_t + \order\left(\delta T^2\right)
\end{equation*}
\end{lemma}
\begin{proof}
We first show that $\max\limits_{\pmb{H} \in \mathcal{H}^{\calA_0}}\left\langle \pmb{H}_t^{\calA_0} - \pmb{H}, -\alpha_t\pmb{\hat\Sigma}^{-1}_t \right\rangle - \frac{D(\pmb{H}, \pmb{H}_t^{\calA_0})}{2\eta_t} \le \frac{\alpha_t}{2}\|\pmb{\hat\Sigma}^{-1}_t \|_{\nabla^{-2}F(\pmb{H}_t^{\calA_0})}.$

 \noindent Define 
 \begin{equation*}
 G(\pmb{H}) = \left\langle \pmb{H}_t^{\calA_0} - \pmb{H}, -\alpha_t\pmb{\hat\Sigma}^{-1}_t \right\rangle - \frac{D(\pmb{H}, \pmb{H}_t^{\calA_0})}{2\eta_t}    
 \end{equation*}
 and $\lambda = \|\alpha_t \pmb{\hat\Sigma}^{-1}_t \|_{\nabla^{-2}F(\pmb{H}_t^{\calA_0})}$. Since $\pmb{\hatSigma}_t \succeq \frac{1}{t}I$, $\pmb{H}^{\calA_0}_t \preceq 2I$, $\eta_t\alpha_t \le \frac{1}{64t}$, we have
\begin{equation*}
\eta_t\lambda = \eta_t\|\alpha_t \pmb{\hat\Sigma}^{-1}_t \|_{\nabla^{-2}F(\pmb{H}_t^{\calA_0})}  
 = \eta_t\alpha_t \sqrt{\tr(\pmb{H}_t^{\calA_0}\pmb{\hat\Sigma}_t^{-1}\pmb{H}_t^{\calA_0}\pmb{\hat\Sigma}_t^{-1})} \le 2\eta_t\alpha_t t  \le  \frac{1}{32}.    
\end{equation*}
 Let $\pmb{H}'$ be the maximizer of $G$. Since $G(\pmb{H}_t^{\calA_0}) = 0$, we have $G(\pmb{H}') \ge 0$. It suffices to show $\|\pmb{H}' - \pmb{H}_t^{\calA_0}\|_{\nabla^{2}F(\pmb{H}_t^{\calA_0})} \le 16\eta_t\lambda$ because from \pref{lem:matrix norm holder}, it leads to
 \begin{equation*}
G(\pmb{H}') \le \| \pmb{H}_t^{\calA_0} - \pmb{H}'\|_{\nabla^{2}F(\pmb{H}_t^{\calA_0})} \|\alpha_t\pmb{\hat\Sigma}^{-1}_t\|_{\nabla^{-2}F(\pmb{H}_t^{\calA_0})} \le 16\eta_t \lambda \alpha_t \|\pmb{\hat\Sigma}^{-1}_t\|_{\nabla^{-2}F(\pmb{H}_t^{\calA_0})} = \frac{\alpha_t}{2}\|\pmb{\hat\Sigma}^{-1}_t \|_{\nabla^{-2}F(\pmb{H}_t^{\calA_0})}
 \end{equation*}

\noindent To show $\|\pmb{H}' - \pmb{H}_t^{\calA_0}\|_{\nabla^{2}F(\pmb{H}_t^{\calA_0})} \le 16\eta_t\lambda$, it suffices to show that for all $\pmb{U}$ such that $\|\pmb{U} - \pmb{H}_t^{\calA_0}\|_{\nabla^{2}F(\pmb{H}_t^{\calA_0})} = 16\eta_t \lambda$, $G(\pmb{U}) \le 0$. This is because given this condition, if $\|\pmb{H}' - \pmb{H}_t^{\calA_0}\|_{\nabla^{2}F(\pmb{H}_t^{\calA_0})} > 16\eta_t\lambda$, then there is a $\pmb{U}$ in the line segment between $\pmb{H}_t^{\calA_0}$ and $\pmb{H}'$ such that  $\|\pmb{U} - \pmb{H}_t^{\calA_0}\|_{\nabla^{2}F(\pmb{H}_t^{\calA_0})} = 16\eta_t\lambda$. From the condition, $G(\pmb{U}) \le 0 \le \min\{G(\pmb{H}_t^{\calA_0}), G(\pmb{H}')\}$ which contradicts to the strictly concave of $G$.

\noindent Now consider any $\pmb{U}$ such that $\|\pmb{U} - \pmb{H}_t^{\calA_0}\|_{\nabla^{2}F(\pmb{H}_t^{\calA_0})} = 16\eta_t \lambda$. By Taylor expansion, there exists $\pmb{U}'$ in the line segment between $\pmb{U}$ and $\pmb{H}_t^{\calA_0}$ such that
\begin{equation*}
G(\pmb{U}) \le \|\pmb{U} -  \pmb{H}_t^{\calA_0}\|_{\nabla^{2}F(\pmb{H}_t^{\calA_0})} \|\alpha_t\pmb{\hat\Sigma}^{-1}_t\|_{\nabla^{-2}F(\pmb{H}_t^{\calA_0})} - \frac{1}{4\eta_t}\|\pmb{U} -  \pmb{H}_t^{\calA_0}\|_{\nabla^{2}F(\pmb{U}')}^2
\end{equation*}
We have $\|\pmb{U}' - \pmb{H}_t^{\calA_0}\|_{\nabla^{2}F(\pmb{H}_t^{\calA_0})} \le \|\pmb{U} - \pmb{H}_t^{\calA_0}\|_{\nabla^{2}F(\pmb{H}_t^{\calA_0})} = 16\eta_t \lambda \le \frac{1}{2}$. From the Equation 2.2 in page 23 of \cite{nemirovski2004interior} (also appear in Eq.(5) of \cite{abernethy2009competing}) and $\log\det$ is a self-concordant function, we have  $\|\pmb{U} -  \pmb{H}_t^{\calA_0}\|_{\nabla^{2}F(\pmb{U}')}^2 \geq \frac{1}{4}\|\pmb{U} -  \pmb{H}_t^{\calA_0}\|_{\nabla^{2}F(\pmb{H}_t^{\calA_0})}^2$. Thus, we have 
\begin{equation*}
G(\pmb{U}) \le  \|\pmb{U} -  \pmb{H}_t^{\calA_0}\|_{\nabla^{2}F(\pmb{H}_t^{\calA_0})} \|\alpha_t\pmb{\hat\Sigma}^{-1}_t\|_{\nabla^{-2}F(\pmb{H}_t^{\calA_0})} - \frac{1}{16\eta_t}\|\pmb{U} -  \pmb{H}_t^{\calA_0}\|_{{(\pmb{H}_t^{\calA_0})^{-1}}}^2 = 16\eta_t \lambda^2 - \frac{(16\eta_t\lambda)^2}{16\eta_t} = 0
\end{equation*}

\noindent We have $\|\pmb{\hat\Sigma}^{-1}_t \|_{\nabla^{-2}F(\pmb{H}_t^{\calA_0})} = \sqrt{\tr(\pmb{H}_t^{\calA_0}\pmb{\hat\Sigma}^{-1}_t\pmb{H}_t^{\calA_0}\pmb{\hat\Sigma}^{-1}_t)}= \sqrt{\tr((\pmb{H}_t^{\calA_0}\pmb{\hat\Sigma}^{-1}_t)^2)}$.
Observe the following two facts: 1) all eigenvalues of $\pmb{H}_t^{\calA_0}\pmb{\hatSigma}_t^{-1}$ are non-negative since $\pmb{H}_t^{\calA_0}$ and $\pmb{\hatSigma}_t^{-1}$ are both positive semi-definite, 2) for a square matrix $A$ with all non-negative eigenvalues, $\tr(A^2)\leq \tr(A)^2$ because $\tr(A^2)=\sum_{i}\lambda_i(A^2)=\sum_i \lambda_i(A)^2 \leq (\sum_i \lambda_i(A))^2$. We have 
%Since $\pmb{H}_t^{\calA_0}$ is positive semi-definite, $\pmb{\hat\Sigma}^{-\frac{1}{2}}_t\pmb{H}_t^{\calA_0}\pmb{\hat\Sigma}^{-\frac{1}{2}}_t$ is positive semi-definite, we have  
\begin{align*}
    \sqrt{\tr((\pmb{H}_t^{\calA_0}\pmb{\hat\Sigma}^{-1}_t)^2)} \leq \tr(\pmb{H}_t^{\calA_0}\pmb{\hat\Sigma}^{-1}_t). 
\end{align*}

%Since the eigenvalues of $AB$ and $BA$ are the same, the eigenvalues of $\pmb{H}_t^{\calA_0}\pmb{\hat\Sigma}^{-1}_t$ are all non-negative. From the fact that $\tr(A^2) \le \tr(A)^2$ if all eigenvalues of $A$ are non-negative, we have  $\tr((\pmb{H}_t^{\calA_0}\pmb{\hat\Sigma}^{-1}_t)^2) \le \tr(\pmb{H}_t^{\calA_0}\pmb{\hat\Sigma}^{-1}_t)^2$. Thus, 
%we have $\|\pmb{\hat\Sigma}^{-1}_t \|_{\nabla^{-2}F(\pmb{H}_t^{\calA_0})} \le \tr(\pmb{H}_t^{\calA_0}\pmb{\hat\Sigma}^{-1}_t)$.
This allows us to conclude 
\begin{align*}    \mathbb{E}_t^{\calE}\left[\frac{\alpha_t}{2}\|\pmb{\hat\Sigma}^{-1}_t \|_{\nabla^{-2}F(\pmb{H}_t^{\calA_0})} \right] \le  \frac{\alpha_t}{2}\E_t^{\calE} \left[\tr(\pmb{H}_t^{\calA_0}\pmb{\hat\Sigma}^{-1}_t)\right] \le \alpha_t d
\end{align*}
where we use that $\pmb{\hat\Sigma}_t \succeq \frac{1}{2}\mathbb{E}_{\calA_0\sim D}[\pmb{H}_t^{\calA_0}]$ given $\calE_{t-1}$.%

On the other hand, since $\pmb{\hat\Sigma}_t\succeq \frac{1}{t}\pmb{I}  \succeq \frac{1}{T}\pmb{I}$, for any $t=1,\cdots, T$, we have trivial bound
\begin{align*}
\E_t\left[ \max\limits_{\pmb{H} \in \mathcal{H}^{\calA_0}}\left\langle \pmb{H}_t^{\calA_0} - \pmb{H}, -\alpha_t\pmb{\hat\Sigma}^{-1}_t \right\rangle - \frac{D(\pmb{H}, \pmb{H}_t^{\calA_0})}{2\eta_t} ~\bigg|~ \overline{\calE_{t-1}} \right] &\le \order(T)
\end{align*}
Overall, 
\begin{align*}
    \textbf{\textup{Stability-2}} &= \mathbb{E}\left[\sum_{t=1}^T \max\limits_{\pmb{H} \in \mathcal{H}^{\calA_0}}\left\langle \pmb{H}_t^{\calA_0} - \pmb{H}, -\alpha_t\pmb{\hat\Sigma}^{-1}_t \right\rangle - \frac{D(\pmb{H}, \pmb{H}_t^{\calA_0})}{2\eta_t} \right]
    \\&\le \E\left[\sum_{t=1}^T \E_t\left[\max\limits_{\pmb{H} \in \mathcal{H}^{\calA_0}}\left\langle \pmb{H}_t^{\calA_0} - \pmb{H}, -\alpha_t\pmb{\hat\Sigma}^{-1}_t \right\rangle - \frac{D(\pmb{H}, \pmb{H}_t^{\calA_0})}{2\eta_t} \right]\right]
     \\&= \E\left[\sum_{t=1}^T \E_t\left[ \max\limits_{\pmb{H} \in \mathcal{H}^{\calA_0}}\left\langle   \pmb{H}_t^{\calA_0} - \pmb{H}, \hat{\gamma}_{t} \right\rangle - \frac{D(\pmb{H}, \pmb{H}_t^{\calA_0})}{2\eta_t} ~\bigg|~ \calE_{t-1}\right]\ind\{\calE_{t-1}\} \right]
    \\& \qquad +  \E\left[ \sum_{t=1}^T\E_t\left[\max\limits_{\pmb{H} \in \mathcal{H}^{\calA_0}} \left\langle  \pmb{H}_t^{\calA_0} - \pmb{H}, \hat{\gamma}_{t} \right\rangle - \frac{D(\pmb{H}, \pmb{H}_t^{\calA_0})}{2\eta_t} ~\bigg|~ \overline{\calE_{t-1}}\right]\ind\{\overline{\calE_{t-1}}\}\right]
    \\&\le d\sum_{t=1}^T \alpha_t + \order\left(\delta T^2\right). 
\end{align*}
\end{proof}

\subsection{Bounding the Error term}
\begin{lemma}\label{lem:error term bound}
\begin{equation*}
    \textbf{\textup{Error}}= \mathbb{E}\left[\sum_{t=1}^{T} \left\langle\overline{\pmb{U}}^{\calA_0}-\pmb{U}^{\calA_0},\hat{\gamma}_t -  \alpha_t\pmb{\hatSigma}^{-1}_t \right\rangle\right] \le 
 \order(1). 
\end{equation*}
\end{lemma}

\begin{proof}
Since  $\overline{\pmb{U}}^{\calA_0} = \left(1-\frac{1}{T^2}\right) \pmb{U}^{\calA_0} + \frac{1}{T^2} \pmb{H}^{\calA_0}_*$, and $\hatSigma_t \succeq \frac{1}{T}I, \pmb{\hatSigma}_t \succeq \frac{1}{T}\pmb{I}$ we have 
\begin{align*}
\textbf{\textup{Error}} &= \mathbb{E}\left[\sum_{t=1}^{T} \left\langle\overline{\pmb{U}}^{\calA_0}-\pmb{U}^{\calA_0},\hat{\gamma}_t -  \alpha_t\pmb{\hatSigma}^{-1}_t \right\rangle\right] 
\\&= \mathbb{E}\left[\frac{1}{T^2}\sum_{t=1}^{T} \left\langle - \pmb{U}^{\calA_0}+\pmb{H}_*^{\calA_0}, \hat{\gamma}_t -  \alpha_t\pmb{\hatSigma}^{-1}_t \right\rangle\right]
\\&\le \order(1). 
\end{align*}

\end{proof}

\subsection{Finishing up}
Recall the regret decomposition at the beginning of \pref{app: regret analysis}. From  \pref{lem:bound penalty}, \pref{lem: stability term for loss}, \pref{lem: bonus stability}, and \pref{lem:error term bound}, we have
\begin{align*}
\textbf{FTRL-Reg} &= \textbf{Penalty} + \textbf{Stability-1} + \textbf{Stability-2} + \textbf{Error}
%\\&\le \frac{2d\log(T)}{\eta_T} +  2d \sum_{t=1}^T \eta_t + \order(\delta T^2) +   d\sum_{t=1}^T \alpha_t + \order(\delta \sum_{t=1}^T \alpha_t T) + \order(1 + \sum_{i=1}^T \alpha_t)
\\&\le \order\left(\frac{d\log(T)}{\eta_T} + d \sum_{t=1}^T\eta_t + d \sum_{t=1}^T\alpha_t + \delta T^2 \right) 
\end{align*}
From \pref{lem: handling bias} and \pref{lem: bonus term}, we can cancel out the additional regret induced by bias through the well-designed bonus term. Namely,
\begin{align*}
\textbf{Bias} + \textbf{Bonus} &= \frac{1}{4}\sum_{t=1}^T \alpha_t \|x_t-u\|_{\hatSigma_t^{-1}}^2 + \order\left(\sum_{t=1}^T \frac{d^3\log(T/\delta)}{\alpha_t t} + \delta T^2\right) 
\\ &\qquad + 2(d+2)\sum_{t=1}^T\alpha_t - \frac{1}{4}\sum_{t=1}^T \alpha_t  \|u - x_t\|_{\hatSigma_t^{-1}}^2 + \order\left( \sum_{t=1}^T\frac{d^3\alpha_t\log{\frac{T}{\delta}}}{t} + \delta \sum_{t=1}^T \alpha_t T\right)
\\&= \order\left(d\sum_{t=1}^T \alpha_t +  \sum_{t=1}^T \frac{d^3\log(T/\delta)}{\alpha_t t}  + \sum_{t=1}^T\frac{d^3\alpha_t\log{(T/\delta)}}{t} + \delta T^2 \right)
\end{align*}
Thus, we have 
\begin{align*}
\Reg &= \textbf{Bias} + \textbf{Bonus} + \textbf{FTRL-Reg}
\\&= \order\left( \frac{d\log(T)}{\eta_T} + d \sum_{t=1}^T\eta_t + d \sum_{t=1}^T\alpha_t + \sum_{t=1}^T \frac{d^3\log(T/\delta)}{\alpha_t t}  + \sum_{t=1}^T\frac{d^3\alpha_t\log{(T/\delta)}}{t} + \delta T^2  \right)
\end{align*}
Recall that we have an additional condition in \pref{lem: bonus stability} such that for any $t$, $\eta_t \alpha_t \le \frac{1}{64t}$. Picking $\alpha_t = \frac{d}{\sqrt{t}}, \eta_t = \frac{1}{64d\sqrt{t}}$ and $\delta = \frac{1}{T^2}$, we get
\begin{equation*}
\Reg = \order\left(d^2\sqrt{T}\log(T) + d^4\log(T) \right) = \order(d^2\sqrt{T}\log(T)) 
\end{equation*}
where we assume $d^2\leq \sqrt{T}$ without loss of generality (otherwise the bound is vacuous).

% For $u$ such that $u = \E_{a \sim p^{\calA_0}}[a]$ with $\widehat{\Cov}(p^{\calA_0})\in \calH^{\calA_0}$, add up all above terms 

\section{Handling Misspecification}\label{app: misspecification}
In this section, we discuss how to handle misspecification as defined in \pref{sec: misspecification}. In \pref{app: known misspecification}, we study the case where the amount of misspecification $\eps$ is known by the learner. In \pref{app: unknown misspecification}, we use a blackbox approach to turn it into an algorithm that achieves almost the same regret bound (up to $\log T$ factors) without knowning $\eps$. 

\subsection{Known misspecification} \label{app: known misspecification}
As discussed in \pref{sec: misspecification}, when the amount of misspecification $\eps$ is known, we still use \pref{alg: FTRL}, but with different $\alpha_t$ and $\eta_t$. Throughout this subsection, we let $\alpha_t = \frac{d}{\sqrt{t}} + \frac{\eps}{\sqrt{d}}$ and  $\eta_t = \frac{1}{64\left(d\sqrt{t} + \frac{\eps}{\sqrt{d}}t\right)}$, and point out the modifications of the analysis from \pref{app: regret analysis}. 

We start with the regret decomposition similar to that in \pref{app: regret analysis}, but here we define 
\begin{align*}
    y_t &= \argmin_{y\in\mathbb{B}^d_2}\max_{\calA\in \supp(D)}\max_{a\in\calA}|f_t(a) - \langle a, y \rangle|, \\
    \eps_t &= \max_{\calA\in \supp(D)}\max_{a\in\calA}|f_t(a) - \langle a, y_t \rangle|, \\
    c_t(a) &= f_t(a) - \langle a, y_t \rangle. 
\end{align*}
The regret decomposition goes as follows: 
\begin{align*}
\Reg(u) 
&= \E\left[\sum_{t=1}^T 
\left(f_t(a_t) - f_t(u^{\calA_t})\right) \right]\\
&\leq \mathbb{E}\left[\sum_{t=1}^{T} \left\langle a_t-u^{\calA_t}, y_t \right\rangle\right] + \sum_{t=1}^T 
\eps_t \\
&\leq \mathbb{E}\left[\sum_{t=1}^{T} \left\langle \pmb{H}_t^{\calA_t}-\pmb{U}^{\calA_t}, \gamma_t \right\rangle\right] + \eps T=  \mathbb{E}\left[\sum_{t=1}^{T} \left\langle \pmb{H}_t^{\calA_0}-\pmb{U}^{\calA_0}, \gamma_t \right\rangle\right] + \eps T
\\
&\le 
\underbrace{\mathbb{E}\left[\sum_{t=1}^{T} \left\langle \pmb{H}_t^{\calA_0}-\pmb{U}^{\calA_0}, \gamma_t - \hat{\gamma}_t \right\rangle\right]}_{\textbf{Bias}}+
 \underbrace{\mathbb{E}\left[\sum_{t=1}^{T} \left\langle\pmb{H}_t^{\calA_0}-\pmb{U}^{\calA_0}, \alpha_t\pmb{\hatSigma}^{-1}_t \right\rangle\right]}_{\textbf{Bonus}} 
 \\
 &\qquad + 
 \underbrace{\mathbb{E}\left[\sum_{t=1}^{T} \left\langle\pmb{H}_t^{\calA_0}-\pmb{U}^{\calA_0},\hat{\gamma}_t -  \alpha_t\pmb{\hatSigma}^{-1}_t \right\rangle\right]}_{\textbf{FTRL-Reg}} + \eps T. 
 \end{align*}

Now $\hat{y}_t = \hatSigma_t^{-1}(a_t - \hat{x}_t)\ell_t$ with $\E[\ell_t]=a_t^\top y_t + c_t(a_t)$. 

For the \textbf{Bias} term, the proof is almost the same as \pref{lem: handling bias}. The only difference is that from the fourth line, we have 
\begin{align*}
    \E_t\left[ (x_t - u)^\top\left(y_t - \hatSigma_t^{-1}(a_t-\hat{x}_t)\left(a_t^\top y_t + c_t(a_t)\right) \right) \right] 
\end{align*}
for some $c_t(a_t)$ such that $|c_t(a_t)|\leq \eps_t$. This leads to an additional term of 
\begin{align*}
    &\E_t^{\calE}\left[- (x_t-u)^\top \hatSigma_t^{-1} (a_t-\hat{x}_t)c_t(a_t) \right] \\
    &\leq \E_t^{\calE}\left[\sqrt{(x_t-u)^\top \hatSigma_t^{-1} c_t(a_t)^2(a_t-\hatx_t)(a_t-\hatx_t)^\top \hatSigma_t^{-1} (x_t-u) } \right] \\
    &\leq \E_t^{\calE}\left[\sqrt{(x_t-u)^\top \hatSigma_t^{-1} \E_{\calA_t,a_t} \left[c_t(a_t)^2(a_t-\hatx_t)(a_t-\hatx_t)^\top \right]\hatSigma_t^{-1} (x_t-u) } \right] \\
    &\leq \E_t^{\calE}\left[\eps_t\sqrt{(x_t-u)^\top \hatSigma_t^{-1} \left(\E_{\calA_t,a_t} \left[(a_t-\hatx_t)(a_t-\hatx_t)^\top \right] \right)\hatSigma_t^{-1} (x_t-u) } \right] \\
    &\leq \E_t^{\calE}\left[\eps_t \sqrt{(x_t-u)^\top \hatSigma_t^{-1} 
H_t \hatSigma_t^{-1}(x_t-u) }\right] \\
&\leq \eps_t  \|x_t-u\|_{\hatSigma_t^{-1}}%  \\
%&\leq \alpha_t \|x_t-u\|^2_{\hatSigma_t^{-1}} + \order\left(\frac{\eps_t^2}{\alpha_t}\right)
\end{align*}

% We will pick $\alpha_t = \frac{d}{\sqrt{t}} + \frac{\eps}{\sqrt{d}}$. $\eta=\frac{1}{d\sqrt{t} + \frac{\eps}{\sqrt{d}}t}$

% For the \textbf{Bias} term, the proof is almost the same as \pref{lem: handling bias}. The only difference is that from the fourth line, we have 
% \begin{align*}
%     \E_t\left[ (x_t - u)^\top\left(y_t - \hatSigma_t^{-1}(a_t-\hat{x}_t)\left(a_t^\top y_t + c_t(a_t)\right) \right) \right] 
% \end{align*}
% for some $c_t(a_t)$ such that $|c_t(a_t)|\leq \eps_t$. This leads to an additional term of 
% \begin{align*}
%     &\E_t^{\calE}\left[- (x_t-u)^\top \hatSigma_t^{-1} (a_t-\hat{x}_t)c_t(a_t) \right] \\
%     &= \E_t^{\calE}\left[- (x_t-u)^\top \hatSigma_t^{-1} (x_t-\hat{x}_t)c_t(a_t) \right]  \tag{taking expectation of $a_t$}
%     \\&\leq \E_t^{\calE}\left[\left|(x_t-u)^\top \hatSigma_t^{-1} (x_t-\hat{x}_t)c_t(a_t)\right| \right]
%     \\&\le \eps_t \E_t^{\calE}\left[\left|(x_t-u)^\top \hatSigma_t^{-1} (x_t-\hat{x}_t)\right| \right]
%     \\&\le \eps_t \E_t^{\calE}\left[ \|x_t-u\|_{\hatSigma_t^{-1}}  \|x_t-\hat{x}_t\|_{\hatSigma_t^{-1}} \right]
% \\&\leq \order\left(\eps_t \sqrt{\frac{d^3\log(T/\delta)}{t}} \right) \|x_t-u\|_{\hatSigma_t^{-1}}\tag{\pref{lem: local norm concentration for vectors} given $\calE_{t-1}$} 
% \end{align*}
Plugging it into the proof of \pref{lem: handling bias}, we have
\begin{align*}
 \E_t^{\calE}\left[ \left\langle \pmb{H}_t^{\calA_0} - \pmb{U}^{\calA_0}, \gamma_t - \hat{\gamma}_t \right\rangle \right] &\le \order\left(\sqrt{\frac{d^3\log(T/\delta)}{t}}  + \eps_t \right) \|x_t-u\|_{\hatSigma_t^{-1}}
 \\& \le  \frac{\alpha_t}{4}\|x_t-u\|_{\hatSigma_t^{-1}}^2 +  \order\left(\frac{d^3\log(T/\delta)}{\alpha_t t}  + \frac{\eps_t^2 }{\alpha_t}\right)
\end{align*}

%On the other hand, since  $\|\hat{y}_t\|_2 \le \order(T)$,  we have trivial bound
%\begin{align*}
%\E_t\left[ \left\langle \pmb{H}_t^{\calA_0} - \pmb{U}^{\calA_0}, \gamma_t - \hat{\gamma}_t \right\rangle ~|~  \overline{\calE_{t-1}} \right] \le \order(T)
%\end{align*}

Other parts of the proof follow those in \pref{lem: handling bias}. Finally, we get 
 \begin{align*}
    \textbf{\textup{Bias}} &= \E\left[\sum_{t=1}^T \left\langle \pmb{H}_t^{\calA_0} - \pmb{U}^{\calA_0}, \gamma_t - \hat{\gamma}_t \right\rangle \right] 
    %\\&= \E\left[\sum_{t=1}^T \E_t\left[ \left\langle \pmb{H}_t^{\calA_0} - \pmb{U}^{\calA_0}, \gamma_t - \hat{\gamma}_t \right\rangle \right] \right]
    %\\&=  \E\left[\sum_{t=1}^T \E_t\left[ \left\langle \pmb{H}_t^{\calA_0} - \pmb{U}^{\calA_0}, \gamma_t - \hat{\gamma}_t \right\rangle ~|~ \calE_{t-1}\right]\ind\{\calE_{t-1}\}\right] + \E\left[ \sum_{t=1}^T \E_t\left[ \left\langle \pmb{H}_t^{\calA_0} - \pmb{U}^{\calA_0}, \gamma_t - \hat{\gamma}_t \right\rangle ~|~ \overline{\calE_{t-1}} \right]\ind\{\overline{\calE_{t-1}}\}\right] 
 \\&\le \frac{1}{4}\sum_{t=1}^T \alpha_t \|x_t-u\|_{\hatSigma_t^{-1}}^2 + \order\left(\sum_{t=1}^T \frac{d^3\log(T/\delta)}{\alpha_t t} + \sum_{t=1}^T \frac{\eps_t^2 }{\alpha_t} + \delta  T^2\right) 
\end{align*}

The \textbf{Bonus} term will not be affected, according to \pref{lem: bonus term}, we have
\begin{align*}
\textbf{\textup{Bonus}} \le 2(d+2)\sum_{t=1}^T\alpha_t - \frac{1}{4}\sum_{t=1}^T \alpha_t  \|u - x_t\|_{\hatSigma_t^{-1}}^2 + \order\left( \sum_{t=1}^T\frac{d^3\alpha_t\log{(T/\delta))}}{t} + \delta T^2\right)
\end{align*}

The \textbf{Penalty} term will not be affected, according to \pref{lem:bound penalty}, we have
\begin{equation*}
\frac{F(\overline{\pmb{U}}^{\calA_0}) - \min_{\pmb{H} \in \mathcal{H}^{\calA_0}}F(\pmb{H})}{\eta_T} \le \frac{2d\log(T)}{\eta_T}
\end{equation*}

\textbf{Stability-1} term is also unchanged, as we assume that $\ell_t$ still lies in $[-1,1]$ even under misspecification. We still have 
\begin{align*}
    \textbf{\textup{Stability-1}} 
    \le \order\left(d \sum_{t=1}^T \eta_t +  \delta T^2\right)
\end{align*}

The \textbf{Stability-2} term will not be affected as long as $\eta_t \alpha_t \le \frac{1}{64t}$. According to \pref{lem: bonus stability}, we have
\begin{equation*}
    \textbf{\textup{Stability-2}} \le \order\left(d\sum_{t=1}^T \alpha_t + \delta T^2\right)
\end{equation*}

The \textbf{Error} term is also unaffected. We still have $\textbf{Error} = \order(1)$. 
%the proof is almost the same as \pref{lem:error term bound}, for some $c_t(a_t)$ such that $|c_t(a_t)|\leq \eps_t = \order(1)$, we have 
\iffalse
\begin{align*}
\textbf{\textup{Error}} &= \mathbb{E}\left[\sum_{t=1}^{T} \left\langle\overline{\pmb{U}}^{\calA_0}-\pmb{U}^{\calA_0},\hat{\gamma}_t -  \alpha_t\pmb{\hatSigma}^{-1}_t \right\rangle\right] 
\\&= \mathbb{E}\left[\sum_{t=1}^{T} \left\langle\frac{1}{T^2}(\overline{\pmb{U}}^{\calA_0}-\pmb{H}_*^{\calA_0}), \hat{\gamma}_t -  \alpha_t\pmb{\hatSigma}^{-1}_t \right\rangle\right]
\\&\le \frac{1}{T^2}\order(T^2 + \sum_{i=1}^T \alpha_t T)\tag{$\|\hat{y}_t\| \le \order(T)$ and $\pmb{\hatSigma}_t \succeq \frac{1}{T}\pmb{I}$}
\\&= \order(1 + \sum_{i=1}^T \frac{\alpha_t}{T})
\end{align*}
\fi

Adding these terms together, the regret caused by bias and the negative term induced by bonus cancel out. We have
% \begin{align*}
% &\Reg =
% \\&\order\left( \frac{d\log(T)}{\eta_T} + d \sum_{t=1}^T(1+\eps_t)^2\eta_t + d \sum_{t=1}^T\alpha_t + \sum_{t=1}^T \frac{d^3\log(T/\delta)}{\alpha_t t}  + \sum_{t=1}^T\frac{d^3\alpha_t\log{(T/\delta)}}{t} + \sum_{t=1}^T \frac{\eps_t^2 }{\alpha_t} + \delta T\sum_{t=1}^T(1+\eps_t+\alpha_t) \right)
% \end{align*}
\begin{align*}
\Reg =\order\left( \frac{d\log(T)}{\eta_T} + d \sum_{t=1}^T(\eta_t + \alpha_t) + \sum_{t=1}^T \frac{d^3\log(T/\delta)}{\alpha_t t}  + \sum_{t=1}^T\frac{d^3\alpha_t\log{(T/\delta)}}{t} + \sum_{t=1}^T \frac{\eps_t^2 }{\alpha_t} + \delta T^2 \right)
\end{align*}
% It seems we can not pick $\alpha_t = \frac{d}{\sqrt{t}} + \frac{\eps}{\sqrt{d}}$. $\eta_t =\frac{1}{d\sqrt{t} + \frac{\eps}{\sqrt{d}}}$ to get the optimal solution?

% Picking $\alpha_t  = \frac{d}{\sqrt{t}}$ and $\eta_t = \frac{1}{d\sqrt{t}}$ will get $\otil(d^2\sqrt{T} + \eps T)$ 
%\CW{This choice makes $\sum_{t} \frac{\epsilon_t^2}{\alpha_t}=\frac{1}{d}\sum_t \epsilon_t^2 \sqrt{t}$ too large} 
 
Recall that we pick $\alpha_t = \frac{d}{\sqrt{t}} + \frac{\eps}{\sqrt{d}}$. $\eta_t = \frac{1}{64d\sqrt{t} + 64\frac{\eps}{\sqrt{d}}t}$ and $\delta = \frac{1}{T^2}$. This gives
\begin{equation*}
\Reg = \order(d^2\sqrt{T}\log(T) + d^4\log(T) + \sqrt{d}\eps T) = \order(d^2\sqrt{T}\log(T)+ \sqrt{d}\eps T) 
\end{equation*}
where we assume $d^2 \leq \sqrt{T}$ without loss of generality. 

%\begin{align*}
%    &\E_t\left[- (x_t-u)^\top \hatSigma_t^{-1} (a_t-\hat{x}_t)c_t(a_t) \right] \\
%    &\leq \eps_t\E_t\left[\|x_t-u\|_{\hatSigma_t^{-1}} \|a_t-\hat{x}_t\|_{\hatSigma_t^{-1}}  \right] \\
%    &\leq \alpha_t \|x_t-u\|^2_{\hatSigma_t^{-1}} + \frac{\eps_t^2}{\alpha_t} \E_t\left[\|a_t - \hat{x}_t\|^2_{\hatSigma_t^{-1}}\right] \\
%    &\leq \alpha_t \|x_t-u\|^2_{\hatSigma_t^{-1}} + \frac{\eps_t^2}{\alpha_t} \E_t\left[\|a_t - x_t\|^2_{\hatSigma_t^{-1}} + \|x_t - \hat{x}_t\|^2_{\hatSigma_t^{-1}}\right] \\
%    &\lesssim \alpha_t \|x_t-u\|^2_{\hatSigma_t^{-1}} + \frac{\eps_t^2}{\alpha_t}\left(\tr(H_t\hatSigma_t^{-1})\right) \\
%    &\lesssim \alpha_t \|x_t-u\|^2_{\hatSigma_t^{-1}} + \frac{d\eps_t^2}{\alpha_t}
%\end{align*}

\subsection{Unknown misspecification} \label{app: unknown misspecification}

In this subsection, we use a model selection technique to convert the algorithm in \pref{app: known misspecification} which requires knowledge on $\eps$ into an algorithm that achieves a similar regret bound without knowing $\eps$. Such a procedure to handle unknown misspecification/corruption has appeared in several previous works \citep{foster2020adapting, wei2022model}, though we adopt the technique in \cite{jin2023no} to handle the adversarial case. %\footnote{Since \cite{jin2023no} has not been published, for completeness, we restate all their results in \pref{app: unknown misspecification}. The goal is to use their reduction idea to handle the unknown misspecification case. We do not claim our contribution in the reduction idea. }

The idea here is a black-box reduction which turns an algorithm that only deals with known $\eps$ to one that handles unknown $\eps$. 
More specifically, the reduction has two layers. The bottom layer takes as input an arbitrary misspecification-robust algorithm that operates under known $\eps$ (e.g., \pref{alg: FTRL}), and outputs a \emph{stable} misspecification-robust algorithm (formally defined later) that still operates under known $\eps$. The top layer follows the standard Corral idea and takes as input a stable algorithm that operates under known $\eps$, and outputs an algorithm that operates under unknown $\eps$. Below, we explain these two layers of reduction in details. 

\paragraph{Bottom Layer (from an Arbitrary Algorithm to a Stable Algorithm)}

The input of the bottom layer is an arbitrary misspecification-robust algorithm, formally defined as: 
\begin{definition}\label{def: base alg}
An algorithm is misspecification-robust if it takes $\theta$ as input, and achieves the following regret for any random stopping time $t'\leq T$ and any policy~$u$: 
    \begin{align*}
        \E\left[\sum_{t=1}^{t'}  
(f_t(a_t) - f_t(u^{\calA_t})) \right] \leq \E\left[c_1 \sqrt{ t'} + c_2  \theta \right] + \Pr\left[\eps_{1;t'}>\theta\right] T 
    \end{align*}
    for problem-dependent and $\log(T)$ factors $c_1, c_2\geq 1$ and $\eps_{1:t'}\triangleq \sqrt{t'\sum_{\tau=1}^{t'}\eps_\tau^2}$. 
\end{definition}
In our case, $c_1 = \Theta(d^2\log T)$ and $c_2=\Theta(\sqrt{d})$. 
While the regret bound in \pref{def: base alg} might look cumbersome, it is in fact fairly reasonable: if the guess $\theta$ is not smaller than the true amount of $\eps_{1:t'}$, the regret should be of order $d^2\sqrt{t'}+\sqrt{d}\theta$; otherwise, the regret bound is vacuous since $T$ is its largest possible value.
The only extra requirement is that the algorithm needs to be \emph{anytime} (i.e., the regret bound holds for any stopping time $t'$),
but even this is known to be easily achievable by using a doubling trick over a fixed-time algorithm.
It is then clear that \pref{alg: FTRL} (together with a doubling trick) indeed satisfies \pref{def: base alg}. 

As mentioned, the output of the bottom layer is a stable robust algorithm. To characterize stability, we follow~\cite{agarwal2017corralling} and define a new learning protocol that abstracts the interaction between the output algorithm of the bottom layer and the master algorithm from the top layer:
\begin{protocol}\label{proto: stable protocol}
    \textup{In every round $t$, before the learner makes a decision, a probability $w_t\in[0,1]$ is revealed to the learner. After making a decision, the learner sees the desired feedback from the environment with probability $w_t$, and sees nothing with probability $1-w_t$. 
}
\end{protocol}
 
One can convert any misspecification-robust algorithm (defined in \pref{def: base alg}) into a stable misspecification-robust algorithm (characterized in \pref{thm: conversion thm}). 

This conversion is achieved by a procedure that called \stable (see \pref{alg: split} for details).
The high-level idea of \stable is as follows. 
Noticing that the challenge when learning in \pref{proto: stable protocol} is that $w_t$ varies over time,
we discretize the value of $w_t$ and instantiate one instance of the input algorithm to deal with one possible discretized value, so that it is learning in \pref{proto: stable protocol} but with a \emph{fixed} $w_t$, making it straightforward to bound its regret based on what it promises in \pref{def: base alg}.

\begin{algorithm}[t]
    \caption{\textbf{ST}able \textbf{A}lgorithm \textbf{B}y \textbf{I}ndependent \textbf{L}earners and \textbf{I}nstance \textbf{SE}lection (\stable)} \label{alg: split}
    \textbf{Input}: $\eps$ and a base algorithm satisfying \pref{def: base alg}. \\ 
    %\textbf{Define}: %Let $\cmax=2L$ be a global upper bound of $C_t$.  
    %\\
    \textbf{Initialize}: $\lceil \log_2 T \rceil$ instances of the base algorithm $\alg_{1}, \ldots, \alg_{\lceil \log_2 T \rceil}$, where $\alg_{j}$ is configured with the parameter 
    \begin{align*}
        \theta=\theta_j \triangleq 2^{-j}\eps T + 4\sqrt{2^{-j}T\log T} + 8\log(T). 
    \end{align*}
    
    \For{$t=1,2,\ldots$}{
         Receive $w_t$. \\
         \If{$w_t\leq \frac{1}{T}$}{
            play an arbitrary policy $\pi_t$ \\
            \textbf{continue} (without updating any instances)
         }
         Let $j_t$ be such that $w_t\in (2^{-j_t-1}, 2^{-j_t}]$. \\
         Let $\pi_t$ be the policy suggested by $\alg_{j_t}$. \\
         Output $\pi_t$. \\
         If feedback is received, send it to $\alg_{j_t}$ with probability $\frac{2^{-j_t-1}}{w_t}$, and discard it otherwise. 
    }
\end{algorithm}

More concretely, \stable instantiates $\order(\log_2 T)$ instances $\{\alg_j\}_{j=0}^{\lceil\log_2 T\rceil}$ of the input algorithm that satisfies \pref{def: base alg}, each with a different parameter $\theta_j$. 
Upon receiving $w_t$ from the environment, it dispatches round $t$ to the $j$-th instance where $j$ is such that $w_t\in(2^{-j-1}, 2^{-j}]$, and uses the policy generated by $\alg_{j}$ to interact with the environment (if $w_t\leq \frac{1}{T}$, simply ignore this round). 
Based on \pref{proto: stable protocol}, the feedback for this round is received with probability $w_t$.
To \emph{equalize} the probability of $\alg_j$ receiving feedback as mentioned in the high-level idea, 
when the feedback is actually obtained,
\stable sends it to $\alg_j$ only with probability $\frac{2^{-j-1}}{w_t}$ (and discards it otherwise). 
This way, every time $\alg_j$ is assigned to a round, it always receives the desired feedback with probability $w_t\cdot \frac{2^{-j-1}}{w_t}=2^{-j-1}$. This equalization step allows us to use the original guarantee of the base algorithm (\pref{def: base alg}) and run it as it is, without requiring it to perform extra importance weighting steps as in \cite{agarwal2017corralling}. 

The choice of $\theta_j$ is crucial in making sure that \stable only has $\eps T$ regret overhead instead of $\frac{\eps T}{\min_{t\in[T]}w_t}$. 
Since $\alg_j$ only receives feedback with probability $2^{-j-1}$,
the expected total misspecification it experiences is on the order of $2^{-j-1}\eps T$. Therefore, its input parameter $\theta_j$ only needs to be of this order instead of the total amount of misspecification $\eps T$. 

The formal guarantee of the conversion is stated in the following \pref{thm: conversion thm}. 

\begin{theorem}\label{thm: conversion thm}
    If an algorithm is misspecification robust according to \pref{def: base alg} for some constants $(c_1, c_2)$, then \pref{alg: split} ensures 
    \begin{align*}
        \Reg \leq \order\left(\E\left[c_1'\sqrt{T \rho_T} \right] + c_2' \eps T\right)
    \end{align*}
    under \pref{proto: stable protocol}, where $\rho_T=\frac{1}{\min_{t\in[T]}w_t}$, with $c_1'= \Theta((c_1+c_2)\sqrt{\log T})$.  
\end{theorem}

\begin{proof}[Proof of \pref{thm: conversion thm}]
    Define indicators 
    \begin{align*}
        g_{t,j} &= \ind\{w_t\in (2^{-j-1}, 2^{-j}]\} \\
        h_{t,j} &= \ind\{\alg_j \text{\ receives the feedback for episode $t$}\}. 
    \end{align*}
    Now we consider the regret of $\alg_j$. Notice that $\alg_j$ makes an update only when $g_{t,j}h_{t,j}=1$. By the guarantee of the base algorithm (\pref{def: base alg}), we have 
    \begin{align}
        &\E\left[\sum_{t=1}^T (f_t(a_t) - f_t(u^{\calA_t})) g_{t,j}h_{t,j} \right]  \nonumber \\
        &\leq \E\left[c_1 \sqrt{\sum_{t=1}^T g_{t,j}h_{t,j} }+ c_2\theta_j \max_{t\leq T}g_{t,j}\right] + \Pr\left[ 
\sqrt{\left(\sum_{t=1}^T  g_{t,j}h_{t,j}\right)\left(\sum_{t=1}^T \eps_t^2 g_{t,j}h_{t,j}\right)} > \theta_j \right]T.   \label{eq: tmp 22}
    \end{align}
    We first bound the last term: Notice that $\E[h_{t,j}| g_{t,j}]=2^{-j-1} g_{t,j}$ by \pref{alg: split}. Therefore, 
    \begin{align}
        \sum_{t=1}^T \eps_t^2 g_{t,j}\E[h_{t,j}|g_{t,j}] &= 2^{-j-1}\sum_{t=1}^T \eps_t^2 g_{t,j} \leq 2^{-j-1}\eps^2 T \label{eq: tmp 33} \\
        \sum_{t=1}^T g_{t,j}\E[h_{t,j}|g_{t,j}] &= 2^{-j-1}\sum_{t=1}^T g_{t,j} \leq 2^{-j-1} T \label{eq: tmp 38}
    \end{align}
    By Freedman's inequality, with probability at least $1-\frac{1}{T^2}$, 
    \begin{align*}
         &\sum_{t=1}^T \eps_t^2 g_{t,j}h_{t,j} - \sum_{t=1}^T \eps_t^2 g_{t,j} \E[h_{t,j}| g_{t,j}] \\
         &\leq 2\sqrt{\sum_{t=1}^T (\eps_t)^4 g_{t,j}  \E [h_{t,j}|g_{t,j}]\log (T)} + 4 \log(T) \\
         &\leq 4\sqrt{\sum_{t=1}^T \eps_t^2 g_{t,j}  \E [h_{t,j}|g_{t,j}]\log (T)} + 4\log(T) \\
         &\leq \sum_{t=1}^T \eps_t^2 g_{t,j}  \E [h_{t,j}|g_{t,j}] + 8 \log(T)   \tag{AM-GM inequality}
    \end{align*}
    which gives 
    \begin{align*}
        \sum_{t=1}^T \eps_t^2 g_{t,j}h_{t,j} &\leq 2\sum_{t=1}^T \eps_t^2 g_{t,j}  \E [h_{t,j}|g_{t,j}] + 8\log(T) \leq  2^{-j}\eps^2 T + 8\log(T) %\leq \theta_j
    \end{align*}
    with probability at least $1-\frac{1}{T^2}$ using  \pref{eq: tmp 33}. 
    Similarly, 
    \begin{align*}
        \sum_{t=1}^T g_{t,j}h_{t,j} &\leq 2\sum_{t=1}^T g_{t,j}  \E [h_{t,j}|g_{t,j}] + 8\log(T) \leq  2^{-j} T + 8\log(T)
    \end{align*}
    with probability at least $1-\frac{1}{T^2}$. 
    Therefore, with probability at least $1-\frac{2}{T^2}$,  
    \begin{align*}
        \sqrt{\left(\sum_{t=1}^T  g_{t,j}h_{t,j}\right)\left(\sum_{t=1}^T \eps_t^2 g_{t,j}h_{t,j}\right)}
        &\leq \sqrt{2^{-2j}\eps^2 T^2 + 16\cdot 2^{-j}T\log T + 64\log^2 T} \\
        &\leq 2^{-j}\eps T + 4\sqrt{2^{-j}T\log T} + 8\log(T) \\
        &\leq \theta_j
    \end{align*}
    Therefore, the last term in \pref{eq: tmp 22} is bounded by $\frac{2}{T^2} T \leq \frac{2}{T}$. 

    Next, we deal with other terms in \pref{eq: tmp 22}. Again, by $\E[h_{t,j}| g_{t,j}]=2^{-j-1} g_{t,j}$, \pref{eq: tmp 22} implies 
    \begin{align*}
        &2^{-j-1}\E\left[\sum_{t=1}^T (f_t(a_t) - f_t(u^{\calA_t})) g_{t,j} \right]  \leq \E\left[c_1 \sqrt{2^{-j-1}  \sum_{t=1}^T g_{t,j}} +  c_2 \theta_j\max_{t\leq T} g_{t,j}\right] + \frac{2}{T}.  
    \end{align*}
    which implies after rearranging:
    \begin{align*}
        &\E\left[\sum_{t=1}^T (f_t(a_t) - f_t(u^{\calA_t})) g_{t,j} \right] \\
        &\leq \E\left[ c_1  \sqrt{\frac{1}{2^{-j-1}} \sum_{t=1}^T g_{t,j} } + \left( \frac{c_2\theta_j}{2^{-j-1}}\right)\max_{t\leq T} g_{t,j}\right] + \frac{2}{T2^{-j-1}} \\
        &\leq \E\left[c_1 \sqrt{  \sum_{t=1}^T \frac{2g_{t,j}}{w_t} } + 4c_2 \left( \eps T + \sqrt{\frac{T\log T}{2^{-j}}} + \log T\right) \max_{t\leq T} g_{t,j}\right]  
+ \frac{2}{T2^{-j-1}}. \tag{using that when $g_{t,j}=1$, $\frac{1}{2^{-j-1}}\leq \frac{2}{w_t}$, and the definition of $\theta_j$}
    \end{align*}
    Now, summing this inequality over all $j\in\{0, 1, \ldots, \lceil \log_2 T\rceil\}$, we get 
    \begin{align*}
        &\E\left[\sum_{t=1}^T (f_t(a_t) - f_t(u^{\calA_t})) \ind\left\{w_t >  \frac{1}{T}\right\}   \right] \\
        &\leq \order\left(\E\left[c_1\sqrt{N \sum_{t=1}^T \frac{1}{w_t}}+ Nc_2 \eps T + c_2 \sqrt{\frac{T\log T}{\min_{t\leq T} w_t}} + c_2 N\log T\right] + 1\right) \\
        &\leq \order\left(\E\left[(c_1 + c_2)\sqrt{T\log(T) \rho_T} \right] + c_2 \eps T\log T\right)
    \end{align*}
    where $N\leq \order(\log T)$ is the number of $\alg_j$'s that has been executed at least once. 

    On the other hand,  
    \begin{align*}
        \E\left[\sum_{t=1}^T (f_t(a_t) - f_t(u^{\calA_t}))\ind\left\{ 
w_t\leq \frac{1}{T} \right\}\right] < T \E\left[\ind\left\{ 
\rho_T\geq T \right\}\right]\leq \E\left[\rho_T\right]. 
    \end{align*}
    Combining the two parts and using the assumption $c_2 \geq 1$ finishes the proof. 
\end{proof}

\paragraph{Top Layer (from Known $\eps$ to Unknown $\eps$)}  In this subsection, we use the algorithm that we construct in \pref{thm: conversion thm} as a base algorithm, and further construct an algorithm with $\sqrt{T} + \eps$ regret under unknown $\eps$. The idea is to run multiple base algorithms, each with a different hypothesis on $\eps$; on top of them, run another multi-armed bandit algorithm to adaptively choose among them. The goal is to let the top-level bandit algorithm perform almost as well as the best base algorithm. This is the Corral idea outlined in \cite{agarwal2017corralling, foster2020adapting, luo2022corralling}, and the algorithm is presented in \pref{alg: corral}. 

\begin{theorem} \label{thm:corral}
    Using an algorithm constructed in \pref{thm: conversion thm} as a base algorithm, \pref{alg: corral} ensures $\Reg=\order\left(c_1'\sqrt{T\log^3 T} +  c_2' \eps T \right)$ without knowing $\eps$. 
\end{theorem}

The top-level bandit algorithm is an FTRL with log-barrier regularizer. We first state the standard regret bound of FTRL under log-barrier regularizer, whose proof can be found in, e.g., Theorem 7 of \cite{wei2018more}. 
\begin{lemma}\label{lem: FTRL basic property}
    The FTRL algorithm over a convex subset $\Omega$ of the $(M-1)$-dimensional simplex $\Delta(M)$:
    \begin{align*}
        w_t = \argmin_{w\in\Omega} \left\{ \left\langle 
w, \sum_{\tau=1}^{t-1} \ell_\tau \right\rangle + \frac{1}{\eta}\sum_{i=1}^M \log \frac{1}{w_i} \right\}
    \end{align*}
    ensures for all $u\in\Omega$, 
    \begin{align*}
        \sum_{t=1}^T \langle w - u, \ell_t\rangle \leq \frac{M\log T}{\eta} + \eta \sum_{t=1}^T \sum_{i=1}^M w_{t,i}^2 \ell_{t,i}^2
    \end{align*}
    as long as $\eta w_{t,i}|\ell_{t,i}|\leq \frac{1}{2}$ for all $t,i$. 
\end{lemma}

\begin{algorithm}[t]
    \caption{(A Variant of) Corral}\label{alg: corral}
    \textbf{Initialize}: a log-barrier algorithm with each arm being an instance of an algorithm satisfying the guarantee in \pref{thm: conversion thm}. The hypothesis on $\eps T$  is set to $2^{i}$ for arm $i$ ($i=1,2,\ldots,M\triangleq \lceil \log_2 T\rceil$). 

    \textbf{Initialize}: $\rho_{0,i} = M,\; \forall i$. \\

    \For{$t=1,2,\ldots, T$}{
        Let 
        \begin{align*}
            w_t = \argmin_{w\in \Delta(M), w_i\geq \frac{1}{T},\forall i}\left\{ \left\langle w, \sum_{\tau=1}^{t-1} 
(\hat{z}_\tau - r_\tau) \right\rangle + \frac{1}{\eta} \sum_{i=1}^M \log\frac{1}{w_i} \right\}
        \end{align*}
        where $\eta = \frac{1}{4c_1'\sqrt{ T}}$. \\
        For all $i$, send $w_{t,i}$ to instance $i$. \\
        Draw $i_t\sim w_t$. \\
        Execute the $a_t$ output by instance $i_t$\\
        Receive the loss $z_{t,i_t}$ for action $a_t$ (whose expectation is $f_t(a_t)$)  and send it to instance $i_t$.  \\
        Define for all $i$:
        \begin{align*}
             \hat{z}_{t,i} &= \frac{z_{t,i}\ind[i_t=i]}{w_{t,i}}, \\
             \rho_{t,i} &= \min_{\tau\leq t}\frac{1}{w_{\tau,i}}, \\
             r_{t,i} &= c_1'\left(\sqrt{\rho_{t,i} T} - \sqrt{\rho_{t-1,i} T}\right). 
        \end{align*} 
    }
\end{algorithm}

\begin{proof}[Proof of \pref{thm:corral}]
The Corral algorithm is essential an FTRL with log-barrier regularizer. To apply \pref{lem: FTRL basic property}, we first verify the condition $\eta w_{t,i}|\ell_{t,i}|\leq \frac{1}{2}$ where $\ell_{t,i}=\hat{z}_{t,i}-r_{t,i}$. By our choice of $\eta$, 
\begin{align*}
    \eta w_{t,i} |\hat{z}_{t,i}| 
    &\leq \eta z_{t,i} \leq \frac{1}{4}, \tag{because $c_1'\geq 1$}\\
    \eta w_{t,i} r_{t,i} 
    &= \eta c_1'\sqrt{T} w_{t,i}(\sqrt{\rho_{t,i}} - \sqrt{\rho_{t-1,i}}). 
\end{align*}
The right-hand side of the last equality is non-zero only when $\rho_{t,i} > \rho_{t-1,i}$, implying that $\rho_{t,i} = \frac{1}{w_{t,i}}$. Therefore, we further bound it by 
\begin{align}
    \eta w_{t,i}r_{t,i} 
    &\leq \eta c_1'\sqrt{ T} \frac{1}{\rho_{t,i}}(\sqrt{\rho_{t,i}} - \sqrt{\rho_{t-1,i}})\nonumber \\
    &= \eta c_1 '\sqrt{T} \left(\frac{1}{\sqrt{\rho_{t,i}}} - \frac{\sqrt{\rho_{t-1,i}}}{\rho_{t,i}}\right) \nonumber \\
    &\leq \eta c_1 '\sqrt{T} \left(\frac{1}{\sqrt{\rho_{t-1,i}}} - \frac{1}{\sqrt{\rho_{t,i}}}\right)  \tag{$\frac{1}{\sqrt{a}} - \frac{\sqrt{b}}{a}\leq \frac{1}{\sqrt{b}} - \frac{1}{\sqrt{a}}$ for $a, b > 0$} \\
    & \label{eq: from the middle}\\
    &\leq \eta c_1 '\sqrt{T}  \tag{$\rho_{t,i}\geq 1$}\nonumber \\
    &= \frac{1}{4}   \nonumber \tag{definition of $\eta$}
\end{align}
which can be combined to get the desired property $\eta w_{t,i}|\hat{z}_{t,i} - r_{t,i}|\leq \frac{1}{2}$. 

Hence, by the regret guarantee of log-barrier FTRL (\pref{lem: FTRL basic property}), we have 
\begin{align*}
    &\E\left[\sum_{t=1}^T (z_{t,i_t} - z_{t,i^\star}) \right]  \\
    &\leq \order\Bigg(\frac{M\log T}{\eta} + \eta  \E\Bigg[\underbrace{\sum_{t=1}^T \sum_{i=1}^M w_{t,i}^2 (\hat{z}_{t,i} - r_{t,i})^2 }_{\textbf{term}_1}\Bigg] \Bigg) + \E\Bigg[\underbrace{\sum_{t=1}^T 
\left(\sum_{i=1}^M w_{t,i} r_{t,i} - r_{t,i^\star}\right)}_{\textbf{term}_2} \Bigg] 
\end{align*}
where $i^\star$ is the smallest $i$ such that $2^{i}$ upper bounds the true total misspecification amount $\eps T$. 

\textbf{Bounding $\textbf{term}_1$}: 
\begin{align*}
    \textbf{term}_1 \leq 2\eta \sum_{t=1}^T \sum_{i=1}^M w_{t,i}^2 (\hat{z}_{t,i}^2 + r_{t,i}^2) 
\end{align*}
where 
\begin{align*}
2\eta \sum_{t=1}^T\sum_{i=1}^M w_{t,i}^2\hat{z}_{t,i}^2 = 2\eta \sum_{t=1}^T \sum_{i=1}^M z_{t,i}^2 \ind\{i_t=i\} \leq \order(\eta T) 
\end{align*}
and
\begin{align*}
    2\eta \sum_{t=1}^T \sum_{i=1}^M  w_{t,i}^2 r_{t,i}^2 &\leq 4\eta \sum_{t=1}^T \sum_{i=1}^M  (c_1'\sqrt{ T})^2 \left(\frac{1}{\sqrt{\rho_{t-1,i}}} - \frac{1}{\sqrt{\rho_{t,i}}}\right)^2 \tag{continue from \pref{eq: from the middle}}\\
    &\leq 4\eta c_1'^2 T \times \sum_{t=1}^T \sum_{i=1}^M \left(\frac{1}{\sqrt{\rho_{t-1,i}}} - \frac{1}{\sqrt{\rho_{t,i}}}\right) \tag{$\frac{1}{\sqrt{\rho_{t-1,i}}} - \frac{1}{\sqrt{\rho_{t,i}}} \leq 1$ and $1-a\leq -\ln a$} \\
    &\leq 4\eta c_1'^2 T M^{\frac{3}{2}}. \tag{telescoping and using $\rho_{0,i} =M$ and $\rho_{T,i} \leq T$}
\end{align*}
\textbf{Bounding $\textbf{term}_2$}: 
\begin{align*}
    \textbf{term}_2&= \sum_{t=1}^T \sum_{i=1}^M w_{t,i} r_{t,i} - \sum_{t=1}^T r_{t,i^\star} \\
    &\leq  c_1'\sqrt{T} \sum_{t=1}^T \sum_{i=1}^M  \left(\frac{1}{\sqrt{\rho_{t-1,i}}} - \frac{1}{\sqrt{\rho_{t,i}}}\right)  - \left(c_1'\sqrt{\rho_{T,i^\star} T}  - c_1'\sqrt{\rho_{0,i^\star} T}  \right) \\
    \tag{continue from \pref{eq: from the middle} and using $1-a\leq -\ln a$}\\
    &\leq \order\left(c_1'\sqrt{ T}M^{\frac{3}{2}}\right)  - c_1'\sqrt{\rho_{T,i^\star} T}.
\end{align*}
Combining the two terms and using $\eta=\Theta\left(\frac{1}{c_1'\sqrt{T} + c_2'}\right)$, $M=\Theta(\log T)$, we get 
\begin{align}
     \E\left[\sum_{t=1}^T (f_t(a_t) - z_{t,i^\star}) \right] 
 &= \E\left[\sum_{t=1}^T (z_{t,i_t} - z_{t,i^\star}) \right]   \nonumber \\
 &= \order\left(c_1' \sqrt{T\log^3 T} \right)  - \E\left[c_1'\sqrt{\rho_{T,i^\star} T}  \right]   \label{eq: top layer guarantee}
\end{align}

On the other hand, by the guarantee of the base algorithm (\pref{thm: conversion thm}) and that $\eps T\in[2^{i^\star-1}, 2^{i^\star}]$, we have
\begin{align}
    \E\left[\sum_{t=1}^T (z_{t,i^\star} - f_t(u^{\calA_t})\right]\leq\E\left[ c_1'\sqrt{\rho_{T,i^\star}  T} \right]  + c_2' \eps T. 
 \label{eq: bottom layer guarantee} 
\end{align}
   Combining \pref{eq: top layer guarantee} and \pref{eq: bottom layer guarantee}, we get 
   \begin{align*}
       \E\left[\sum_{t=1}^T (f_t(a_t) - f_t(u^{\calA_t}))\right] \leq \order\left( c_1'\sqrt{ T\log^3 T} \right) + c_2' \eps T,
   \end{align*}
   which finishes the proof.
\end{proof}

\begin{proof}[Proof of \pref{thm: misspeicifcation thm}]
   As shown in \pref{app: known misspecification}, our \pref{alg: FTRL} can be adapted to satisfy \pref{def: base alg} with $c_1 = \Theta(d^2\log T)$ and $c_2 = \Theta(\sqrt{d})$. By a concatenation of \pref{thm: conversion thm} and \pref{thm:corral}, we conclude that there is an algorithm that achieves 
   \begin{align*}
       \order\left((c_1+c_2)\sqrt{T}\log^2 T + c_2 \eps T\log T\right) = \order\left(d^2\sqrt{T}\log^2 T + \sqrt{d}\eps T\log T\right).
   \end{align*}
   regret under unknown $\eps$. 
\end{proof}

\section{Analysis for Linear EXP4} \label{app: EXP4}

\begin{proof}[Proof of \pref{thm: exp4 guarantee}]
    We first show that 
    \begin{align}
    \forall \pi\in\Pi:\ \  \Reg(\pi) \triangleq  \E\left[\sum_{t=1}^T 
a_t^\top y_t - \sum_{t=1}^T \pi(\calA_t)^\top y_t \right] \leq \order\left(\gamma T + \frac{\ln|\Pi|}{\eta}+\eta d T\right)\,.  \label{eq: Reg for Pi}
\end{align}
    The magnitude of the loss is bounded by 
    \begin{align*}
        |\hatell_{t,\pi}| 
        & = \left|\left\langle \pi(\calA_t), \tildeH_t^{-1}a_t\ell_t \right\rangle\right|  \\
        &\leq \norm{ \pi(\calA_t)}_{\tildeH_t^{-1}}\norm{a_t}_{\tildeH_t^{-1}}\\&\leq \frac{1}{\gamma}\norm{ \pi(\calA_t)}_{G_t^{-1}}\norm{a_t}_{G_t^{-1}}\leq \frac{d}{\gamma}\,.
    \end{align*}
    If $\gamma \geq 2d\eta$, then we have $|\hatell_{t,\pi}|\leq \frac{1}{2}$ and we can use the standard regret bound of exponential weights: 
    \begin{align*}
        \forall \pi\in\Pi:\qquad \Reg(\pi) \leq \gamma T + \frac{\ln|\Pi|}{\eta} +\eta\sum_{t=1}^{T}\E\left[\E_{a_t\sim p_t}\left[\sum_{\pi\in\Pi}P_{t,\pi}\hatell_{t,\pi}^2\right]\right]\,.
    \end{align*}
    Let $H_t=\E_{a\sim p_t}[aa^\top]$. Then we have $\tildeH_t^{-1}\preceq \frac{1}{1-\gamma}H_t^{-1}$, and thus
    \begin{align*}
        \E_{a_t\sim p_t}\left[\sum_{\pi\in\Pi}P_{t,\pi}\hatell_{t,\pi}^2\right]
        &\leq \E_{a_t\sim p_t} \left[ 
\sum_{\pi\in\Pi} P_{t,\pi} \cdot  \langle \pi(\calA_t), \tildeH_t^{-1}a_t\rangle^2 \right]
             \\
        &= \E_{a_t\sim p_t}\E_{a\sim p_t} \left[\langle a, \tildeH_t^{-1} a_t\rangle^2\right]  \tag{by the definition of $p_{t,a}$} \\
        &\leq \frac{1}{(1-\gamma)^2}\operatorname{Tr}\left(H_tH_t^{-1}H_tH_t^{-1}\right)=\order(d)\,.
    \end{align*}
    Combining all proves \pref{eq: Reg for Pi}. 

   Next, we show that there exists $\theta\in\Theta$ such that 
    \begin{align}
         \E_{\calA\sim D} \left[ 
\sum_{t=1}^T  (\pi_\theta(\calA) - \pi^\star(\calA))^\top y_t \right] \leq \order(1).  \label{eq: suffice to show}
    \end{align}

    Let $\hat{\theta}$ be the closest element in $\Theta$ to $\sum_{t=1}^T y_t$. By the definition of $\Theta$ and the assumption that $\|y_t\|\leq 1$, we have $\norm{\hat{\theta} - \sum_{t=1}^T y_t}\leq \epsilon$. Thus, for any $\calA$, 
    \begin{align*}
        &\sum_{t=1}^T   (\pi_{\hat{\theta}}(\calA) - \pi^\star(\calA))^\top y_t 
        \leq \sum_{a\in\calA}  (\pi_{\hat{\theta}}(\calA) - \pi^\star(\calA))^\top \hat{\theta} + \epsilon \leq \epsilon
    \end{align*}
    where the last inequality is by the fact that $\pi_{\hat{\theta}}(\calA)=\argmin_{a\in\calA} a^\top \hat{\theta}$. Taking expectation over $\calA$ gives \pref{eq: suffice to show}. 

Finally, combining \pref{eq: Reg for Pi} and \pref{eq: suffice to show}, choosing $\epsilon=1$ and $\gamma = 2d\eta = 2d\sqrt{\frac{\log T}{T}}$, we get 
\begin{align*}
    \Reg 
    &= \E\left[\sum_{t=1}^T 
a_t^\top y_t - \sum_{t=1}^T \pi^\star(\calA_t)^\top y_t \right] \\
&= \E\left[\sum_{t=1}^T 
a_t^\top y_t - \sum_{t=1}^T \pi_{\hat\theta}(\calA_t)^\top y_t \right] +  \E_{\calA\sim D} \left[ 
\sum_{t=1}^T  (\pi_{\hat\theta}(\calA) - \pi^\star(\calA))^\top y_t \right] \\
    &= \order\left(\gamma T + \frac{\ln ((2T)^d)}{\eta} + \eta d T + 1\right) \\
    &= \order\left(d\sqrt{T\log T}\right),  
\end{align*}
finishing the proof. 
\end{proof}
\section{Comparison with \cite{dai2023refined} and \cite{sherman2023improved}} \label{app: comparison}

We state the exponential weight algorithm adopted by \cite{luo2021policy, dai2023refined, sherman2023improved} in \pref{alg: exponential weight}, which is an algorithm that we know to achieve the prior-art regret bound in our setting (though they studied a more general MDP setting). 

Their algorithm proceeds in \emph{epochs} (indexed by $k$), where every epoch consists of $W$ rounds. The policy on action set $\calA$ in the $k$-th epoch is defined as 
\begin{align*}
    p^{\calA}_k(a) \propto \exp\left(-\eta \sum_{s=1}^{k-1} (a^\top \hat{y}_s - b_s(a)) \right)
\end{align*}
where $\hat{y}_k$ is the loss estimator for epoch $k$, and $b_k(a)$ is a (non-linear) bonus. In all $W$ rounds in epoch $k$, the same policy is executed. The samples obtained in these $W$ rounds are randomly divided into two halfs. One half is used to estimate the covariance matrix $\hatSigma_k$, and the other half is used to construct the loss estimator $\hat{y}_k$ (see \pref{line: exponential construction} of \pref{alg: exponential weight}). 

\setcounter{AlgoLine}{0}
\begin{algorithm}
    \caption{Exponential weights with magnitude-reduced loss estimators} \label{alg: exponential weight}

    \nl 
    \For{$k=1,2,\ldots, \frac{T}{W}$}{
    \nl    For all $\calA$, define 
        \begin{align*}
            p_k^{\calA}(a) = \frac{\exp\left( -\eta \sum_{s=1}^{k-1} 
(a^\top \hat{y}_{s} - b_s(a)) \right)}{\sum_{a'\in\calA} \exp\left(-\eta \sum_{s=1}^{k-1} (a'^\top \hat{y}_s  - b_s(a'))\right) } \qquad \text{for all } a\in\calA. 
        \end{align*}
     \nl Randomly partition $\{(k-1)W+1, \ldots, kW\}$ into two equal parts $\calT_k, \calT_{k}'$. \\
     \ \\
     \nl \For{$t=(k-1)W+1, \ldots, kW$}{
         receive $\calA_t$, sample $a_t\sim p_k^{\calA_t}$, and receive $\ell_t$.  
     }
     \ \\
     \nl Define
     \begin{align*}
         \hatSigma_k &= \beta I + \frac{1}{|\calT_k|} \sum_{t\in\calT_k} a_t a_t^\top  \\
         \hat{y}_k &=  \hatSigma_k^{-1} \left(\frac{1}{|\calT_k'|}\sum_{t\in\calT_{k}'}a_t\ell_t\right) \\
         b_k(a) &= \alpha \|a\|_{\hatSigma_k^{-1}}. 
     \end{align*}
     \label{line: exponential construction}
    }
    
\end{algorithm}

\subsection{Regret Analysis Sketch}
The regret analysis starts with a standard decomposition that is similar to ours. We abuse the notation by defining $y_k = \frac{1}{W}\sum_{t=(k-1)W}^{kW} y_t$. Then
\begin{align*}
    \Reg&= W\E\left[\sum_{k=1}^{T/W} p^{\calA_0}_k(a) \langle 
 a- u^{\calA_0}, y_k \rangle\right] \\
 &= \underbrace{W\E\left[\sum_{k=1}^{T/W} p^{\calA_0}_k(a) \Big(\langle 
 a, \hat{y}_k \rangle - b_k(a)\Big) - \Big(u^{\calA_0} - b_k(u^{\calA_0})\Big) \right] }_{\textbf{EW-Reg}} + \underbrace{W \E\left[\sum_{k=1}^{T/W} p_k^{\calA_0}(a)b_k(a) - b_k(u^{\calA_0}) \right]}_{\textbf{Bonus}} \\
 &\qquad \qquad + \underbrace{W \E\left[ \sum_{k=1}^{T/W} p_k^{\calA_0}(a) \langle a-u^{\calA_0}, y_k - \hat{y}_k \rangle \right] }_{\textbf{Bias}}.
 \end{align*}
Bounding the regret term follows the standard analysis of exponential weight: 
\begin{align*}
    \textbf{EW-Reg} &\leq W\E\left[\frac{\ln|\calA_0|}{\eta} + \eta \sum_{k=1}^{T/W} \sum_{a\in\calA_0} p_k^{\calA_0}(a) \langle a, \hat{y}_k \rangle^2 + \eta \sum_{k=1}^{T/W}\sum_{a\in\calA_0} p_k^{\calA_0}(a) b_k(a)^2 \right] \\
    &\leq W\E\left[\frac{\ln|\calA_0|}{\eta} + \eta \sum_{k=1}^{T/W} \sum_{a\in\calA_0} p_k^{\calA_0}(a) a^\top \hatSigma_k^{-1} H_k \hatSigma_k^{-1} a+ \eta \sum_{k=1}^{T/W} \frac{\alpha^2}{\beta} \right]
\end{align*}
where $H_k = \E_{\calA\sim D}\E_{a\sim p^{\calA}_k}[aa^\top]$. 
Then they use the following fact to bound the stability term: as long as $W\geq \frac{d}{\beta^2}$, it holds with high probability that $\hatSigma_k^{-1} H_k \hatSigma_k^{-1} \preceq 2\hatSigma_k^{-1}$. Thus \textbf{EW-Reg} can be further bounded by 
\begin{align*}
    \textbf{EW-Reg} &\lesssim W\left(\frac{\ln |\calA_0|}{\eta} + \eta \E\left[\sum_{k=1}^{T/W} 
\sum_{a\in\calA_0} p_k^{\calA_0}(a) \|a\|^2_{\hatSigma_k^{-1}}\right] + \eta \frac{T}{W}\frac{\alpha^2}{\beta} \right) \\
&\leq \frac{W\ln|\calA_0|}{\eta} + \eta dT + \eta T\frac{\alpha^2}{\beta}. 
\end{align*}
By the definition of the bonus function $b_t$, it holds that
\begin{align*}
   \textbf{Bonus} = W\E\left[\alpha \sum_{k=1}^{T/W} \sum_{a\in\calA_0} p_k^{\calA_0}(a)\|a\|_{\hatSigma_k^{-1}} \right] - W\E\left[\alpha \sum_{k=1}^{T/W} \|u^{\calA_0}\|_{\hatSigma_k^{-1}} \right].
\end{align*}
Finally, the bias term can be bounded as follows: 
\begin{align*}
    \textbf{Bias} &= W\E\left[\sum_{k=1}^{T/W} p^{\calA_0}_k(a)(a - u^{\calA_0})^\top (y_k - \hatSigma_k^{-1}H_k y_k) \right] \\
    &= W\E\left[\sum_{k=1}^{T/W} p^{\calA_0}_k(a)(a - u^{\calA_0})^\top \hatSigma_k^{-1} (\hatSigma_k - H_k)y_k \right] \\
    &\leq W\E\left[\sum_{k=1}^{T/W} p^{\calA_0}_k(a)\|a - u^{\calA_0}\|_{\hatSigma_k^{-1}}  \|(\hatSigma_k - H_k)y_k\|_{\hatSigma_k^{-1}} \right].
\end{align*}
The bias here has a similar form as in our case. They use the following fact to bound the bias: as long as $W\geq \frac{d}{\beta^2}$, 
it holds that $\|(\hatSigma_k - H_k)y_k\|_{\hatSigma_k^{-1}}\leq \sqrt{\beta d}$. Therefore, the bias can further be upper bounded by 
\begin{align*}
   \textbf{Bias}\leq W\E\left[\sqrt{\beta d}\sum_{k=1}^{T/W} \sum_{a\in\calA_0} p^{\calA_0}_k(a)\|a\|_{\hatSigma_k^{-1}} + \sqrt{\beta d} \sum_{k=1}^{T/W} \|u^{\calA_0}\|_{\hatSigma_k^{-1}} \right]. 
\end{align*}

Combining the three parts, we get that the overall regret is of order
\begin{align*}
   \E\left[\frac{W\ln|\calA_0|}{\eta} + \eta dT + \eta T\frac{\alpha^2}{\beta} + W(\alpha + \sqrt{\beta d}) \sum_{k=1}^{T/W} \sum_{a\in\calA_0} p_k^{\calA_0}(a)\|a\|_{\hatSigma_k^{-1}} + W(\sqrt{\beta d} - \alpha) \sum_{k=1}^{T/W} \|u^{\calA_0}\|_{\hatSigma_k^{-1}}\right]. 
\end{align*}
Choosing $\alpha\approx \sqrt{\beta d}$, we further bound it by
\begin{align*}
   &\E\left[\frac{W\ln|\calA_0|}{\eta} + \eta dT + W\sqrt{\beta d} \sum_{k=1}^{T/W}\sum_{a\in\calA_0} p^{\calA_0}_k(a)\|a\|_{\hatSigma_k^{-1}}\right] \\
   &\leq \E\left[\frac{W\ln|\calA_0|}{\eta} + \eta dT +  W\sqrt{\beta d} \sum_{k=1}^{T/W} \sqrt{\sum_{a\in\calA_0} p^{\calA_0}_k(a)\|a\|_{\hatSigma_k^{-1}}^2} \right]\\
   &\leq \frac{W\ln|\calA_0|}{\eta} + \eta dT + \sqrt{\beta}dT. 
\end{align*}
Recall the constraint $W\geq \frac{d}{\beta^2}$. Choosing $W=\frac{d}{\beta^2}$ gives 
\begin{align}
    \frac{d\ln|\calA_0|}{\eta \beta^2} + \eta dT + \sqrt{\beta } dT \label{eq: last eq}
\end{align}
which gives $d(\ln|\calA_0|)^{\frac{1}{6}}T^{\frac{5}{6}}$ with the optimally chosen $\eta$ and $\beta$. 

\paragraph{Remark} Due to the restrictions on the magnitude of the loss estimator required by the exponential weight algorithm, there is actually another constraint $\frac{\eta}{\beta}\leq 1$, which 
makes \pref{eq: last eq} be $d(\ln|\calA_0|)^{\frac{1}{7}}T^{\frac{6}{7}}$ at best. This is exactly the bound obtained by \cite{sherman2023improved}. A more sophisticated way to construct $\hat{y}_k$ developed by \cite{dai2023refined} removes this additional requirement and allows a bound of $d(\ln|\calA_0|)^{\frac{1}{6}}T^{\frac{5}{6}}$. The sub-optimal bound $T^{\frac{8}{9}}$ reported in \cite{dai2023refined} is due to issues related to MDPs, which are not presented in the contextual bandit case here. 

\iffalse

\subsection{Implication of Our Result on Linear MDPs}
There are two key element to improve the $T^{\frac{5}{6}}$ regret bound to $\sqrt{T}$. The first is a tighter concentration bound analysis used in bounding the bias term, and the second is by data reusing. Applying the data reusing technique to MDPs is trickier because the state distribution changes with the policy of the learner. Therefore, it does not enjoy the i.i.d. context property as in our case. However, the improvement in concentration bound can be directly applied to MDPs. As shown in \pref{lem: local norm concentration for matrix 2}, to bound the bias term by $\sqrt{\beta d}$, only $\order\left(\frac{1}{\beta}\right)$ samples are required. This allows us to improve \pref{eq: last eq} to the order of 
\begin{align*}
    \frac{d}{\eta \beta} + \eta dT + \sqrt{\beta}dT
\end{align*}
and obtain a $\order(dT^{\frac{3}{4}})$ bound for linear MDPs (omitting dependencies on the horizon length $H$). This would improve the prior art $dT^{\frac{6}{7}}$ obtained by \cite{sherman2023improved}. Since our bound for the bias term is nearly tight, to improve the bound from $T^{\frac{3}{4}}$ to $\sqrt{T}$, the effort should be dedicated to reusing data. 

\fi

\end{document}